\documentclass[11pt]{article}

\usepackage[margin=1in]{geometry}
\usepackage{graphicx}
\usepackage{amsmath, amssymb, amsthm, mathtools}
\usepackage{enumitem}
\usepackage{xcolor}
\usepackage[
    colorlinks=true,
    linkcolor=black,
    citecolor=blue,
    urlcolor=blue
]{hyperref}
\usepackage{microtype}

\usepackage[section]{placeins}
\usepackage{caption}
\usepackage{booktabs}
\usepackage[sort&compress,numbers]{natbib}

\newcommand{\R}{\mathbb{R}}
\newcommand{\E}{\mathbb{E}}
\newcommand{\PP}{\mathbb{P}}
\newcommand{\1}{\mathbf{1}}
\newcommand{\norm}[1]{\left\lVert #1 \right\rVert}
\newcommand{\abs}[1]{\left\lvert #1 \right\rvert}
\newcommand{\ip}[2]{\left\langle #1,#2\right\rangle}
\newcommand{\opnorm}[1]{\left\lVert #1 \right\rVert_{\mathrm{op}}}
\newcommand{\Sph}{\mathbb{S}}
\newcommand{\assitemref}[1]{\hyperref[#1]{\textbf{(\ref*{#1})}}}
\newcommand{\assitemrefup}[1]{\hyperref[#1]{\textup{(\ref*{#1})}}}

\DeclareMathOperator{\Lip}{Lip}
\DeclareMathOperator{\tr}{tr}

\DeclareMathOperator{\Var}{Var}
\DeclareMathOperator{\supp}{supp}

\theoremstyle{plain}
\newtheorem{theorem}{Theorem}[section]
\newtheorem{proposition}[theorem]{Proposition}
\newtheorem{lemma}[theorem]{Lemma}
\newtheorem{corollary}[theorem]{Corollary}

\theoremstyle{definition}
\newtheorem{assumption}[theorem]{Assumption}
\newtheorem{definition}[theorem]{Definition}
\newtheorem{remark}[theorem]{Remark}

\numberwithin{equation}{section}

\title{Scaling Limits of Constant-Stepsize SGD at Flat Minima}

\newcommand\blfootnote[1]{%
  \begingroup
  \renewcommand\thefootnote{}\footnote{#1}%
  \addtocounter{footnote}{-1}%
  \endgroup
}   

\author{Jingyi Zhang$^1$,  
  Cheng Mao$^2$,
  Debankur Mukherjee$^3$}
\date{}

\begin{document}
\maketitle

\begin{abstract}
For stochastic gradient descent (SGD) with a constant stepsize $\alpha$, the invariant law of the iterates, centered at a minimizer, describes the behavior of the algorithm over long time horizons.  In the strongly convex case, this invariant law has the familiar $\sqrt{\alpha}$ scaling and a Gaussian limit as $\alpha\downarrow 0$. We show that this behavior changes fundamentally for convex objectives $H$ with flat minima and (sub)quadratic tails.

More specifically, we study SGD with Markovian noise generated by a contractive driving chain.  For every sufficiently small constant stepsize $\alpha$, we prove existence, uniqueness, and geometric convergence to an augmented invariant law in a Wasserstein distance induced by an $\alpha$-dependent metric.
When the minimizer \(x_\star\) has local flatness exponent \(m\ge2\), meaning that \(\nabla^2H(x)\asymp \norm{x-x_\star}^{m-2}I_d\) as
\(x\to x_\star\), we obtain a contraction bound with factor
\(1-c\alpha^{m-1}\), where \(c>0\) is a constant.  This recovers the factor \(1-c\alpha\) in the quadratic case \(m=2\).
We then analyze the small-stepsize scaling limit. We show that the invariant law concentrates on the scale $\alpha^{1/m}$ and that the rescaled iterates converge weakly to the stationary distribution of the stochastic differential equation
\[
        dY_t=-h_0(Y_t)\,dt+\Sigma^{1/2}\,dB_t ,
\]
where $h_0$ is the limiting drift at the minimizer and $\Sigma$ denotes the asymptotic covariance.
This recovers the Gaussian limit when \(m=2\) and gives generally non-Gaussian
stationary limits in the flat case \(m>2\).
Finally, we give corresponding results for coordinate-separable objectives with unequal flatness exponents.
\end{abstract}

\blfootnote{$^1$Georgia Institute of Technology, \emph{Email:} \href{mailto:jzhang3450@gatech.edu}{jzhang3450@gatech.edu}}
\blfootnote{$^2$Georgia Institute of Technology, \emph{Email:}  \href{mailto:cheng.mao@math.gatech.edu}{cheng.mao@math.gatech.edu}}
\blfootnote{$^3$Georgia Institute of Technology, \emph{Email:} \href{mailto:debankur.mukherjee@isye.gatech.edu}{debankur.mukherjee@isye.gatech.edu}}

\tableofcontents

\section{Introduction}
\label{sec:introduction}

Stochastic approximation originated in recursive methods for root-finding and optimization based on noisy observations \cite{RobbinsMonro1951,KieferWolfowitz1952,benveniste1990,KushnerYin2003,borkar2008}.  Classical theory uses decreasing stepsizes and studies almost-sure convergence to a target point.  In large-scale optimization, by contrast, constant or piecewise-constant stepsizes are often used for long time intervals.  With a constant stepsize and nondegenerate gradient noise, the last stochastic iterate generally does not converge to the minimizer.  It is instead natural to study the invariant law of the Markov chain generated by the algorithm, which describes the stationary error of the last iterate.

For smooth strongly convex objectives, the stationary error of stochastic gradient descent (SGD) is well understood.  It is typically of order \(\sqrt\alpha\), and after the \(\alpha^{-1/2}\)-rescaling the invariant law converges to the stationary distribution of an Ornstein--Uhlenbeck diffusion, hence Gaussian.  This picture appears in small-stepsize asymptotic laws \cite{pflug1986}, in diffusion approximations for constant-step SGD \cite{MandtHoffmanBlei2017}, and in stationary small-stepsize characterizations for SGD-type algorithms \cite{chen2022stationary}.  Complementary Markov chain analyses establish invariant measure expansions, convergence, and concentration in strongly convex settings \cite{DieuleveutDurmusBach2020,MeradGaiffas2025}. However, the quadratic case is not representative of all convex objectives.  When the deterministic drift vanishes to order higher than one at the minimizer, the normalization of the invariant law and its scaling limit both change.

This paper studies convex objectives whose minimizer may be flatter than quadratic.  Near the minimizer, the deterministic drift may satisfy
\[
        \nabla^2H(x)\asymp \norm{x}^{m-2}I_d,
        \qquad m\ge2,
\]
in dimension $d$, so the restoring force is of order \(\norm{x}^{m-1}\).  At the same time, the objective may have subquadratic tails: for \(1\le \beta<2\), the drift for large \(\norm{x}\) may grow only as \(\norm{x}^{\beta-1}\).  These features occur in standard statistical objectives.  Quantile estimation when the distribution function crosses the target probability level at higher order provides a canonical example of local flatness.  For median and \(L_1\) regression, this type of nonregular local behavior was studied by Knight~\cite{Knight1998}; their regularly varying mechanism applies to general quantile loss introduced by Koenker and Bassett~\cite{KoenkerBassett1978}.  The same local geometry appears in the Rockafellar--Uryasev variational representation of conditional value-at-risk (CVaR)~\cite{RockafellarUryasev2000}.  Subquadratic tails arise in robust losses with bounded or sublinear scores, such as Huber-type and generalized Charbonnier losses \cite{Huber1964,Charbonnier1997,Barron2019}, and in logistic losses \cite{BartlettJordanMcAuliffe2006}.  Recent work on subquadratic SGD treats the locally strongly convex case \(m=2\) with \(\beta<2\) \cite{zhangPiecewiseLyapunovAnalysis2025}; the case \(m>2\) requires a different local scale and a different contraction argument.

We emphasize the different roles played by the two key exponents used throughout the paper: The local flatness exponent \(m\) determines the order of the restoring drift near the minimizer and therefore the stationary scaling.  
The tail exponent \(\beta\) determines the drift for large \(\norm{x}\) and the weight function needed to control excursions.  

Moreover, we consider SGD with Markovian noise. 
A driving Markov chain \((\xi_n)_{n \ge 0}\) evolves according to a contractive recursion, and the SGD iterate follows
\begin{equation*}
        X_{n+1}=X_n-\alpha\{h(X_n)+g(X_n,\xi_{n+1})\},
        \qquad h=\nabla H.
\end{equation*}
The augmented process \((X_n,\xi_n)\) is Markov.  This formulation covers not only statistical settings with i.i.d.\ noise but also temporally dependent data streams. 


\paragraph{Challenges and main contributions.}
There are three difficulties for analyzing SGD in the above setting.  First, ordinary Euclidean contraction degenerates near the minimizer when \(m>2\), which is also the region where the invariant laws concentrate as \(\alpha\downarrow0\).  Second, for \(\beta<2\), quadratic Lyapunov functions are not well matched to the drift for large \(\norm{x}\).  Third, under Markovian noise, the current iterate is correlated with future values of the driving chain, which prevents a direct one-step martingale argument. Our main results below address these challenges.

First, for the convergence of SGD with a small fixed stepsize $\alpha$, 
we prove that the SGD iterates converge to a unique invariant law geometrically in a suitably chosen Wasserstein distance, with contraction factor \(1-c\alpha^{m-1}\) for a constant \(c>0\). 
The proof uses the metric of Qu, Blanchet, and Glynn~\cite{QBG2025} induced by a Lyapunov weight function $V_\alpha$ that we carefully construct.
The weight function consists of two terms, one for controlling the subquadratic tail of the objective and the other for controlling the near-minimizer region.  
Moreover, directly applying the result of \cite{QBG2025} does not work in our case, and we show contraction of SGD iterates with a new argument.

Second, with the fixed-\(\alpha\) invariant law in hand, we then identify its scaling limit as $\alpha \downarrow 0$.  The local \(m\)-flat geometry
determines the normalization: the SGD iterates concentrate on the scale \(\alpha^{1/m}\).
Then, the scaling limit is given by the stationary distribution of the stochastic differential equation
\[
        dY_t=-h_0(Y_t)\,dt+\Sigma^{1/2}\,dB_t ,
\]
where $h_0$ denotes the limiting drift, and $\Sigma$ denotes the asymptotic variance which is identified via a Poisson equation argument.
When \(m=2\), the above equation recovers the Ornstein--Uhlenbeck diffusion, and the invariant law is
Gaussian. 
For \(m>2\), the drift \(h_0\) is nonlinear, and the stationary distribution is generally non-Gaussian.
We note that 
Chen, Mou, and Maguluri~\cite{chen2022stationary} numerically exhibited the non-Gaussian quartic scaling. Wang et al.~\cite[Section~5]{WangEtAl2026} conjectured a one-dimensional flat-minimum Gibbs approximation with the corresponding nonstandard scaling and supported it numerically.  Our result establishes the scaling limit rigorously in a multidimensional setting with Markovian noise, confirming and extending the phenomenon suggested by prior work.
 
Finally, for coordinate-separable objective functions, we provide the corresponding constant-stepsize and scaling limit results. Notably, if the coordinates have unequal local flatness exponents, only the flattest coordinates can remain nonzero under the common scale associated with the largest exponent.

Table~\ref{tab:regime-summary} summarizes the cases appearing in the main results.
\begin{table}[t]
\centering
\begin{tabular}{lll}
\toprule
Feature & Representative condition & Effect on the invariant law \\
\midrule
Quadratic & \(h(x)\sim Ax\) & \(\alpha^{1/2}\) scale, Gaussian limit \\
Flat minimum & \(h(x)\sim \norm{x}^{m-2}x\) & \(\alpha^{1/m}\) scale, nonlinear diffusion \\
Markovian noise & correlated \(g(0,\xi_n)\) & asymptotic covariance \(\Sigma\) \\
Subquadratic tail & \(\norm{h(x)}\asymp \norm{x}^{\beta-1}\) & \(\norm{x}^{2-\beta}\) term in the Lyapunov weight \\
\bottomrule
\end{tabular}
\caption{The primary features covered by the main results.  The exponent \(m\) is local and determines the normalization of the invariant law and the limiting drift \(h_0\).  The exponent \(\beta\) is global and determines the tail part of the induced metric.}
\label{tab:regime-summary}
\end{table}

\paragraph{Related literature.}
The classical stochastic approximation literature studies decreasing stepsizes and averaging of iterates, establishing almost-sure convergence and asymptotic normality \cite{RobbinsMonro1951,KieferWolfowitz1952,Fabian1968,Ljung1977,Ruppert1988,PolyakJuditsky1992,benveniste1990,KushnerYin2003,borkar2008}.  Modern nonasymptotic analyses of decreasing-step stochastic optimization include \cite{NemirovskiEtAl2009,moulines2011,BachMoulines2013}.  These results concern convergence of the iterates to an optimizer.  The object here is different: for a constant stepsize, the last iterate has a nontrivial invariant law, and the limit \(\alpha\downarrow0\) is taken at stationarity.

For constant stepsizes, the algorithm is analyzed as a Markov chain.  Pflug~\cite{pflug1986} studied small-stepsize asymptotic laws for stochastic optimization. Mandt, Hoffman, and Blei~\cite{MandtHoffmanBlei2017} developed quadratic diffusion approximation for constant-stepsize SGD.  Dieuleveut, Durmus, and Bach~\cite{DieuleveutDurmusBach2020} developed invariant measure and bias expansions for smooth strongly convex SGD. Merad and Ga{\"i}ffas~\cite{MeradGaiffas2025} proved Wasserstein convergence and concentration in related strongly convex settings. In a different asymptotic regime, Yu, Balasubramanian, Volgushev, and Erdogdu~\cite{YuBalasubramanianVolgushevErdogdu2021} established a central limit theorem for averages of iterates at a fixed stepsize, together with a characterization of the invariant bias.

Besides the classical work of Pflug~\cite{pflug1986}, the recent works most closely related to our scaling limit of invariant laws are Chen, Mou, and Maguluri~\cite{chen2022stationary}, Zhang et al.~\cite{ZhangHuoChenXie2024Prelimit}, and Wang et al.~\cite{WangEtAl2026}.  Pflug~\cite{pflug1986} and Chen, Mou, and Maguluri~\cite{chen2022stationary} both first pass to constant-stepsize invariant laws and then let \(\alpha\downarrow0\) to obtain a Gaussian limit. The latter work obtains Gaussian limits in smooth strongly convex, linear, and contractive settings, and numerically demonstrates the \(\alpha^{1/4}\) scaling and a non-Gaussian limit for a quartic objective.  Zhang et al.~\cite{ZhangHuoChenXie2024Prelimit} establish steady-state convergence for nonsmooth contractive stochastic approximation: their scale remains \(\sqrt\alpha\), while nonsmooth local dynamics can produce a non-Gaussian limit.  Wang et al.~\cite{WangEtAl2026} give a one-dimensional Gibbs approximation for flat convex objectives with the correct nonstandard scaling, conditional on two conjectures concerning stationary moments and Stein equation regularity.  We establish the flat-minimum limit unconditionally and allow a multidimensional setting with Markovian noise.  

The tail behavior we study in this work is closest to recent work on subquadratic SGD for robust and quantile regression~\cite{zhangPiecewiseLyapunovAnalysis2025}.  In the notation below, that setting corresponds to \(m=2\) and \(\beta<2\): the objective is locally strongly convex, but its tail for large \(\norm{x}\) is subquadratic.  The present paper allows this tail behavior to coexist with local flatness \(m>2\).  

Markovian noise introduces a separate issue because the stationary iterate is correlated with future values of the driving chain.  Poisson equations are the standard tool for separating this temporal dependence in stochastic approximation.  For linear stochastic approximation, Huo, Chen, and Xie~\cite{HuoChenXie2026Bias} proved convergence to a unique stationary distribution and developed a stepsize expansion of its bias.  More recent analyses quantify how Markovian memory and nonlinear updates affect stationary bias and weak convergence \cite{HuoZhangChenXie2024Collusion,HadaviMouSamsonovWai2026DecisionDependent}.  Here the Poisson equation facilitates identifying the diffusion covariance $\Sigma$ as the asymptotic covariance of the stationary sequence \(g(0,\xi_n)\).

Our proof of constant-stepsize ergodicity builds on recent work by Qu, Blanchet, and Glynn~\cite{QBG2025}. Their work proves Wasserstein convergence from contractive drift conditions.  
We do not verify their conditions directly as they may not always hold in our case. 
Thus, while we adopt the induced metric of~\cite{QBG2025}, our constant-stepsize ergodicity result requires a new directional kernel estimate adapted to the SGD dynamics.

For the scaling limit as \(\alpha\downarrow0\), a related line of work addresses quantitative stationary approximations, often through generator comparisons or Stein-type arguments.  As a methodological precursor, Gast~\cite{Gast2017} used generator comparison to obtain rates for mean-field models.  Allmeier and Gast~\cite{AllmeierGast2024} used related generator expansions to compute bias corrections for constant-step stochastic approximation with Markovian noise.  Wang et al.~\cite{WangEtAl2026} proved nonasymptotic Wasserstein and tail bounds in smooth settings, including Markovian noise through Poisson equations. Wei et al.~\cite{WeiLiLouWu2025} obtained finite-sample Gaussian approximations for constant-stepsize SGD through linearization.  Our result instead identifies the scaling limit of invariant laws for multidimensional flat objectives, where the limiting drift is nonlinear and the limit is generally non-Gaussian.  While we do not obtain a finite-\(\alpha\) convergence rate, to the best of our knowledge this is the first rigorous invariant-law scaling limit for flat minima, even in dimension one.

\paragraph{Organization.}
The paper is organized around the two limits involved in the analysis: first
the limit of the SGD iterates at constant stepsize, and then the small-stepsize limit at
stationarity.  Section~\ref{sec:main-results} states all the main results.
Section~\ref{sec:applications-numerics}
studies statistical examples that motivate local flatness and subquadratic tails. 
Section~\ref{sec:main-proofs} provides the proof strategy and proves the main theorems.  
Additional technical proofs are deferred to the appendices.

\section{Main results}
\label{sec:main-results}

To simplify the statement of the main results, we assume that zero is the minimizer of the objective function without loss of generality. 
Indeed, with \(x_\star\) denoting the minimizer of the original objective, replacing the optimization variable by \(x-x_\star\) translates the minimizer to the origin. 

Let \(H:\R^d\to\R\) be the objective function, and let \(h:=\nabla H\). Consider SGD with Markovian noise defined as follows. Let \((\xi_n)\) be the driving Markov chain. It takes values in a closed convex set \(\Xi\subseteq\R^{d_\Xi}\) and evolves according to
\[
        \xi_{n+1}=\Phi(\xi_n,U_{n+1}),
\]
where \((U_n)_{n\ge1}\) are i.i.d.\ innovations taking values in a measurable space \(\mathcal U\) and \(\Phi:\Xi\times\mathcal U\to\Xi\) is Borel measurable.  For a constant stepsize \(\alpha\in(0,1)\), the SGD is defined by 
\begin{equation}
\label{eq:markov-sgd-main}
        X_{n+1}
        =
        X_n-\alpha\{h(X_n)+g(X_n,\xi_{n+1})\}.
\end{equation}
Equivalently, with \(Z_n=(X_n,\xi_n)\in\mathsf Z:=\R^d\times\Xi\),
\[
        Z_{n+1}=F_{\alpha,U_{n+1}}(Z_n),
\]
where
\begin{equation}
\label{eq:augmented-random-map}
        F_{\alpha,u}(x,\xi)
        :=
        \left(
        x-\alpha\{h(x)+g(x,\Phi(\xi,u))\},
        \Phi(\xi,u)
        \right).
\end{equation}
We suppress the dependence on \(\alpha\) when it is clear from context.
Note that the augmented process \(Z_n\) is a Markov chain while 
the \(X\)-coordinate alone need not be Markov.  


\begin{remark}
\label{rem:iid-specialization}
The i.i.d.\ noise setting is a special case of the preceding Markovian
formulation. Indeed, to construct i.i.d.\ \(\xi_1,\xi_2,\ldots\) with common law \(\mu\), it suffices to take
\[
        \mathcal U=\Xi,\qquad
        U_1,U_2,\ldots\stackrel{\rm i.i.d.}{\sim}\mu,
        \qquad
        \Phi(\xi,u)=u.
\]
Then
\[
        \xi_{n+1}=\Phi(\xi_n,U_{n+1})=U_{n+1}.
\]
Consequently, the results below apply to this i.i.d.\ setting.
\end{remark}

\subsection{Geometric ergodicity in Wasserstein distance}
\label{subsec:markov-ergodicity}

Our first result is regarding the convergence of the constant-stepsize SGD to an invariant law. More precisely, we establish the geometric ergodicity of the SGD iterates in a Wasserstein distance. 
We first introduce the
metrics used in this work.  These metrics come from the work of Qu, Blanchet, and
Glynn~\cite{QBG2025}, specialized to the augmented finite-dimensional state
space.

\begin{definition}
\label{def:induced-wasserstein}
Let \(E\) be a closed convex subset of a finite-dimensional normed space,
with base norm \(\vert\cdot\vert_\star\).  For a continuous weight function
\(V:E\to[1,\infty)\), define the metric induced by \(V\)
\[
        d_{V,\star}(z,z')
        :=
        \inf_{\gamma\in\operatorname{AC}_E(z,z')}
        \int_0^1
        V(\gamma(t))\,\vert\dot\gamma(t)\vert_\star\,dt,
\]
where \(\operatorname{AC}_E(z,z')\) is the set of absolutely continuous curves
\(\gamma:[0,1]\to E\) with \(\gamma(0)=z\) and \(\gamma(1)=z'\).  For
probability measures \(\mu,\nu\) on \(E\), define
\[
        W_{V,\star}(\mu,\nu)
        :=
        \inf_{\gamma\in\Pi(\mu,\nu)}
        \int_{E\times E}
        d_{V,\star}(z,z')\,\gamma(dz,dz'),
\]
where \(\Pi(\mu,\nu)\) is the set of couplings of \(\mu\) and \(\nu\).

When \(V\equiv1\) on \(\R^d\) with the Euclidean norm, we simply write \(W_1\) for the Wasserstein distance.
For the augmented chain, the base norm is defined to be
\begin{equation} \label{eq:aug-base-norm}
        \vert (x,\xi)\vert_\alpha
        :=
        \norm{x}+\alpha^{-1}\norm{\xi},
        \qquad (x,\xi)\in\R^d\times\R^{d_\Xi},
\end{equation}
where the choice of the balancing factor \(\alpha^{-1}\) is suggested by the analysis. 
The corresponding induced metric and Wasserstein distance are denoted
by \(d_{V,\alpha}\) and \(W_{V,\alpha}\).
\end{definition}



Next, we introduce the assumptions on the objective function $H$. 

\begin{assumption}[Objective function]
\label{ass:H}
Fix \(m\ge2\) and \(\beta\in[1,2]\).  The objective \(H:\R^d\to\R\) satisfies the following conditions.
\begin{enumerate}[label=\textbf{(H\arabic*)}, ref=H\arabic*, leftmargin=3.6em]
\item \label{itm:H1}\textbf{Regularity and convexity.}
\(H\in C^2(\R^d)\), \(H\) is convex, and \(0\in\arg\min H\).  

\item \label{itm:H2}\textbf{Local flatness.}
There exists \(R_H>0\) and constants \(c_{\rm in},C_{\rm in}>0\) such that, for all \(0<\norm{x}\le R_H\),
\[
\nabla^2H(x)\succeq c_{\rm in}\norm{x}^{m-2}I_d,
\qquad
\norm{h(x)}\le C_{\rm in}\norm{x}^{m-1}.
\]
When \(m=2\), this condition is understood to extend to \(x=0\) by continuity of \(\nabla^2H\).

\item \label{itm:H3}\textbf{Tail condition.}
One of the following two tail conditions holds.
\begin{enumerate}[label=\textbf{(H3-\alph*)}, ref=H3-\alph*, leftmargin=3.2em]
\item \label{itm:H3a}\textbf{Quadratic tail.} For \(\beta=2\), there exists \(c_{\rm out}>0\) such that, for all \(\norm{x}\ge R_H\),
\[
\nabla^2H(x)\succeq c_{\rm out}I_d.
\]
\item \label{itm:H3b}\textbf{Subquadratic tail.} For \(1\le\beta<2\), there exist \(c_{\rm out},C_{\rm out}>0\) such that, for all \(\norm{x}\ge R_H\),
\[
\ip{x}{h(x)}\ge c_{\rm out}\norm{x}^{\beta},\qquad
\norm{h(x)}\le C_{\rm out}\norm{x}^{\beta-1}.
\]
\end{enumerate}
\end{enumerate}
\end{assumption}

The tail condition \assitemref{itm:H3} supplies the restoring drift towards the minimizer for the SGD, while the local flatness \assitemref{itm:H2} determines the scale of fluctuations near the minimizer and the scaling limit.

\begin{remark}
To exemplify \assitemref{itm:H2}, 
we note that it includes \emph{locally defined} functions:
\begin{itemize} 
\item
powers of norms such as
\(H(x)=\|x\|^m/m\) and
\(H(x)=m^{-1}(x^\top A x)^{m/2}\) with \(A\succ0\);
\item
polynomials such as
\(H(x)=\|x\|^4+c\|x\|_4^4\) and \(H(x)=\|x\|^4 + c\|x\|^6\) where
\(c>0\) for $m = 4$;
\item
non-polynomial objectives 
such as
\(H(x)=\sin^4(\|x\|)/4\).  
\end{itemize}
Examples with a degenerate Hessian or unequal flatness along different coordinates, such as 
\(x_1^4+x_2^4\) and \(x_1^4+x_2^6\), fail \assitemref{itm:H2}, but they are
covered by the coordinate-separable results in
Section~\ref{subsec:separable-main}.  Exponentially flat functions, such as \(H(x)=e^{-1/x^2}\) near \(0\) with \(H(0)=0\), vanish to infinite order at the minimizer and therefore do not satisfy \assitemref{itm:H2} for any finite \(m\).
\end{remark}

The next assumptions concern the noise in the SGD and the one-step iterate in \eqref{eq:markov-sgd-main}--\eqref{eq:augmented-random-map}.  

\begin{assumption}[Stochastic update]
\label{ass:N}
The driving Markov chain and stochastic update satisfy the following conditions.  
\begin{enumerate}[label=\textbf{(N\arabic*)}, ref=N\arabic*, leftmargin=3.6em]
\item \label{itm:N1}\textbf{Driving-chain contraction.}
There exists a measurable function \(L_\Phi:\mathcal U\to[0,1]\) such
that, for all \(\xi,\eta\in\Xi\) and all \(u\in\mathcal U\),
\[
\norm{\Phi(\xi,u)-\Phi(\eta,u)}
\le L_\Phi(u)\norm{\xi-\eta},
\qquad
\E L_\Phi(U_1)<1.
\]

\item \label{itm:N2}\textbf{Lipschitzness of noise.}
There exists \(L_{g,\Phi}>0\) such that, for all \(x\in\R^d\), all \(\xi,\eta\in\Xi\), and all \(u\in\mathcal U\),
\[
\norm{g(x,\Phi(\xi,u))-g(x,\Phi(\eta,u))}
\le L_{g,\Phi}\norm{\xi-\eta}.
\]

\item \label{itm:N3}\textbf{Reference-point integrability.}
There exists \(\xi_\star\in\Xi\) such that
\begin{equation}
\label{eq:markov-reference-state}
        \E\norm{\Phi(\xi_\star,U_1)-\xi_\star}<\infty.
\end{equation}
If \(\beta=2\), assume also
\begin{equation*}
        \E\norm{g(0,\Phi(\xi_\star,U_1))}<\infty.
\end{equation*}
When \(1\le\beta<2\), this condition follows from \assitemref{itm:N5} at \(x=0\).

\item \label{itm:N4}\textbf{Subquadratic tail dissipativity.}
Define 
\begin{equation}
\label{eq:markov-conditional-mean-g}
        \bar g(x,\xi)
        :=\E_{U_1}[g(x,\Phi(\xi,U_1))].
\end{equation}
When \(\beta=2\), this condition is not imposed.
When \(1\le\beta<2\), there exists
\(c_{\rm diss}>0\) such that, for all \(\norm{x}\ge R_H\) and
all \(\xi\in\Xi\),
\begin{equation*}
        \ip{x}{h(x)+\bar g(x,\xi)}
        \ge c_{\rm diss}\norm{x}^{\beta}.
\end{equation*}

\item \label{itm:N5}\textbf{Subquadratic exponential integrability.}
When \(\beta=2\), this condition is not imposed.
When \(1\le\beta<2\), there exists \(\lambda_0>0\) such that
\[
\sup_{x\in\R^d}\sup_{\xi\in\Xi}
\E\exp\!\left(
\lambda_0\frac{\norm{g(x,\Phi(\xi,U_1))}}
{1+\norm{x}^{\beta-1}}
\right)<\infty.
\]
Here and below, \(\norm{x}^0=1\) by convention.

\item \label{itm:N6}\textbf{Nondegeneracy of noise at minimizer.}
If \(m=2\), this condition is not imposed.  If \(m>2\), assume that there exist constants \(\varepsilon_g>0\) and \(p_g>0\) such that
\[
        \inf_{\xi\in\Xi}
        \PP\!\left(\norm{g(0,\Phi(\xi,U_1))}\ge\varepsilon_g\right)
        \ge p_g.
\]

\item \label{itm:N7}
\textbf{Co-coercivity and noise perturbation.}
For \((x,\xi)\in\R^d\times\Xi\), set
\[
        \mathsf G(x,\xi):=h(x)+g(x,\xi).
\]
For every \(\xi\in\Xi\), the map \(x\mapsto\mathsf G(x,\xi)\) is
\(C^1\).
There exist constants \(L_{\mathsf G}\ge1\) and \(\theta\in[0,1)\)
such that, for every \(x,y\in\R^d\) and every \(\xi\in\Xi\),
\begin{equation}
\label{eq:sample-cocoercivity}
        \ip{x-y}{\mathsf G(x,\xi)-\mathsf G(y,\xi)}
        \ge
        L_{\mathsf G}^{-1}
        \norm{\mathsf G(x,\xi)-\mathsf G(y,\xi)}^2,
\end{equation}
and \(\bar g\) defined in
\eqref{eq:markov-conditional-mean-g} satisfies
\begin{equation}
\label{eq:mean-perturbation-control}
        \ip{x-y}{\bar g(x,\xi)-\bar g(y,\xi)}
        \ge
        -\theta\ip{x-y}{h(x)-h(y)}.
\end{equation}
\end{enumerate}
\end{assumption}


\begin{remark}
We now make a few clarifications about the above assumptions. 
\begin{itemize}
\item
Assumption~\ref{ass:N}\assitemref{itm:N4} is imposing that the
mean increment of the SGD iterate has an inward radial component of order
\(\norm{x}^{\beta-1}\). In particular, conditional centering
\(\bar g\equiv0\) makes this assumption immediate from
Assumption~\ref{ass:H}\assitemref{itm:H3b}.


\item
Assumption~\ref{ass:N}\assitemref{itm:N6} requires the noise to move the iterate away from the minimizer with uniformly positive probability. 

\item
In Assumption~\ref{ass:N}\assitemref{itm:N7}, the uniform co-coercivity \eqref{eq:sample-cocoercivity} is automatic when
\(\mathsf G(\cdot,\xi)\) is the gradient of a convex function with
Lipschitz gradient.  The condition \eqref{eq:mean-perturbation-control} allows the conditional
mean perturbation to oppose the deterministic drift, but not to cancel more
than a \(\theta\)-fraction of it.
\end{itemize}
\end{remark}

Before stating the first result, recall that the augmented space \(\mathsf Z=\R^d\times\Xi\) is equipped with the base norm \(\vert\cdot\vert_\alpha\) in \eqref{eq:aug-base-norm}.  The metrics \(d_{V,\alpha}\) and \(W_{V,\alpha}\) are those of Definition~\ref{def:induced-wasserstein}.  
For a metric \(d\) on a space \(E\), define
\[
        \mathcal P_1(E,d)
        :=
        \left\{\mu:\int_E d(z,z_0)\,\mu(dz)<\infty
        \text{ for some } z_0\in E\right\}.
\]

\begin{theorem}[Geometric ergodicity of constant-stepsize SGD]
\label{thm:main-markov}
Assume Assumptions~\ref{ass:H} and~\ref{ass:N}.  Then there exist
\(\alpha_0>0\) and \(c\in(0,1)\) such that, for every
\(0<\alpha\le\alpha_0\), one can choose a continuous weight function
\(V_\alpha:\mathsf Z\to[1,\infty)\) for which the augmented chain
\(Z_n=(X_n,\xi_n)\) admits a unique invariant law
\[
        \pi_\alpha\in\mathcal P_1(\mathsf Z,d_{V_\alpha,\alpha}).
\]
Moreover, 
\begin{equation}
\label{eq:markov-contraction-laws}
        W_{V_\alpha,\alpha}(\mu P_\alpha^n,\nu P_\alpha^n)
        \le
        (1-c\alpha^{m-1})^n
        W_{V_\alpha,\alpha}(\mu,\nu),
        \qquad n\ge0,
\end{equation}
for all initial laws \(\mu,\nu\) with finite \(W_{V_\alpha,\alpha}(\mu,\nu)\), where \(P_\alpha\) is the transition kernel of the augmented chain.
\end{theorem}

\begin{remark}
The contraction factor \(1-c\alpha^{m-1}\) reflects the local flatness of the
objective.  In the quadratic case \(m=2\), the deterministic curvature at
the minimizer gives the factor \(1-c\alpha\), matching the usual strongly
convex behavior.  When \(m>2\), the curvature vanishes exactly where the
invariant law concentrates.  Nevertheless, the locally flat objective and
the noise still ensure strict contraction for each fixed small stepsize,
with the contraction factor approaching one at order \(\alpha^{m-1}\).
\end{remark}

The proof of Theorem~\ref{thm:main-markov} is given in Section~\ref{sec:proof-fixed-markov}. 
Projecting the contraction of the augmented chain to the $X$-coordinate gives the following ordinary
\(W_1\) consequence, whose proof is given in Appendix~\ref{sec:proof-ordinary-W1}.

\begin{corollary}
\label{cor:ordinary-W1-markov}
Assume the hypotheses of Theorem~\ref{thm:main-markov}, and fix
\(\alpha\in(0,\alpha_0]\).  
Let $(\pi_\alpha)_X$ denote the $X$-marginal of the invariant law $\pi_\alpha$ from Theorem~\ref{thm:main-markov}. 
Let \(z_\star=(0,\xi_\star)\), with \(\xi_\star\) as in
Assumption~\ref{ass:N}\assitemref{itm:N3}.  There exists a sufficiently small
constant \(\kappa>0\), depending only on the constants in
Assumptions~\ref{ass:H} and~\ref{ass:N} and independent of \(\alpha\), such
that the following holds.  Define
\[
        \Gamma_\beta(r):=
        (1+r)^{\beta-1}
        \exp\!\left\{
        \kappa\bigl((1+r)^{2-\beta}-1\bigr)
        \right\},
        \qquad 1\le\beta\le2.
\]
Then
\[
        \mathcal P_1
        \bigl(\mathsf Z,d_{V_\alpha,\alpha}\bigr)
        =
        \left\{
        \mu:
        \int_{\mathsf Z}
        \Bigl[
        \Gamma_\beta(\norm{x})
        +\alpha^{-1}\norm{\xi-\xi_\star}
        \Bigr]\,\mu(dx,d\xi)<\infty
        \right\}.
\]
For every
\(\mu\in\mathcal P_1(\mathsf Z,d_{V_\alpha,\alpha})\) and every \(n\ge0\),
\[
        W_1\bigl((\mu P_\alpha^n)_X,(\pi_\alpha)_X\bigr)
        \le
        (1-c\alpha^{m-1})^n
        W_{V_\alpha,\alpha}(\mu,\pi_\alpha).
\]
\end{corollary}

The set $\mathcal P_1 \bigl(\mathsf Z,d_{V_\alpha,\alpha}\bigr)$ is explicitly characterized in the above corollary. 
In particular, every deterministic initial condition belongs to
\(\mathcal P_1(\mathsf Z,d_{V_\alpha,\alpha})\).  When \(1\le\beta<2\), membership in this
class requires an exponential moment of order \(2-\beta\) in the \(X\)-coordinate and a
first moment in the \(\xi\)-coordinate.  

\subsection{Small-stepsize scaling limit}

The invariant law $\pi_\alpha$ of the SGD iterates in the last subsection is implicit and generally difficult to characterize. Therefore, we turn to studying its scaling limit as $\alpha \downarrow 0$. 
To this end, we introduce two additional conditions. 


\begin{assumption}[Objective function]
\label{ass:Hscale}
Assume Assumption~\ref{ass:H} together with the following.
\begin{enumerate}[label=\textbf{(H\arabic*)}, ref=H\arabic*, leftmargin=3.6em, start=4]
\item \label{itm:H4}\textbf{Local gradient expansion.}
There exists \(H_0\in C^2(\R^d)\), convex and homogeneous of degree \(m\), such that \(H_0(0)=0\), \(\nabla H_0(0)=0\), and
\[
\nabla H(x)=\nabla H_0(x)+o(\norm{x}^{m-1})
\qquad\text{as }x\to0.
\]
\end{enumerate}
\end{assumption}

\begin{assumption}[Stochastic update]
\label{ass:Nscale}
Assume Assumption~\ref{ass:N} together with the following.  Let \(\pi_\Xi\) denote the invariant law of the driving chain $(\xi_n)_{n \ge 0}$, whose existence and uniqueness are proved in Lemma~\ref{lem:xi-invariant-law}.  The Markovian noise satisfies the following condition.
\begin{enumerate}[label=\textbf{(N\arabic*)}, ref=N\arabic*, leftmargin=3.6em, start=8]
\item \label{itm:N8}\textbf{Stationary centering and second moments.}
The gradient error $g(x,\xi)$ satisfies
\begin{equation}
\label{eq:stationary-centering-gkx}
        \int_\Xi g(x,\xi)\,\pi_\Xi(d\xi)=0,
        \qquad x\in\R^d,
\end{equation}
and 
\begin{equation}
\label{eq:g0-moment-gk}
        \int_\Xi
        \bigl(\norm{g(0,\xi)}^2+\norm{\xi-\xi_\star}^2\bigr)\,
        \pi_\Xi(d\xi)<\infty.
\end{equation}
\end{enumerate}
\end{assumption}



Let \((\xi_n)_{n\ge0}\) be the driving chain initialized according to its invariant law \(\pi_\Xi\), so that \((\xi_n)_{n\ge0}\) is stationary.  By Assumption~\ref{ass:Nscale}\assitemref{itm:N8}, \((g(0,\xi_n))_{n\ge0}\) is centered and stationary.  Define
\begin{equation}
\label{eq:GK-lag-series}
\Sigma:=\Gamma_0+\sum_{k=1}^{\infty}(\Gamma_k+\Gamma_k^\top),
        \qquad
        \Gamma_k:=\E[g(0,\xi_0)g(0,\xi_k)^\top].
\end{equation}
We will show that $\Sigma$ is the asymptotic covariance appearing in the scaling limit. To see that $\Sigma$ is a well-defined covariance matrix, Lemma~\ref{lem:poisson-gk-construction} shows that the series converges absolutely and that \(\Sigma\) is a finite positive semidefinite matrix. 

\begin{remark}
\label{rem:why-green-kubo}
The identity \eqref{eq:GK-lag-series} is the Green--Kubo representation
of the asymptotic covariance; see, for example,
\cite{benveniste1990,KushnerYin2003}.
For Markovian noise, the correlations between $\xi_k$'s accumulate over time.  Lemma~\ref{lem:poisson-gk-construction} uses the Poisson equation to decompose the correlated sequence into a conditionally centered increment plus a telescoping term.  The telescoping part is lower order in the generator calculation, while the quadratic variation of the centered increment gives \(\Sigma\) as in \eqref{eq:GK-lag-series}.
\end{remark}

The following theorem, proved in Section~\ref{sec:proof-clt-markov}, identifies both the scale $\alpha^{1/m}$ of the invariant law and its limiting distribution as the solution of a stochastic differential equation.

\begin{theorem}[Scaling limit of invariant law]
\label{thm:clt-iso-markov}
Assume Assumptions~\ref{ass:Hscale} and~\ref{ass:Nscale}.  Let \(H_0\) be as in Assumption~\ref{ass:Hscale}\assitemref{itm:H4} and set \(h_0:=\nabla H_0\).  For each sufficiently small \(\alpha>0\), let $(X_\infty^{(\alpha)},\xi_\infty^{(\alpha)})$ have law $\pi_\alpha$ from Theorem~\ref{thm:main-markov}, and set
\[
        Y_\alpha:=\alpha^{-1/m}X_\infty^{(\alpha)}.
\]
Then $Y_\alpha\Rightarrow Y_\infty$ as $\alpha\downarrow0$, where $Y_\infty$ is the unique invariant distribution of
\begin{equation}
\label{eq:limitSDE-markov}
        dY_t=-h_0(Y_t)\,dt+\Sigma^{1/2}\,dB_t,
\end{equation}
where $\Sigma$ is defined in \eqref{eq:GK-lag-series} and \(B=(B_t)_{t\ge0}\) is a standard \(d\)-dimensional Brownian motion.
\end{theorem}

We note that the covariance matrix $\Sigma$ is allowed to be degenerate, in which case the scaling limit is singular.

\begin{remark}
If \(m=2\) and \(h_0(y)=Ay\), then \eqref{eq:limitSDE-markov} defines an Ornstein--Uhlenbeck process.  Its invariant law is Gaussian with covariance \(C\) solving
\[
        AC+CA^\top=\Sigma.
\]
Thus Theorem~\ref{thm:clt-iso-markov} recovers the classical \(\sqrt\alpha\) scaling and Gaussian limit. 
For \(m>2\), the drift \(h_0\) is nonlinear, the scale becomes \(\alpha^{1/m}\), and the limiting invariant law is non-Gaussian.
\end{remark}

\subsection{Coordinate-separable objectives}
\label{subsec:separable-main}

The preceding results impose a common local flatness exponent in all
directions.  We now consider coordinate-separable objectives, for which the
flatness exponent, and hence the natural scale of the invariant law, may vary
across coordinates.  Suppose that
\begin{equation}
\label{eq:separable-model}
H(x)=\sum_{i=1}^d H_i(x_i),
\qquad
g(x,\xi)=\bigl(g_1(x_1,\xi),\ldots,g_d(x_d,\xi)\bigr).
\end{equation}
Writing \(h_i=H_i'\), the SGD recursion becomes
\begin{equation}
\label{eq:separable-sgd}
X_{n+1,i}
=X_{n,i}-\alpha\{h_i(X_{n,i})+g_i(X_{n,i},\xi_{n+1})\},
\qquad 1\le i\le d,
\end{equation}
where \(\xi_{n+1}=\Phi(\xi_n,U_{n+1})\) is the driving-chain recursion from
the beginning of this section.  Thus the coordinates may remain dependent
through the common driving chain.

For each \(1\le i\le d\), assume that the one-dimensional pair
\((H_i,g_i)\) satisfies the one-dimensional versions of
Assumptions~\ref{ass:H} and~\ref{ass:N}, with parameters
\((m_i,\beta_i)\).  The constants in these assumptions may depend on \(i\).
Applying the construction of Section~\ref{subsec:roadmap-fixed} with \(d=1\)
and \((H,g,m,\beta)\) replaced by \((H_i,g_i,m_i,\beta_i)\), on
\(\R\times\Xi\) equipped with the common base norm
\[
        \lvert(x_i,\xi)\rvert_\alpha
        :=|x_i|+\alpha^{-1}\norm{\xi},
\]
gives a weight \(V_{\alpha,i}\) and an induced metric
\(d_{i,\alpha}:=d_{V_{\alpha,i},\alpha}^{(i)}\).  Define the additive metric
on \(\R^d\times\Xi\) by
\begin{equation}
\label{eq:separable-metric}
d_{{\rm sep},\alpha}((x,\xi),(y,\eta))
:=\sum_{i=1}^d d_{i,\alpha}((x_i,\xi),(y_i,\eta)),
\end{equation}
and let \(W_{{\rm sep},\alpha}\) denote the Wasserstein--1 distance with cost
\(d_{{\rm sep},\alpha}\).  Set
\[
        m_{\max}:=\max_{1\le i\le d}m_i.
\]

The following corollary extends Theorem~\ref{thm:main-markov} to this
separable setting.

\begin{corollary}[Geometric ergodicity for separable objectives]
\label{cor:sep-markov-fixed}
Suppose that \eqref{eq:separable-model} holds and that, for every
\(1\le i\le d\), the pair \((H_i,g_i)\) satisfies the one-dimensional
versions of Assumptions~\ref{ass:H} and~\ref{ass:N}.  Then there exist
\(\alpha_0\in(0,1)\) and \(c_0\in(0,1)\) such that, for every
\(\alpha\in(0,\alpha_0]\), the augmented chain \(Z_n=(X_n,\xi_n)\)
admits a unique invariant law
\[
        \pi_\alpha^{\rm sep}
        \in\mathcal P_1(\R^d\times\Xi,d_{{\rm sep},\alpha}).
\]
Moreover, if \(P_\alpha^{\rm sep}\) denotes its transition kernel, then
\begin{equation}
\label{eq:sep-markov-contraction}
W_{{\rm sep},\alpha}
\bigl(\mu(P_\alpha^{\rm sep})^n,\nu(P_\alpha^{\rm sep})^n\bigr)
\le
(1-c_0\alpha^{m_{\max}-1})^n
W_{{\rm sep},\alpha}(\mu,\nu),
\qquad n\ge0,
\end{equation}
for all probability laws \(\mu,\nu\) for which the right-hand side is finite.
\end{corollary}

The contraction rate is determined by \(m_{\max}\): the coordinates with the
largest flatness exponent have the largest contraction factor and hence govern
the convergence of the full chain.  Corollary~\ref{cor:sep-markov-fixed} is
proved in Appendix~\ref{sec:proof-sep-markov-fixed}.

\begin{remark}
\label{rem:sep-independent-coordinates}
The setup above also includes independent coordinate driving chains as a special case.  More
precisely, suppose that \(\Xi=\prod_{i=1}^d\Xi_i\), the driving map and
innovations act independently across coordinates, and each \(g_i\) depends
only on the \(i\)th coordinate of \(\xi\).  Then the transition kernel is a
product kernel.  Hence the product of the one-dimensional invariant laws is
invariant and, by the uniqueness in Corollary~\ref{cor:sep-markov-fixed},
equals \(\pi_\alpha^{\rm sep}\).
\end{remark}

We next turn to the scaling limit.  For each coordinate, assume in addition
the one-dimensional versions of Assumptions~\ref{ass:Hscale}
and~\ref{ass:Nscale}.  Let \(H_{i,0}\) be the homogeneous function appearing
in the one-dimensional version of
Assumption~\ref{ass:Hscale}\assitemref{itm:H4}, and set
\(h_{i,0}:=H_{i,0}'\).  Initialize the driving chain in its invariant law
\(\pi_\Xi\), and define
\[
        g_0^{\rm sep}(\xi)
        :=\bigl(g_1(0,\xi),\ldots,g_d(0,\xi)\bigr),
        \qquad \xi\in\Xi.
\]
Then \((g_0^{\rm sep}(\xi_n))_{n\ge0}\) is centered and stationary.  Its
asymptotic covariance is
\begin{equation}
\label{eq:separable-GK-series}
\Sigma_{\rm sep}
:=\Gamma_0^{\rm sep}
+\sum_{k=1}^{\infty}
\bigl(\Gamma_k^{\rm sep}+(\Gamma_k^{\rm sep})^\top\bigr),
\qquad
\Gamma_k^{\rm sep}
:=\E\bigl[g_0^{\rm sep}(\xi_0)g_0^{\rm sep}(\xi_k)^\top\bigr].
\end{equation}
The coordinatewise Poisson construction in
Lemma~\ref{lem:poisson-gk-construction} shows that the series converges
absolutely and that \(\Sigma_{\rm sep}\) is finite and positive semidefinite.
Its off-diagonal entries retain the temporal cross-covariances created by a
common driving chain.
Let
\begin{equation}
        \mathcal I_*:=\{i:m_i=m_{\max}\},
        \qquad
        P_*:=\operatorname{diag}\bigl(\1_{\{i\in\mathcal I_*\}}\bigr),
        \qquad
        \Sigma_*:=P_*\Sigma_{\rm sep}P_*. \label{eq:def-Sigma-star}
\end{equation}
Both \(\Sigma_{\rm sep}\) and \(\Sigma_*\) are allowed to be
degenerate.

For each \(1\le i\le d\), the \((x_i,\xi)\)-marginal of
\(\pi_\alpha^{\rm sep}\) is invariant for the one-dimensional augmented
chain associated with \((H_i,g_i)\).  Therefore
Theorem~\ref{thm:clt-iso-markov} applied to this chain gives
\[
        \alpha^{-1/m_i}X_{\infty,i}^{(\alpha),\rm sep}
        \Rightarrow Y_{\infty,i},
\]
where \(Y_{\infty,i}\) has the unique invariant law of the one-dimensional
diffusion
\begin{equation*}
        dY_{i,t}
        =-h_{i,0}(Y_{i,t})\,dt
        +(\Sigma_{\rm sep})_{ii}^{1/2}\,dB_{i,t},
\end{equation*}
and \(B_i\) is a standard one-dimensional Brownian motion.  Thus coordinate
\(i\) has natural scale \(\alpha^{1/m_i}\).

\begin{corollary}[Scaling limit for separable objectives]
\label{cor:sep-markov-scaling}
Suppose that \eqref{eq:separable-model} holds and that, for every
\(1\le i\le d\), the pair \((H_i,g_i)\) satisfies the one-dimensional
versions of Assumptions~\ref{ass:Hscale} and~\ref{ass:Nscale}.  Let
\((X_\infty^{(\alpha),\rm sep},\xi_\infty^{(\alpha),\rm sep})\) have law
\(\pi_\alpha^{\rm sep}\), and set
\[
        h_0^{\rm sep}(y)
        :=\bigl(h_{1,0}(y_1),\ldots,h_{d,0}(y_d)\bigr).
\]
Then, under the common normalization determined by \(m_{\max}\),
\[
        \alpha^{-1/m_{\max}}X_\infty^{(\alpha),\rm sep}
        \Rightarrow Y_\infty^{\rm sep},
\]
where \(Y_\infty^{\rm sep}\) has the unique invariant law of the
\(d\)-dimensional diffusion
\begin{equation}
\label{eq:sep-limitSDE-d}
        dY_t
        =-h_0^{\rm sep}(Y_t)\,dt
        +\Sigma_*^{1/2}\,dB_t.
\end{equation}
Here, $\Sigma_*$ is defined in \eqref{eq:def-Sigma-star} and \(B\) is a standard \(d\)-dimensional Brownian motion.
\end{corollary}

For \(i\notin\mathcal I_*\), the \(i\)th component of
\eqref{eq:sep-limitSDE-d} has no Brownian forcing and follows the deterministic
flow \(dY_{i,t}=-h_{i,0}(Y_{i,t})\,dt\), whose unique invariant law is
\(\delta_0\).  Thus, under the common normalization, only the coordinates
with the largest flatness exponent can have a nonzero limit. 
If \(m_i=m\) for every \(i\), then \(P_*=I_d\) and
\(\Sigma_*=\Sigma_{\rm sep}\).  Corollary~\ref{cor:sep-markov-scaling}
then gives
\(
        \alpha^{-1/m}X_\infty^{(\alpha),\rm sep}
        \Rightarrow Y_\infty^{\rm sep},
\)
where \eqref{eq:sep-limitSDE-d} becomes
\(
        dY_t
        =-h_0^{\rm sep}(Y_t)\,dt
        +\Sigma_{\rm sep}^{1/2}\,dB_t.
\) 
The proof of
Corollary~\ref{cor:sep-markov-scaling} is given in
Appendix~\ref{sec:proof-sep-markov-scaling}.

\begin{remark}
\label{rem:sep-scaling-independent}
Under the independent-coordinate conditions of
Remark~\ref{rem:sep-independent-coordinates}, the naturally rescaled vector
satisfies
\[
        \left(
        \frac{X_{\infty,1}^{(\alpha),\rm sep}}{\alpha^{1/m_1}},\ldots,
        \frac{X_{\infty,d}^{(\alpha),\rm sep}}{\alpha^{1/m_d}}
        \right)
        \Rightarrow (Y_{\infty,1},\ldots,Y_{\infty,d}),
\]
where the limiting coordinates are independent. See Appendix~\ref{sec:pf-rmk-rem:sep-scaling-independent} for more details.
\end{remark}

\section{Applications and numerical experiments}
\label{sec:applications-numerics}

This section illustrates the roles of the local exponent \(m\) and the tail
exponent \(\beta\), and reports numerical experiments for the invariant-law
predictions.  Section~\ref{subsec:weak-quantile-flat} uses quantile and tail-risk
estimation to interpret \(m\) as a local flatness exponent and
illustrates its consequences for local scaling.  
Section~\ref{subsec:robust-subquadratic-tail} uses robust and
logistic losses to interpret \(\beta\) as a global tail-growth exponent; these examples are locally quadratic with \(m=2\), emphasizing
that \(\beta\) controls large excursions rather than the small-stepsize
normalization.

\subsection{Local flatness in quantile and tail-risk estimation}
\label{subsec:weak-quantile-flat}

Quantile and tail-risk estimation problems provide a statistical interpretation of the
local flatness exponent.  The quantile-loss formulation goes back to Koenker and
Bassett~\cite{KoenkerBassett1978}, while Knight~\cite{Knight1998} studied its
asymptotic behavior under general local conditions on the distribution function,
including the higher-order crossings considered below.  

\paragraph{Background.}
Let $Y$ be a scalar response, let $F$ denote its distribution
function, and let $t_\star$ be a $\tau$-quantile, so that $F(t_\star)=\tau$.
The population quantile objective is
\[
        H_\tau(t):=\E\rho_\tau(Y-t),
        \qquad
        \rho_\tau(u):=u(\tau-\1_{\{u<0\}}).
\]
At continuity points of $F$, the population gradient is \(F(t)-\tau\).
Consequently,
\[
        t_\star\in\operatorname*{argmin}_{t\in\R}H_\tau(t).
\]
The local objective behavior is exactly the local crossing behavior of
$F$ at the target quantile.  Suppose that the following higher-order crossing
condition holds: for some $m\ge2$ and constants $c_+,c_->0$,
\begin{equation}
        F(t_\star+u)-\tau
        =c_+u_+^{m-1}-c_-u_-^{m-1}+o(|u|^{m-1}),
        \qquad u\to0,
        \label{eq:weak-quantile-crossing-v3}
\end{equation}
where $u_+=\max\{u,0\}$ and $u_-=\max\{-u,0\}$.  Integrating the population
gradient gives
\begin{equation*}
        H_\tau(t_\star+u)-H_\tau(t_\star)
        =\frac{c_+}{m}u_+^m+\frac{c_-}{m}u_-^m+o(|u|^m).
\end{equation*}
The positive-density case corresponds to \(m=2\).  If the density vanishes at
the target quantile like \(|u|^{m-2}\), then the equation \(F(t)=\tau\) still has the unique solution \(t=t_\star\), but
\(F(t)-\tau\) vanishes at order \(m-1\) near \(t_\star\). As a result, the population risk is \(m\)-flat.  This setup is closely related to the regularly varying local regime studied by Knight~\cite{Knight1998}.

The Rockafellar--Uryasev variational representation of
CVaR~\cite{RockafellarUryasev2000} gives an analogous population objective.  For a
scalar loss $L$ and confidence level $\tau\in(0,1)$, this representation is
\[
        \operatorname{CVaR}_\tau(L)
        =\min_{t\in\R}C_\tau(t),
        \qquad
        C_\tau(t):=t+\frac{1}{1-\tau}\E(L-t)_+.
\]
At continuity points of $F_L$,
\[
        C_\tau'(t)
        =1-\frac{1}{1-\tau}\PP(L>t)
        =\frac{F_L(t)-\tau}{1-\tau}.
\]
Hence, if $t_\star$ is the unique
$\tau$-quantile of $L$, then
\[
        \operatorname*{argmin}_{t\in\R}C_\tau(t)
        =\{t_\star\}.
\]
The local expansion in \eqref{eq:weak-quantile-crossing-v3} then implies
\[
        C_\tau(t_\star+u)-C_\tau(t_\star)
        =\frac{c_+}{m(1-\tau)}u_+^m
        +\frac{c_-}{m(1-\tau)}u_-^m+o(|u|^m).
\]
Thus, combining the variational representation with the local expansion in
\eqref{eq:weak-quantile-crossing-v3} shows that the quantile objective and the
variational objective for CVaR have the same local exponent \(m\).

\paragraph{Invariant law for median estimation.}
To illustrate the local effect of this flatness, consider median estimation with distribution function
\begin{equation}
        F_m(x)=\frac12+\frac12\operatorname{sgn}(x)|x|^{m-1},
        \qquad |x|\le1.
        \label{eq:weak-median-family-v3}
\end{equation}
Thus, $\tau=1/2$ and the median is zero.
The stochastic gradient at the minimizer has nonzero variance, while the restoring
force is of order $|x|^{m-1}$.  Their balance gives the scale
$\alpha^{1/m}$ and a nonlinear diffusion when $m>2$.  The raw
pinball score is nonsmooth, so the finite-state calculation below illustrates
this mechanism rather than directly applying
Theorem~\ref{thm:clt-iso-markov}. 

The finite-state example in Figure~\ref{fig:quantile-flat-scaling}
uses a two-sample lazy version of the recursion
\begin{equation}
        X_{n+1}=X_n-
        \alpha\left\{\frac12\bigl(\1_{\{Y_{n+1,1}\le X_n\}}
        +\1_{\{Y_{n+1,2}\le X_n\}}\bigr)-\frac12\right\}.
        \label{eq:lazy-weak-median-v3}
\end{equation}
For $\alpha=2^{-j}$, the chain lives on the grid
$\{-1,-1+\alpha/2,\ldots,1\}$.  If $x_k=-1+k\alpha/2$, then its right and left
transition probabilities are
\[
        p_k=(1-F_m(x_k))^2,
        \qquad
        q_k=F_m(x_k)^2,
\]
and its invariant law is computed exactly from the detailed balance equation
\[
        \frac{\pi_\alpha(k+1)}{\pi_\alpha(k)}
        =\frac{p_k}{q_{k+1}}.
\]
The
variance of the noise at the minimizer in \eqref{eq:lazy-weak-median-v3} is
\(\Sigma=1/8\), and the limiting diffusion is
\[
        dZ_t=-\frac12 Z_t|Z_t|^{m-2}\,dt+\sqrt{\Sigma}\,dB_t.
\]
Let \(Z_\infty\) have the invariant law of this diffusion.  Its density is
\[
        p_m(z)=\frac{m(1/(m\Sigma))^{1/m}}{2\Gamma(1/m)}
        \exp\!\left(-\frac{|z|^m}{m\Sigma}\right).
\]
In particular, for $m=4$, $p_4(z)\propto \exp(-2z^4)$, which is not Gaussian.
Figure~\ref{fig:quantile-flat-scaling} confirms all these theoretical predictions.

\begin{figure}
    \centering
    \includegraphics[
        width=\textwidth,
        height=0.38\textheight,
        keepaspectratio
    ]{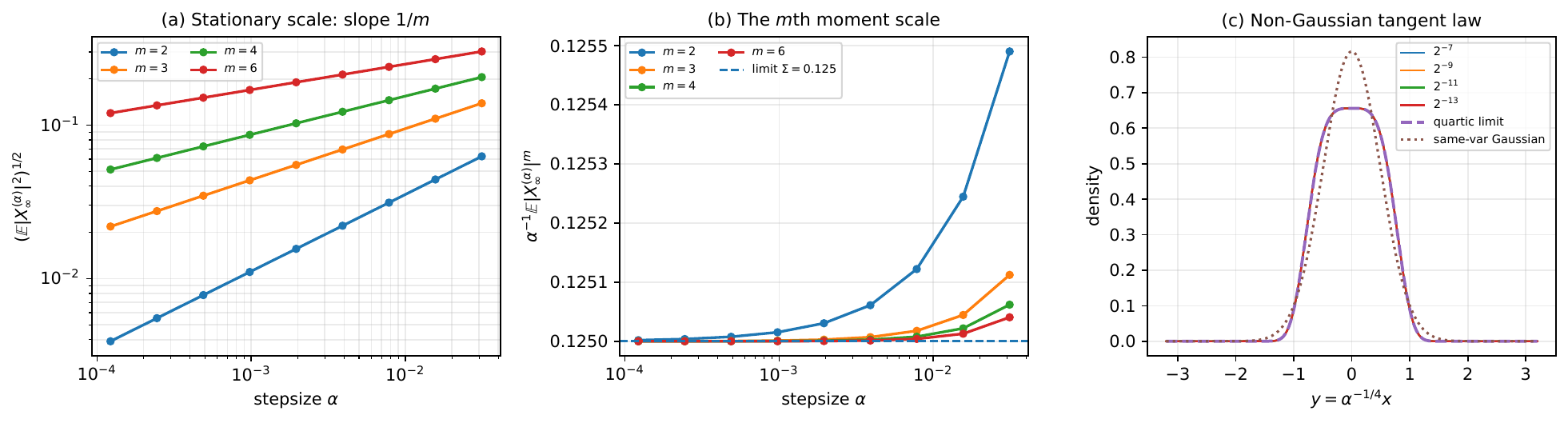}
    \caption{SGD for median estimation.  Panel (a) shows the
    $L_2$ error for several crossing orders \(m\); the fitted
    slopes match the theoretical exponent \(1/m\).  Panel (b) checks the
    predicted local moment scale: in this
    model \(\alpha^{-1}\E|X_\infty^{(\alpha)}|^m\to\Sigma=1/8\).  Panel (c)
    shows the convergence of \(\alpha^{-1/4}X_\infty^{(\alpha)}\) for \(m=4\) to
    the quartic invariant density of the limiting diffusion, together with a
    Gaussian of the same variance for comparison.}
    \label{fig:quantile-flat-scaling}
\end{figure}

\paragraph{Markovian covariance and unequal exponents.}
A related finite-state experiment in Figure~\ref{fig:markovian-separable-effects}(a) verifies the effect of temporal dependence
on the asymptotic covariance.  Let \(\rho\in[0,1)\), and let \((S_n)\) be a
stationary two-state Markov chain on \(\{-1,1\}\) with transition probabilities
\[
        \PP(S_{n+1}=s\mid S_n=s)=\frac{1+\rho}{2},
        \qquad
        \PP(S_{n+1}=-s\mid S_n=s)=\frac{1-\rho}{2},
        \qquad s\in\{-1,1\}.
\]
Its stationary distribution is uniform on \(\{-1,1\}\), and
\(\E[S_nS_{n+k}]=\rho^k\).  Independently, let \((R_n)\) be i.i.d.\ with
\[
        \PP(R_n\le r)=r^{m-1},\quad 0\le r\le1,
\]
and set \(Y_n=S_nR_n\).  The two-state chain can be realized on \([-1,1]\)
by taking \(\Phi(s,U)=s\)
with probability \(\rho\), and \(\Phi(s,U)=1\) or \(-1\), each with
probability \((1-\rho)/2\).  Hence \(\E L_\Phi(U)=\rho<1\), where \(L_\Phi(U)\) is the random Lipschitz coefficient from
Assumption~\ref{ass:N}\assitemref{itm:N1}.  The marginal
law of $Y_n$ is still \eqref{eq:weak-median-family-v3}, but at the median
\(
        \1_{\{Y_n\le0\}}-\frac12=-\frac12S_n.
\)
Thus the asymptotic covariance in the Green--Kubo representation \eqref{eq:GK-lag-series} is
\[
        \Sigma
        =\frac14\left(1+2\sum_{k\ge1}\rho^k\right)
        =\frac14 \cdot \frac{1+\rho}{1-\rho}.
\]
For $m=4$, the invariant law satisfies
$\E Z_\infty^4=\Sigma$, and hence
\[
        \alpha^{-1}\E|X_\infty^{(\alpha)}|^4
        \longrightarrow \Sigma.
\]
Figure~\ref{fig:markovian-separable-effects}(a) shows that the rescaled moment approaches the predicted value $\Sigma = \Sigma_{\rm GK}$ as $\alpha \downarrow 0$.

Moreover, we consider the separable case by taking the product of two scalar recursions of the
form \eqref{eq:lazy-weak-median-v3}, with \(m_1=2\) and \(m_2=4\).
Under the common scale associated with \(m_2\), only the second coordinate
has a nonzero limit.  Figure~\ref{fig:markovian-separable-effects}(b) confirms this
conclusion.

\begin{figure}
    \centering
    \includegraphics[
        width=0.92\textwidth,
        height=0.38\textheight,
        keepaspectratio
    ]{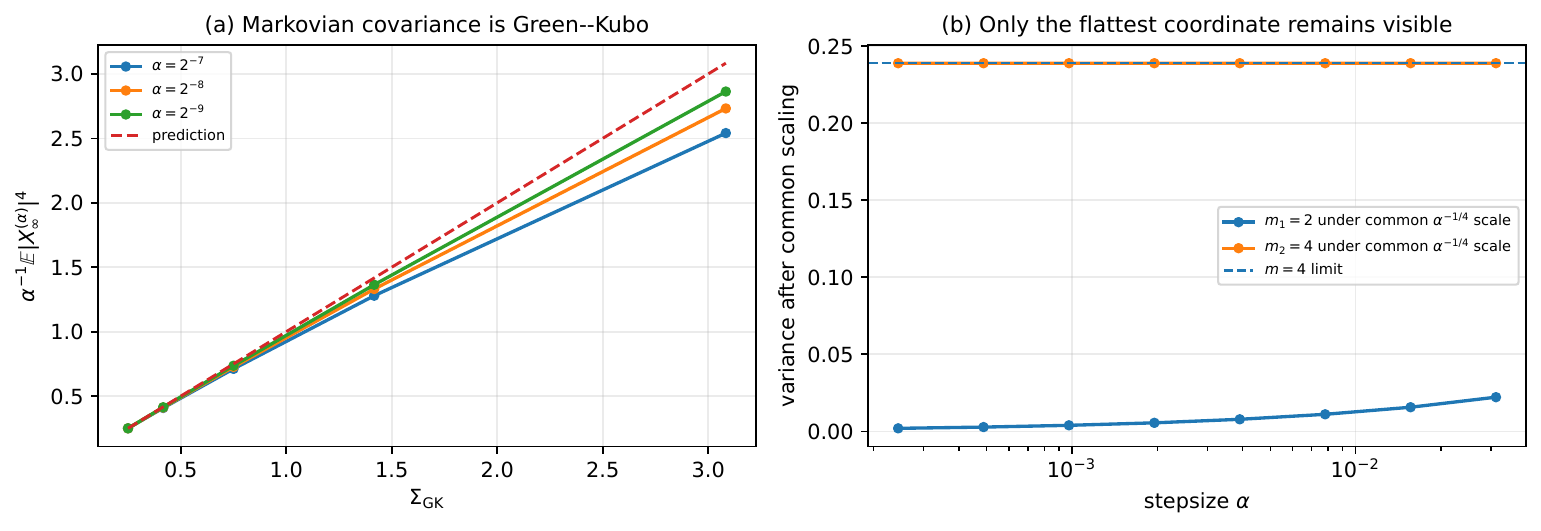}
    \caption{Two consequences of the scaling limit of invariant laws.
    Panel (a) uses a Markovian sign stream and verifies that the fourth moment
    constant in the \(m=4\) quantile example is governed by the asymptotic
    covariance rather than by the marginal variance.  Panel (b) considers a
    separable two-coordinate model with exponents \(m_1=2\) and \(m_2=4\); under
    the common \(\alpha^{-1/4}\) scale, the \(m_1=2\) coordinate converges to zero while
    the \(m_2=4\) coordinate converges to a nondegenerate limit.}
    \label{fig:markovian-separable-effects}
\end{figure}

\subsection{Subquadratic tails for robust and logistic losses}
\label{subsec:robust-subquadratic-tail}

The tail exponent \(\beta\) describes the growth of the objective at
infinity.  The generalized Charbonnier family below realizes every
\(\beta\in[1,2]\) \cite{Charbonnier1997,Barron2019}; the classical Huber loss is a related piecewise-defined robust loss with linear tails \cite{Huber1964}.  Under the bounded, nondegenerate, nonseparable model verified below, the population logistic risk has linear coercive growth corresponding to \(\beta=1\)
\cite{BartlettJordanMcAuliffe2006}.

\paragraph{Robust location estimation.}
Consider a scalar location model.
For \(1\le\beta\le2\), define the generalized Charbonnier loss
\[
        \rho_\beta(u):=\frac{(1+u^2)^{\beta/2}-1}{\beta},
        \qquad
        \psi_\beta(u):=\rho_\beta'(u)
        =u(1+u^2)^{\beta/2-1}.
\]
The endpoint \(\beta=1\) is the pseudo-Huber/Charbonnier loss, while
\(\beta=2\) is ordinary least squares.  In the scalar location model
\[
        Y=\theta_\star+\varepsilon,
        \qquad
        H_\beta(\theta):=\E\rho_\beta(\theta-Y),
\]
recenter \(x=\theta-\theta_\star\).  If the error distribution is symmetric,
then the constant-stepsize SGD becomes
\begin{equation*}
        X_{n+1}=X_n-\alpha\psi_\beta(X_n-\varepsilon_{n+1}).
\end{equation*}
The population drift is
\[
        h_\beta(x)=\E\psi_\beta(x-\varepsilon).
\]
Since
\[
        \psi_\beta'(u)
        =
        (1+u^2)^{\beta/2-2}\{1+(\beta-1)u^2\},
\]
we have
\[
        h_\beta'(0)=\E\psi_\beta'(-\varepsilon)>0
\]
for every nondegenerate bounded symmetric error.  Thus the minimizer is
locally quadratic with \(m=2\).  Consequently the invariant law remains
the classical Gaussian limit,
\[
        \alpha^{-1/2}X_\infty^{(\alpha)}
        \Rightarrow
        N\left(0,\frac{\Var(\psi_\beta(-\varepsilon))}{2h_\beta'(0)}\right).
\]

The difference from least squares is not local but global.  As \(|u|\to\infty\),
\[
        \rho_\beta(u)\sim \frac{|u|^\beta}{\beta},
        \qquad
        \psi_\beta(u)\sim \operatorname{sgn}(u)|u|^{\beta-1}.
\]
Therefore, for bounded errors,
\[
        H_\beta(x)\sim\frac{|x|^\beta}{\beta},
        \qquad
        h_\beta(x)\sim \operatorname{sgn}(x)|x|^{\beta-1},
        \qquad |x|\to\infty.
\]
For \(1\le\beta<2\), the behavior at infinity of the deterministic flow \(\dot x_t=-h_\beta(x_t)\) is described by
\begin{equation}
        \frac{d}{dt}|x_t|^{2-\beta}
        \sim -(2-\beta),
        \qquad |x_t|\to\infty.
        \label{eq:robust-tail-coordinate-v3}
\end{equation}
This calculation explains why the weight function for subquadratic tails involves 
\(|x|^{2-\beta}\).

For the numerical illustration, take
\[
        \varepsilon\sim0.9\,\operatorname{Unif}[-1,1]
        +0.1\,\operatorname{Unif}[-10,10].
\]
This bounded symmetric mixture satisfies the assumptions used above.
Figure~\ref{fig:robust-subquadratic-tail} reports the corresponding numerical results.

\begin{figure}[!t]
    \centering
    \includegraphics[
        width=0.92\textwidth,
        height=0.62\textheight,
        keepaspectratio
    ]{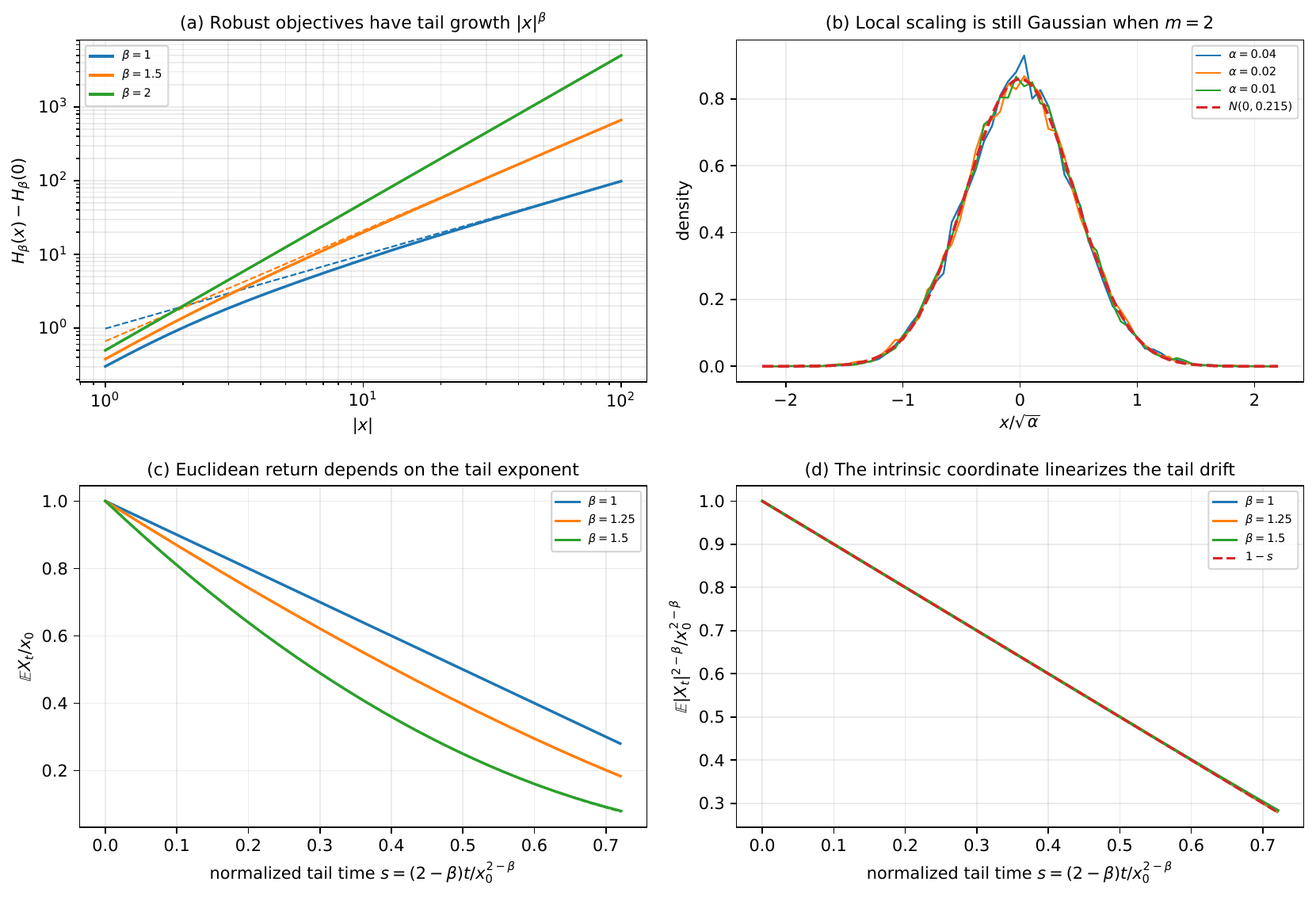}
    \caption{Robust location estimation with generalized Charbonnier losses.
    Panel (a) verifies the population tail \(H_\beta(x)\asymp |x|^\beta\).
    Panel (b) verifies the Gaussian scaling limit of \(X_\infty^{(\alpha)}/\sqrt\alpha\) for the pseudo-Huber loss with \(\beta=1\) because it remains locally
    quadratic.  Panel (c) shows that the decay of the normalized state \(\mathbb E X_t/x_0\), starting from a large \(x_0>0\), depends on the tail exponent.  Panel (d) normalizes the same trajectories by the tail scaling \(|x|^{2-\beta}\), showing that they follow the linear prediction in
    \eqref{eq:robust-tail-coordinate-v3}.}
    \label{fig:robust-subquadratic-tail}
\end{figure}
\FloatBarrier

\paragraph{Logistic regression.}
Finally, we verify theoretically that our results apply to SGD for logistic regression under a correctly specified model with bounded, nondegenerate design and nonvanishing label noise.
Let \((Y_n,Z_n)_{n\ge1}\) be i.i.d.\ data with covariates satisfying
\[
        \norm{Z_n} \le R,
        \qquad
        \E[Z_n Z_n^\top]\succeq\lambda I_d
\]
for some \(R,\lambda>0\).  With \(\sigma(t):=(1+e^{-t})^{-1}\), suppose that
\[
        \PP(Y_n=1\mid Z_n)=\sigma(Z_n^\top\theta_\star),
        \qquad Y_n\in\{-1,1\}.
\]
Set \(S_n:=Y_nZ_n\).  For the logistic loss
\[
        \ell(\theta;S):=\log(1+e^{-S^\top\theta}),
\]
the unique population minimizer is \(\theta_\star\).  After recentering
\(x=\theta-\theta_\star\), the sample and population gradients are
\[
        \mathsf G(x,S)
        :=-\frac{S}{1+e^{S^\top(\theta_\star+x)}},
        \qquad
        h(x):=\E\mathsf G(x,S).
\]
The centered population objective is \(C^2\) and convex, with \(h(0)=0\), and
hence satisfies Assumption~\ref{ass:H}\assitemref{itm:H1}; moreover,
\(\norm{\mathsf G(x,S)}\le R\).

Since \(\lvert Z^\top\theta_\star\rvert\le R\norm{\theta_\star}\), setting
\(q:=\sigma(-R\norm{\theta_\star})\) gives
\(
        q\le \PP(Y=1\mid Z)\le 1-q .
\)
Hence, for every \(v\in\Sph^{d-1}\),
\[
        \gamma(v):=\E[(-S^\top v)_+]
        \ge q\E|Z^\top v|
        \ge \frac{q}{R}\E[(Z^\top v)^2]
        \ge \frac{q\lambda}{R}.
\]
Moreover, writing \(B:=R\norm{\theta_\star}\), one has the uniform bound
\[
        \sup_{v\in\Sph^{d-1}}
        \left|\frac{1}{r}\ip{rv}{h(rv)}-\gamma(v)\right|
        \le \frac{e^B}{er}.
\]
Thus, for all sufficiently large \(r\),
\[
        \ip{rv}{h(rv)}\ge\frac{q\lambda}{2R}r,
        \qquad
        \norm{h(rv)}\le R,
\]
which verifies Assumption~\ref{ass:H}\assitemref{itm:H3b} with \(\beta=1\).
The above condition and the nondegenerate design also imply, by a standard covering and
concentration argument, that an i.i.d.\ sample is not linearly separable with
probability tending to one as its size grows, for fixed \(d\).

The population Hessian at \(\theta_\star\) satisfies
\[
        A:=\E\!\left[
        \sigma(Z^\top\theta_\star){1-\sigma(Z^\top\theta_\star)}ZZ^\top
        \right]
        \succeq q(1-q)\lambda I_d.
\]
Continuity of the population Hessian verifies
Assumption~\ref{ass:H}\assitemref{itm:H2} with \(m=2\), while
\[
        h(x)=Ax+o(\norm{x}),
        \qquad x\to0,
\]
verifies Assumption~\ref{ass:Hscale}\assitemref{itm:H4} with
\(H_0(x)=\tfrac12x^\top A x\).

For each observation,
\[
        \nabla_x\mathsf G(x,S)
        =
        \frac{e^{S^\top(\theta_\star+x)}}
        {(1+e^{S^\top(\theta_\star+x)})^2}SS^\top
        \preceq \frac{R^2}{4}I_d.
\]
Thus each sample loss is convex with \((R^2/4)\)-Lipschitz gradient, and the
Baillon--Haddad theorem~\cite{BauschkeCombettes2010} verifies the co-coercivity part of
Assumption~\ref{ass:N}\assitemref{itm:N7} with any
\(L_{\mathsf G}\ge\max\{1,R^2/4\}\).  Although the displayed Hessian has rank
at most one, no samplewise strict curvature is required.  With
\(g(x,S):=\mathsf G(x,S)-h(x)\), i.i.d.\ sampling gives
\(\E g(x,S)=0\) and \(\norm{g(x,S)}\le2R\).  Thus the mean-perturbation part
of \assitemref{itm:N7} holds with \(\theta=0\), and the tail estimate above
verifies Assumption~\ref{ass:N}\assitemref{itm:N4}.  The uniform bound on
\(g\) verifies \assitemref{itm:N5}, while centering and boundedness verify
Assumption~\ref{ass:Nscale}\assitemref{itm:N8}; \assitemref{itm:N6} is not
imposed because \(m=2\).  Finally, the i.i.d.\ embedding in
Remark~\ref{rem:iid-specialization} verifies \assitemref{itm:N1}--\assitemref{itm:N3}.
Thus Assumptions~\ref{ass:Hscale}
and~\ref{ass:Nscale} hold with \(m=2\) and \(\beta=1\), so our main results apply.

\section{Proofs of the main results}
\label{sec:main-proofs}

We prove the main results in this section. In Section~\ref{subsec:roadmap-fixed}, 
we show that for any sufficiently small constant stepsize $\alpha$, the augmented Markov chain admits a unique invariant law and contracts geometrically in the Wasserstein distance, thereby establishing Theorem~\ref{thm:main-markov}.
In Section~\ref{subsec:proof-roadmap-scaling}, with the fixed-$\alpha$ invariant law in hand, we derive its scaling limit as $\alpha \downarrow 0$ and prove Theorem~\ref{thm:clt-iso-markov}.
Additional technical proofs are deferred to Appendices~\ref{sec:constant-step-technical-estimates}, \ref{sec:scaling-technical-estimates}, and \ref{sec:separable-proofs-appendix}.



\subsection{Analysis of constant-stepsize SGD}
\label{subsec:roadmap-fixed}
\label{sec:proof-fixed-markov}

At fixed \(\alpha\), we use the framework of Qu,
Blanchet, and Glynn~\cite{QBG2025}, where the main task is to
prove a one-step contraction for a certain weight-induced metric. More specifically, to prove
Theorem~\ref{thm:main-markov}, it suffices to construct a metric
\(d_{V_\alpha,\alpha}\) induced by a weight function $V_\alpha$ such that
\begin{equation}
        \E d_{V_\alpha,\alpha}(F_{U_1}(z),F_{U_1}(z'))
        \le
        (1-c_0 \alpha^{m-1})d_{V_\alpha,\alpha}(z,z') \label{eq:main-contraction-aug}
\end{equation}
for some constant \(c_0>0\) and all \(z,z'\in\mathsf Z\).
Proposition~\ref{prop:CD-to-ergodic} below then converts this into the
invariant law, uniqueness, and geometric convergence.

We remark that the technique for establishing the contraction
in~\cite{QBG2025} cannot be applied directly here.  In short, it controls each realized random map through its local Lipschitz modulus
before averaging.  For the SGD maps considered here, a stochastic-gradient
Jacobian may be rank deficient, so the corresponding update map can have
local Lipschitz modulus one and therefore may not be a strict contraction.  Instead, we retain the induced metric but prove
the desired contraction in expectation via a direction kernel (cf.\ Lemma~\ref{lem:fixed-conditional-drift}).

For a metric \(d\) on a space \(E\), we write \(W_1^d\) for the Wasserstein--1
distance induced by \(d\), namely
\[
W_1^d(\mu,\nu)
:=
\inf_{\gamma\in\Pi(\mu,\nu)}
\int_{E\times E} d(z,z')\,\gamma(dz,dz').
\]

\begin{proposition}
\label{prop:CD-to-ergodic}
Let \((E,d)\) be a Polish space and let
\(Z_{n+1}=F_{U_{n+1}}(Z_n)\) be a random-map Markov chain with transition kernel
\(P\).  Suppose that there exist \(r\in(0,1)\) and \(z_\star\in E\) such that
\begin{equation}
\label{eq:metric-one-step-contraction}
        \E d(F_U(z),F_U(z'))\le r d(z,z'),
        \qquad z,z'\in E,
\end{equation}
and
\begin{equation}
\label{eq:metric-reference-integrability}
        \E d(F_U(z_\star),z_\star)<\infty.
\end{equation}
Then the chain admits a unique invariant law
\(\pi\in\mathcal P_1(E,d)\).
Moreover, for all probability laws
\(\mu,\nu\) for which the right-hand side is finite,
\begin{equation}
\label{eq:metric-W1-contraction}
        W_1^d(\mu P^n,\nu P^n)
        \le r^n W_1^d(\mu,\nu),
        \qquad n\ge0.
\end{equation}
For every bounded function \(\varphi:E\to\R\) that is Lipschitz with respect
to the metric \(d\), and every fixed \(z\in E\), one has
\(P^n\varphi(z)\to\pi(\varphi)\).  
\end{proposition}

The proof is a Banach fixed-point argument for the Markov operator in the
metric \(W_1^d\); the details are given in
Appendix~\ref{subsec:metric-fixed-point-appendix}.



Before applying Proposition~\ref{prop:CD-to-ergodic} to the augmented Markovian map $F_U$, 
we first record a lemma which converts Assumption~\ref{ass:N}\assitemref{itm:N7} into
samplewise nonexpansiveness and averaged contraction for the stochastic-gradient Jacobian.
For \((x,\xi)\in\R^d\times\Xi\), write
\[
        A(x,\xi):=\nabla_x\mathsf G(x,\xi)
        =\nabla^2H(x)+\nabla_xg(x,\xi).
\]

\begin{lemma}
\label{lem:N7-consequences}
Under Assumptions~\ref{ass:H}\assitemref{itm:H1}
and~\ref{ass:N}\assitemref{itm:N7},
\begin{equation}
\label{eq:global-jacobian-bounds}
        \opnorm{A(x,\xi)}
        \le L_{\mathsf G},\qquad
        \opnorm{\nabla_xg(x,\xi)}
        \le\frac{(2-\theta)L_{\mathsf G}}{1-\theta},
        \qquad (x,\xi)\in\R^d\times\Xi.
\end{equation}
Consequently, \(h\) is globally \(\frac{L_{\mathsf G}}{1-\theta}\)-Lipschitz,
\(g(\cdot,\xi)\) is globally
\(\frac{2-\theta}{1-\theta}L_{\mathsf G}\)-Lipschitz uniformly in \(\xi\), and
\begin{equation}
\label{eq:population-cocoercivity}
        \norm{h(x)}^2
        \le\frac{L_{\mathsf G}}{1-\theta}\ip{x}{h(x)},
        \qquad x\in\R^d.
\end{equation}
For every \(0<\alpha\le L_{\mathsf G}^{-1}\),
\begin{equation}
\label{eq:samplewise-nonexpansive}
        \norm{(I_d-\alpha A(x,\xi))v}\le \norm v,
        \qquad x\in\R^d,\ \xi\in\Xi,\ v\in\R^d.
\end{equation}
Furthermore, for every \(x\in\R^d\), \(\xi\in\Xi\), and \(v\ne0\),
\begin{equation}
\label{eq:averaged-directional-contraction}
        \E\!\left[
        \norm{(I_d-\alpha A(x,\Phi(\xi,U_1)))v}
        \right]
        \le
        \left(
        1-\frac{\alpha(1-\theta)}{2}
        \frac{\ip{v}{\nabla^2H(x)v}}{\norm v^2}
        \right)\norm v.
\end{equation}
\end{lemma}


To apply Proposition~\ref{prop:CD-to-ergodic}, the key is to establish \eqref{eq:main-contraction-aug}. Towards this end, we need to construct a suitable weight function $V_\alpha$. 
Let \(R_H\) be as in Assumption~\ref{ass:H}\assitemref{itm:H2}.
For the tail region, define
\[
        \mathcal T_\beta(x):=
        \begin{cases}
        1, & \beta=2,\\[0.3em]
        \exp\!\left\{\kappa\bigl(\norm{x}^{2-\beta}-(\tfrac{3R_H}{4})^{2-\beta}\bigr)_+\right\},
        & 1\le\beta<2,
        \end{cases}
\]
where \(\kappa>0\) will be chosen sufficiently small in the subquadratic case.
Near the minimizer, define
\begin{equation*}
\begin{aligned}
        \omega(x)
        &:=
        \begin{cases}
        0, & m=2,\\
        \left(1-\left(\dfrac{2\norm{x}}{R_H}\right)^{m-2}\right)_+,
        & 2<m\le3,\\[0.4em]
        \left(1-\dfrac{2\norm{x}}{R_H}\right)_+,
        & m>3,
        \end{cases}
        &\qquad
        \delta_\alpha
        &:=
        \begin{cases}
        0, & m=2,\\
        \kappa_0\alpha, & 2<m\le3,\\
        \kappa_0\alpha^{m-2}, & m>3,
        \end{cases}
\end{aligned}
\end{equation*}
where the constant \(\kappa_0>0\) will be chosen in the technical estimates below.
Finally set
\begin{equation*}
        V_\alpha(x,\xi):=V_\alpha(x):=\mathcal T_\beta(x)+\delta_\alpha\omega(x).
\end{equation*}
Define \(d_{V_\alpha}^{X}\) to be the metric from
Definition~\ref{def:induced-wasserstein} with the Euclidean base norm and
weight \(V_\alpha\).

Regarding the two summands in $V_\alpha$, the term \(\mathcal T_\beta\) controls
the tail region: it is constant in the quadratic-tail case \(\beta=2\), and is
exponential in \(\norm{x}^{2-\beta}\) when \(1\le\beta<2\).  The
compactly supported function \(\omega\) is used when \(m>2\): it provides
contraction in a neighborhood of the minimizer. 
In the flat cases \(m>2\), the coefficient \(\delta_\alpha\) is the amplitude
of this near-minimizer correction and is chosen so that its
decrease has order \(\alpha^{m-1}\).

For fixed \(\xi\in\Xi\), write
\[
        \xi^+:=\Phi(\xi,U_1),
        \qquad
        f_{\xi^+}(x):=x-\alpha\{h(x)+g(x,\xi^+)\}.
\]
For \(e\in\Sph^{d-1}\), define the directional contraction factor
\[
        \mathcal D_{\xi,\alpha}(x,e)
        :=
        \norm{(I_d-\alpha A(x,\Phi(\xi,U_1)))e}.
\]
For a nonnegative function \(\phi(x)\), define the directional kernel
\[
        (\mathcal K_{\xi,\alpha}\phi)(x;e)
        :=
        \E\!\left[
        \mathcal D_{\xi,\alpha}(x,e)\,
        \phi\bigl(f_{\Phi(\xi,U_1)}(x)\bigr)
        \right],
        \qquad e\in\Sph^{d-1}.
\]

The next lemma is the main estimate that shows the contraction of the weight $V_\alpha$ under the kernel $\mathcal K_{\xi,\alpha}$. 

\begin{lemma}
\label{lem:fixed-conditional-drift}
Assume Assumptions~\ref{ass:H} and~\ref{ass:N}.  There exist constants
\[
        c_0\in(0,1),\qquad C_V > 0,\qquad \alpha_0\in(0,1],
\]
such that, for every \(0<\alpha\le\alpha_0\), \(\xi\in\Xi\), and
\(x\in\R^d\),
\begin{equation}
\label{eq:conditional-directional-drift-main}
        \sup_{e\in\Sph^{d-1}}
        (\mathcal K_{\xi,\alpha}V_\alpha)(x;e)
        \le
        (1-c_0\alpha^{m-1})V_\alpha(x),
\end{equation}
and
\begin{equation}
\label{eq:conditional-growth-main}
        \E\!\left[
        V_\alpha\bigl(f_{\Phi(\xi,U_1)}(x)\bigr)
        \right]
        \le
        (1+C_V\alpha)V_\alpha(x).
\end{equation}
\end{lemma}


The next lemma lifts the directional contraction to the induced metric on the augmented space, proving \eqref{eq:main-contraction-aug}.  

\begin{lemma}
\label{lem:aug-lipschitz}
Assume Assumptions~\ref{ass:H} and~\ref{ass:N}, and let
\(c_0\) be as in Lemma~\ref{lem:fixed-conditional-drift}.
There exists \(\alpha_0\in(0,1]\) such that
\eqref{eq:main-contraction-aug} holds for every
\(0<\alpha\le\alpha_0\).
\end{lemma}

The final estimate verifies the reference-point integrability required by Proposition~\ref{prop:CD-to-ergodic}. 

\begin{lemma}
\label{lem:fixed-reference-estimates}
There exists \(\alpha_0\in(0,1]\) such that, for every
\(0<\alpha\le\alpha_0\), with \(z_\star=(0,\xi_\star)\) and \(\xi_\star\) as
in Assumption~\ref{ass:N}\assitemref{itm:N3},
\begin{equation}
\label{eq:fixed-reference-integrability}
        \E\,d_{V_\alpha,\alpha}\bigl(F_{U_1}(z_\star),z_\star\bigr)<\infty.
\end{equation}
\end{lemma}

The proofs of Lemmas~\ref{lem:N7-consequences},
\ref{lem:fixed-conditional-drift}, \ref{lem:aug-lipschitz}, and
\ref{lem:fixed-reference-estimates} are given in
Appendices~\ref{subsec:update-condition-appendix},
\ref{subsec:fixed-regional-drift},
\ref{subsec:augmented-induced-metric-proof}, and
\ref{subsec:fixed-reference-estimates-proof}, respectively.

\begin{proof}[Proof of Theorem~\ref{thm:main-markov}]
We first check that \((\mathsf Z,d_{V_\alpha,\alpha})\) is Polish. 
Since \(\mathsf Z\) is closed in a finite-dimensional Euclidean space,
\(V_\alpha\) is continuous, and \(V_\alpha\ge1\), the metric
\(d_{V_\alpha,\alpha}\) dominates the augmented base norm \eqref{eq:aug-base-norm}.  Conversely, on every set
$K_R:=\{z\in\mathsf Z:\vert z\vert_\alpha\le R\},$ 
continuity of \(V_\alpha\) gives \(\sup_{K_R}V_\alpha<\infty\), so
\(d_{V_\alpha,\alpha}\) and the augmented base metric are locally comparable on
\(K_R\). 
Hence \((\mathsf Z,d_{V_\alpha,\alpha})\) is Polish by elementary general topology.

Let \(c_0\) be as in Lemma~\ref{lem:fixed-conditional-drift}, and choose
\(\alpha_0 \in (0, L_{\mathsf G}^{-1}]\) small enough that the conclusions of
Lemmas~\ref{lem:fixed-conditional-drift},~\ref{lem:aug-lipschitz},
and~\ref{lem:fixed-reference-estimates} hold.
Fix \(0<\alpha\le\alpha_0\).  
Equations~\eqref{eq:main-contraction-aug}
and~\eqref{eq:fixed-reference-integrability} verify respectively the two
conditions~\eqref{eq:metric-one-step-contraction}
and~\eqref{eq:metric-reference-integrability} of
Proposition~\ref{prop:CD-to-ergodic}, with
\((E,d)=(\mathsf Z,d_{V_\alpha,\alpha})\), \(r=1-c_0\alpha^{m-1}\), and
\(z_\star\) as in Lemma~\ref{lem:fixed-reference-estimates}. 
Proposition~\ref{prop:CD-to-ergodic}
therefore yields a unique invariant law $\pi_\alpha\in\mathcal P_1(\mathsf Z,d_{V_\alpha,\alpha})$
and gives \eqref{eq:markov-contraction-laws} with $c = c_0$.
%
%
%
%
\end{proof}

\subsection{Analysis of scaling limit}
\label{subsec:proof-roadmap-scaling}
\label{sec:proof-clt-markov}

Next, we turn to the scaling limit of the invariant law $\pi_\alpha$ given by Theorem~\ref{thm:main-markov} as $\alpha \downarrow 0$.
For every sufficiently small \(\alpha>0\), let
\[
        (X_\alpha,\xi_\alpha)\sim\pi_\alpha,
        \qquad
        \xi_\alpha^+:=\Phi(\xi_\alpha,U_1),
        \qquad
        X_\alpha^+:=X_\alpha-\alpha\{h(X_\alpha)+g(X_\alpha,\xi_\alpha^+)\},
\]
where \(U_1\) is independent of \((X_\alpha,\xi_\alpha)\).  Set
\[
        Y_\alpha:=\alpha^{-1/m}X_\alpha,
        \qquad
        Y_\alpha^+:=\alpha^{-1/m}X_\alpha^+,
        \qquad
        h_\alpha(y):=\alpha^{-(1-1/m)}h(\alpha^{1/m}y).
\]
By stationarity of \(\pi_\alpha\),
\[
        (Y_\alpha^+,\xi_\alpha^+)\stackrel d=(Y_\alpha,\xi_\alpha),
\]
and
\[
        Y_\alpha^+-Y_\alpha
        =
        -\alpha^{2-2/m}h_\alpha(Y_\alpha)
        -\alpha^{1-1/m}g(\alpha^{1/m}Y_\alpha,\xi_\alpha^+).
\]
Thus the drift and noise have orders \(\alpha^{2-2/m}\) and
\(\alpha^{1-1/m}\), respectively (recovering the familiar $\alpha$ and $\sqrt{\alpha}$ scalings for $m=2$ in particular).  We prove
Theorem~\ref{thm:clt-iso-markov} by showing tightness of \(Y_\alpha\),
identifying every weak subsequential limit through the stationary generator
equation, and using uniqueness of the invariant law of the limiting diffusion.

The following key lemma gives the Poisson decomposition of the Markovian noise and
identifies the covariance matrix \(\Sigma\).
Let \(Q\) be the transition kernel of the driving chain $(\xi_n)_{n \ge 0}$,
\[
        Q\varphi(\xi):=\E[\varphi(\Phi(\xi,U_1))].
\]

\begin{lemma}
\label{lem:poisson-gk-construction}
Under Assumptions~\ref{ass:H} and~\ref{ass:Nscale}, for \(x\in\R^d\), write \(g_x(\xi):=g(x,\xi)\).  Then the series
\begin{equation*}
        \chi_x(\xi):=\sum_{k=1}^{\infty}Q^k g_x(\xi),
        \qquad x\in\R^d,
        \ \xi\in\Xi,
\end{equation*}
converges absolutely and defines a centered solution of the Poisson equation
\begin{equation}
\label{eq:chi-poisson}
        \chi_x-Q\chi_x=Qg_x.
\end{equation}
Define
\begin{equation*}
        D_x(\xi,u):=
        g_x(\Phi(\xi,u))+\chi_x(\Phi(\xi,u))-\chi_x(\xi).
\end{equation*}
Then
\begin{equation}
\label{eq:Dx-martingale-difference}
        \E[D_x(\xi,U_1)\mid \xi]=0,
\end{equation}
and, with \(\xi^+=\Phi(\xi,u)\),
\begin{equation}
\label{eq:martingale-poisson-decomposition}
        g_x(\xi^+)=D_x(\xi,u)+\chi_x(\xi)-\chi_x(\xi^+).
\end{equation}
The series in \eqref{eq:GK-lag-series} converges absolutely, and the resulting matrix satisfies
\begin{equation*}
        \Sigma
        =
        \E_{\xi\sim\pi_\Xi}\E\!\bigl[D_0(\xi,U_1)D_0(\xi,U_1)^\top\bigr]
        \succeq0,
        \qquad D_0(\xi,u):=D_x(\xi,u)\big|_{x=0}.
\end{equation*}
\end{lemma}

The next lemma gives moment estimates for $X_\alpha$ which will imply tightness of \(Y_\alpha=\alpha^{-1/m}X_\alpha\).

\begin{lemma}
\label{lem:scaling-tightness-package}
Assume Assumptions~\ref{ass:Hscale} and~\ref{ass:Nscale}.  There exist constants \(C,\alpha_0>0\) such that, for every \(0<\alpha\le\alpha_0\),
\begin{align}
        \E\ip{X_\alpha}{h(X_\alpha)}
        &\le C\alpha, \label{eq:scaling-xh-main}\\
        \E\bigl[\norm{X_\alpha}^m\1_{\{\norm{X_\alpha}\le1\}}\bigr]
        &\le C\alpha, \label{eq:scaling-local-moment-main}\\
        \PP(\norm{X_\alpha}\ge1)
        &\le C\alpha, \label{eq:scaling-unit-tail-main}\\
        \E\norm{h(X_\alpha)+g(X_\alpha,\xi_\alpha^+)}^2
        &\le C. \label{eq:scaling-gradient-second-main}
\end{align}
\end{lemma}

The following lemma controls the dependence between $Y_\alpha$ and $\xi_\alpha$.

\begin{lemma}
\label{lem:noisestate-decorrelation}
Assume Assumptions~\ref{ass:Hscale} and~\ref{ass:Nscale}.  Let \(\psi:\R^d\to\R\) be bounded and globally Lipschitz, and let \(f:\Xi\to\R\) be bounded, Lipschitz, and centered under \(\pi_\Xi\).  Then
\[
        \abs{\E[\psi(Y_\alpha)f(\xi_\alpha)]}
        \le C_{\psi,f}\alpha^{1-1/m},
\]
where \(C_{\psi,f}>0\) is independent of \(\alpha\).  In particular,
\[
        \E[\psi(Y_\alpha)f(\xi_\alpha)]\to0.
\]
\end{lemma}

The next lemma identifies every weak subsequential limit of \(Y_\alpha\).
\begin{lemma}
\label{lem:markov-generator-identification}
Assume Assumptions~\ref{ass:Hscale} and~\ref{ass:Nscale}.  Let
\(\alpha_k\downarrow0\) and suppose
\[
        Y_{\alpha_k}\Rightarrow \nu.
\]
Then, for every \(\varphi\in C_c^3(\R^d)\),
\begin{equation}
\label{eq:stationary-generator-limit-main}
        \int_{\R^d}
        \left[
        -\ip{h_0(y)}{\nabla\varphi(y)}
        +
        \frac12\tr\bigl(\Sigma\nabla^2\varphi(y)\bigr)
        \right]\,\nu(dy)
        =
        0.
\end{equation}
\end{lemma}

Finally, Equation~\eqref{eq:stationary-generator-limit-main} identifies a unique limiting law.

\begin{lemma}
\label{lem:limit-diffusion-identification}
Assume Assumption~\ref{ass:Hscale}, and let \(\Sigma\succeq0\).  The (possibly degenerate) diffusion
\[
        dY_t=-h_0(Y_t)\,dt+\Sigma^{1/2}\,dB_t
\]
has a unique invariant law, denoted by \(\nu_\infty\).  Moreover, any probability law \(\nu\) satisfying \eqref{eq:stationary-generator-limit-main} for every \(\varphi\in C_c^3(\R^d)\) equals \(\nu_\infty\).
\end{lemma}

The proofs of Lemmas~\ref{lem:poisson-gk-construction},
\ref{lem:scaling-tightness-package}, \ref{lem:noisestate-decorrelation},
\ref{lem:markov-generator-identification}, and
\ref{lem:limit-diffusion-identification} are given in
Appendices~\ref{sec:proof-poisson-gk-construction},
\ref{sec:proof-scaling-tightness-package},
\ref{sec:proof-noisestate-decorrelation},
\ref{sec:proof-markov-generator-identification}, and
\ref{sec:proof-limit-diffusion-identification}, respectively.

\begin{proof}[Proof of Theorem~\ref{thm:clt-iso-markov}]
Fix \(R\ge1\).  For all sufficiently small \(\alpha\), we have
\(R\alpha^{1/m}\le1\), and Lemma~\ref{lem:scaling-tightness-package}
gives
\begin{align*}
        \PP(\norm{Y_\alpha}>R)
        &=
        \PP(\norm{X_\alpha}>R\alpha^{1/m})\\
        &\le
        \PP(R\alpha^{1/m}<\norm{X_\alpha}\le1)
        +\PP(\norm{X_\alpha}>1)\\
        &\le
        \frac{
        \E[\norm{X_\alpha}^m
        \1_{\{\norm{X_\alpha}\le1\}}]
        }{R^m\alpha}
        +\PP(\norm{X_\alpha}>1)\\
        &\le
        CR^{-m}+C\alpha,
\end{align*}
where the last line uses
\eqref{eq:scaling-local-moment-main} and
\eqref{eq:scaling-unit-tail-main}.  Hence
\[
        \limsup_{\alpha\downarrow0}
        \PP(\norm{Y_\alpha}>R)
        \le CR^{-m}.
\]
Letting \(R\to\infty\) shows that the family
\(\{Y_\alpha\}_{\alpha>0}\) is tight as \(\alpha\downarrow0\).

We next identify all possible subsequential limits.  Let
\((\alpha_k)_{k\ge1}\) be an arbitrary sequence satisfying
\(\alpha_k\downarrow0\).  By tightness, there exist integers
\(
        1\le k_1<k_2<\cdots
\)
and a probability law \(\nu\) on \(\R^d\) such that
\(
        Y_{\alpha_{k_j}}\Rightarrow\nu
\)
as $j\to\infty$.
Since \(\alpha_{k_j}\downarrow0\), we may apply
Lemma~\ref{lem:markov-generator-identification} to the sequence
\((\alpha_{k_j})_{j\ge1}\).  It follows that \(\nu\) satisfies
the stationary generator equation
\eqref{eq:stationary-generator-limit-main}.  Therefore,
Lemma~\ref{lem:limit-diffusion-identification} implies that
\(
        \nu=\nu_\infty,
\)
where \(\nu_\infty\) is the unique invariant law of the limiting
diffusion \eqref{eq:limitSDE-markov}.  Thus,
\(
        Y_{\alpha_{k_j}}\Rightarrow\nu_\infty.
\)

The same argument applies to every subsequence of the original
sequence \((\alpha_k)_{k\ge1}\): every such subsequence has a
further subsequence along which the corresponding random variables
converge weakly to \(\nu_\infty\). As a result, the entire sequence satisfies
\(
        Y_{\alpha_k}\Rightarrow\nu_\infty
\)
as $k\to\infty$.
Because the sequence \((\alpha_k)_{k\ge1}\) was arbitrary, we
conclude that
\(
        Y_\alpha=\alpha^{-1/m}X_\alpha
        \Rightarrow Y_\infty
\)
as $\alpha\downarrow0$ and $Y_\infty\sim\nu_\infty$. 
\end{proof}

\begin{remark}
The generator argument above identifies weak subsequential limits rather than
solving a Stein equation for the limiting invariant law.  This is particularly
useful when \(m>2\), since the limiting generator has nonlinear \(m\)-flat
drift, a generally non-Gaussian invariant law, and possibly degenerate
covariance.  
Obtaining uniform estimates for solutions of the associated
Stein equations that are needed for quantitative convergence rates would
require additional analysis of Poisson equations for diffusions with
nonlinear drift.
\end{remark}



\section{Conclusion}
\label{sec:conclusion}

This paper develops an invariant-law theory for constant-stepsize SGD beyond
the strongly convex setting.  For each sufficiently small constant stepsize, we
prove geometric convergence to a unique invariant law in a Wasserstein
distance induced by a metric adapted to two geometric features of the objective: an
\(m\)-flat near-minimizer region and a \(\beta\)-subquadratic tail.  We then identify the
small-stepsize scaling limit of the invariant law as the solution of \eqref{eq:limitSDE-markov}.
In particular, the classical Gaussian limit is recovered when
\(m=2\), while flatter objectives can produce larger and generally non-Gaussian
stationary fluctuations. Moreover, Markovian dependence in gradient noise enters through the asymptotic
covariance \(\Sigma\).

The main conceptual point is that local and global features play different
roles.  The exponent \(m\) determines the scale of the invariant law and the scaling limit,
whereas \(\beta\) determines the weight function needed to control
excursions.  This separation is reflected in the examples: quantile and
tail-risk estimation illustrate nonquadratic local flatness, while smooth
robust and logistic losses illustrate subquadratic tail growth.

Several directions remain open.  First, the separable extension developed here
allows different coordinate flatness exponents, but the genuinely nonseparable
anisotropic case remains to be understood.  If different coordinate groups have
natural scales \(s_i(\alpha)=\alpha^{1/m_i}\), the anisotropically rescaled
generator contains several effective time scales; a natural approach is
multi-time-scale stochastic averaging
\cite{Khasminskii1968,Borkar1997,KushnerYin2003,KangKurtz2013}.  Second, quantitative convergence rates for the scaling limit of invariant laws may be
accessible through Stein's method for the limiting diffusion, 
based on the solution of the Poisson equation
\cite{Barbour1990,BravermanDaiFeng2016,PardouxVeretennikov2001,
PardouxVeretennikov2003,PardouxVeretennikov2005}.  Other important extensions include replacing the one-step contraction in geometric ergodicity
by a multi-step contraction, 
which would cover Markovian data streams for which contraction only emerges over several transitions, and relaxing the globally contractive driving-chain dynamics
toward Harris-type or mixing-based Markovian noise
\cite{MeynTweedie2009,EberleGuillinZimmer2019}.

\section*{Acknowledgments}

Cheng Mao was supported in part by NSF CAREER Award 2338062. 
Debankur Mukherjee was partially supported by the NSF grant CPS-2240982.

\appendix

\section{Constant-stepsize contraction}
\label{sec:constant-step-technical-estimates}

This appendix proves the metric fixed-point principle and the estimates used in
Theorem~\ref{thm:main-markov}, followed by the projected Wasserstein--1
corollary.  Constants denoted by \(C,c\) may change from line to line but are
independent of \(\alpha\), \(x\), and the current value \(\xi\) of the driving
chain.  We decrease the small-stepsize threshold when needed.

\subsection{Fixed-point principle}
\label{subsec:metric-fixed-point-appendix}

\begin{proof}[Proof of Proposition~\ref{prop:CD-to-ergodic}]
The assumptions imply that \(P\) maps \(\mathcal P_1(E,d)\) into itself,
since
\(\E d(F_U(z),z_\star)\le r d(z,z_\star)+\E d(F_U(z_\star),z_\star)\).
The one-step estimate \eqref{eq:metric-one-step-contraction} iterates under
the synchronous coupling and gives \eqref{eq:metric-W1-contraction}.  Let
\(\mu_0=\delta_{z_\star}\) and \(\mu_n=\mu_0P^n\).  By
\eqref{eq:metric-reference-integrability}, \(W_1^d(\mu_1,\mu_0)<\infty\),
and hence
\(W_1^d(\mu_{k+1},\mu_k)\le r^kW_1^d(\mu_1,\mu_0)\).  Thus \((\mu_n)\)
is Cauchy in \((\mathcal P_1(E,d),W_1^d)\).  Let its limit be \(\pi\).
Applying the contraction to \(\mu_n\) and \(\pi\) shows
\(\mu_{n+1}\to\pi P\), while \(\mu_{n+1}\to\pi\), so \(\pi P=\pi\).

For fixed \(z\in E\), taking \(\mu=\delta_z\) and \(\nu=\pi\) in
\eqref{eq:metric-W1-contraction} gives
\(W_1^d(\delta_zP^n,\pi)\le r^nW_1^d(\delta_z,\pi)\to0\); the right-hand
side is finite because \(\pi\in\mathcal P_1(E,d)\).  Consequently, every
bounded \(d\)-Lipschitz \(\varphi\) satisfies
\(|P^n\varphi(z)-\pi(\varphi)|\le
\Lip_d(\varphi)W_1^d(\delta_zP^n,\pi)\to0\).

Let \(\widetilde\pi\) be any Borel invariant probability law.  By the
pointwise convergence, invariance, and dominated convergence,
\(\widetilde\pi(\varphi)=\lim_n\widetilde\pi(P^n\varphi)=\pi(\varphi)\).
Bounded Lipschitz functions determine Borel probability measures on the
Polish space \((E,d)\), so \(\widetilde\pi=\pi\).
\end{proof}

\subsection{Proof of Lemma~\ref{lem:N7-consequences}}
\label{subsec:update-condition-appendix}

\begin{proof}[Proof of Lemma~\ref{lem:N7-consequences}]
Put \(A_{\rm s}(x,\xi):=(A(x,\xi)+A(x,\xi)^\top)/2\).  Fix
\(\xi\in\Xi\).  Applying Cauchy--Schwarz to
\eqref{eq:sample-cocoercivity} gives
\(\norm{\mathsf G(x,\xi)-\mathsf G(y,\xi)}
\le L_{\mathsf G}\norm{x-y}\) for \(x,y\in\R^d\); the conclusion is
immediate when the left-hand difference vanishes and otherwise follows by
division by its norm.  Since \(A(x,\xi)=\nabla_x\mathsf G(x,\xi)\), this
proves the first bound in \eqref{eq:global-jacobian-bounds}.

Next fix \(x,v\in\R^d\).  Apply \eqref{eq:sample-cocoercivity} to
\(x+tv\) and \(x\), divide by \(t^2\), and let \(t\to0\).  Then
\[
\norm{A(x,\xi)v}^2
\le L_{\mathsf G}\ip{v}{A(x,\xi)v}
=L_{\mathsf G}\ip{v}{A_{\rm s}(x,\xi)v},
\qquad
0\preceq A_{\rm s}(x,\xi)\preceq L_{\mathsf G}I_d.
\]
Here the first inequality implies \(A_{\rm s}(x,\xi)\succeq0\), and the
second also uses \(\opnorm{A_{\rm s}}\le\opnorm A\).

Since \(\mathsf G=h+g\), \eqref{eq:mean-perturbation-control} implies
\[
\ip{x-y}{\E[\mathsf G(x,\Phi(\xi,U_1))-
\mathsf G(y,\Phi(\xi,U_1))]}
\ge(1-\theta)\ip{x-y}{h(x)-h(y)}.
\]
Apply this inequality to \(x+tv\) and \(x\), divide by \(t^2\), and let
\(t\to0\).  For \(t\ne0\), the preceding Lipschitz bound gives
\(\norm{[\mathsf G(x+tv,\Phi(\xi,U_1))-
\mathsf G(x,\Phi(\xi,U_1))]/t}\le L_{\mathsf G}\norm v\), so dominated
convergence yields
\[
\E[A_{\rm s}(x,\Phi(\xi,U_1))]\succeq(1-\theta)\nabla^2H(x),
\qquad
0\preceq\nabla^2H(x)\preceq\frac{L_{\mathsf G}}{1-\theta}I_d.
\]
Since \(\nabla_xg(x,\xi)=A(x,\xi)-\nabla^2H(x)\), the triangle
inequality proves the second bound in \eqref{eq:global-jacobian-bounds}.
Integrating the two Jacobian bounds along line segments yields the stated
global Lipschitz estimates.  Finally, \eqref{eq:population-cocoercivity}
is the Baillon--Haddad inequality for the convex
\(L_{\mathsf G}/(1-\theta)\)-smooth function \(H\), applied with
\(h(0)=0\).

It remains to prove the two contraction bounds.  Write \(A=A(x,\xi)\)
and \(A_{\rm s}=A_{\rm s}(x,\xi)\).  Since
\(\alpha L_{\mathsf G}\le1\),
\[
\begin{aligned}
\norm{(I_d-\alpha A)v}^2
&=\norm v^2-2\alpha\ip{v}{A_{\rm s}v}+\alpha^2\norm{Av}^2\\
&\le\norm v^2-(2\alpha-\alpha^2L_{\mathsf G})\ip{v}{A_{\rm s}v}
\le\norm v^2-\alpha\ip{v}{A_{\rm s}v}\le\norm v^2.
\end{aligned}
\]
This proves \eqref{eq:samplewise-nonexpansive}.  If \(v\ne0\), then
\[
\begin{aligned}
0&\le\alpha\frac{\ip{v}{A_{\rm s}v}}{\norm v^2}
\le\alpha\opnorm{A_{\rm s}}\le\alpha\opnorm A\le1,\\
\norm{(I_d-\alpha A)v}
&\le\left(1-\frac{\alpha}{2}
\frac{\ip{v}{A_{\rm s}v}}{\norm v^2}\right)\norm v,
\end{aligned}
\]
where the second inequality uses \(\sqrt{1-t}\le1-t/2\).  Replacing
\(\xi\) by \(\Phi(\xi,U_1)\), taking expectation with \(v\) fixed, and
using the averaged matrix inequality above proves
\eqref{eq:averaged-directional-contraction}.
\end{proof}

\subsection{Preparatory estimates}

\begin{lemma}
\label{lem:fixed-deterministic-bounds}
Under Assumption~\ref{ass:H}, there exists \(C_h>0\) such that
\(\norm{h(x)}\le C_h\norm{x}^{m-1}\) whenever
\(\norm{x}\le R_H/2\).
\end{lemma}
\begin{proof}
This is the second inequality in
Assumption~\ref{ass:H}\assitemref{itm:H2}. 
\end{proof}

\begin{lemma}
\label{lem:fixed-truncated-bounds}
Assume \(m>2\), and define
\[
\begin{aligned}
G_\xi&:=\norm{g(0,\Phi(\xi,U_1))},&
\bar a_\alpha&:=\inf_{\xi\in\Xi}\E\min\{\alpha G_\xi,1\},\\
\bar a_{\alpha,p}
&:=\inf_{\xi\in\Xi}\E\bigl[\min\{(\alpha G_\xi)^p,1\}\bigr],&
&\hspace{-2em}p\in(0,1].
\end{aligned}
\]
Under Assumption~\ref{ass:N}\assitemref{itm:N6}, there exist constants
\(c_1,c_2>0\) and \(\alpha_g\in(0,1]\) such that, for
\(0<\alpha\le\alpha_g\),
\begin{align}
\bar a_\alpha&\ge c_1\alpha,
\label{eq:fixed-bara-linear}\\
\bar a_{\alpha,p}&\ge c_2\bar a_\alpha^p.
\label{eq:fixed-bara-p}
\end{align}
Moreover, under Assumption~\ref{ass:N}\assitemref{itm:N3} if \(\beta=2\),
and under \assitemref{itm:N5} if \(\beta<2\), there exists \(C_g>0\)
such that
\begin{equation}
\label{eq:fixed-bara-upper}
\bar a_\alpha\le C_g\alpha.
\end{equation}
Consequently \(\bar a_\alpha=\Theta(\alpha)\).
\end{lemma}

\begin{proof}
Assumption~\ref{ass:N}\assitemref{itm:N6} gives
\(\inf_{\xi\in\Xi}\PP(G_\xi\ge\varepsilon_g)\ge p_g\).  For
\(0<\alpha\le\varepsilon_g^{-1}\),
\(\min\{\alpha G_\xi,1\}\ge
\alpha\varepsilon_g\1_{\{G_\xi\ge\varepsilon_g\}}\), and therefore
\(\bar a_\alpha\ge\alpha\varepsilon_gp_g\), proving
\eqref{eq:fixed-bara-linear}.  Similarly,
\(\min\{(\alpha G_\xi)^p,1\}\ge
(\alpha\varepsilon_g)^p\1_{\{G_\xi\ge\varepsilon_g\}}\), so
\(\bar a_{\alpha,p}\ge\alpha^p\varepsilon_g^pp_g\).

It remains to compare \(\alpha^p\) with \(\bar a_\alpha^p\).  Let
\(\xi_\star\) be the reference point in
Assumption~\ref{ass:N}\assitemref{itm:N3}.  If \(\beta=2\), then
\(\E G_{\xi_\star}<\infty\) by \assitemref{itm:N3}; if \(\beta<2\),
Assumption~\ref{ass:N}\assitemref{itm:N5} at \(x=0\) and
\(\xi=\xi_\star\) gives the same conclusion.  Thus
\(\bar a_\alpha\le\E\min(\alpha G_{\xi_\star},1)
\le\alpha\E G_{\xi_\star}\).  Combining this with the previous lower
bound on \(\bar a_{\alpha,p}\) proves \eqref{eq:fixed-bara-p}; the same
inequality gives \eqref{eq:fixed-bara-upper}.  Together with
\eqref{eq:fixed-bara-linear}, this proves
\(\bar a_\alpha=\Theta(\alpha)\). 
\end{proof}

\begin{lemma}
\label{lem:fixed-elementary-estimates}
The following estimates hold.
\begin{enumerate}[label=\textup{(\roman*)}, leftmargin=2.5em]
\item Let \(Z\) be an \(\R^d\)-valued random vector with \(\E Z=0\) and
\(\E e^{\lambda\norm Z}\le M\) for some \(\lambda>0\).  Then there exist
\(t_0\in(0,\lambda/2]\) and \(C>0\) such that
\(\E e^{t\langle u,Z\rangle}\le e^{Ct^2}\) for every unit vector \(u\)
and every \(|t|\le t_0\).
\item If \(Y\ge0\) and \(\E e^{\lambda Y}\le M\), then
\(\E e^{tY}\le1+(2M/\lambda)t\) for \(0\le t\le\lambda/2\).
\item Suppose \(1\le\beta<2\).  There exists \(C_\beta>0\) such that,
whenever \(y\ne0\) and \(\norm z\le\norm y/2\),
\[
\norm{y-z}^{2-\beta}
\le \norm y^{2-\beta}-(2-\beta)\norm y^{1-\beta}
\left\langle\frac{y}{\norm y},z\right\rangle
+C_\beta\norm y^{-\beta}\norm z^2.
\]
\end{enumerate}
\end{lemma}

\begin{proof}
For (i), set \(X=\langle u,Z\rangle\), so \(\E X=0\) and
\(|X|\le\norm Z\).  For \(|t|\le\lambda/2\),
\[
\begin{aligned}
 e^{tX}&\le 1+tX+\frac{t^2\norm Z^2}{2}e^{|t|\norm Z},&
r^2e^{(\lambda/2)r}&\le \frac{16}{\lambda^2}e^{\lambda r},\\
 \E e^{tX}&\le 1+\frac{8M}{\lambda^2}t^2
 \le \exp\!\left(\frac{8M}{\lambda^2}t^2\right).
\end{aligned}
\]
For (ii), \(e^{tY}-1\le tYe^{tY}\le tYe^{(\lambda/2)Y}\) and
\(re^{(\lambda/2)r}\le(2/\lambda)e^{\lambda r}\), so
\(\E e^{tY}\le1+(2M/\lambda)t\).
For (iii), apply Taylor's theorem to
\(\psi(w)=\norm{w}^{2-\beta}\) on \(\R^d\setminus\{0\}\).
Along \(y-sz\), \(s\in[0,1]\), the condition
\(\norm z\le\norm y/2\) gives \(\norm{y-sz}\ge\norm y/2\); since
\(\opnorm{\nabla^2\psi(w)}\le C_\beta\norm w^{-\beta}\), the claimed
inequality follows.
\end{proof}

\subsection{Proof of Lemma~\ref{lem:fixed-conditional-drift}}
\label{subsec:fixed-regional-drift}

For fixed \(\xi\), write
\[
        \xi^+:=\Phi(\xi,U_1),
        \qquad
        f_{\xi^+}(x)=x-\alpha\{h(x)+g(x,\xi^+)\}.
\]
For \(e\in\Sph^{d-1}\), set
\(
        \mathcal D_{\xi,\alpha}(x,e)
        :=
        \norm{(I_d-\alpha A(x,\xi^+))e}.
\)
By \eqref{eq:samplewise-nonexpansive},
\begin{equation}
\label{eq:fixed-directional-nonexpansive}
        \mathcal D_{\xi,\alpha}(x,e)\le1,
        \qquad e\in\Sph^{d-1}.
\end{equation}
Moreover, by \eqref{eq:averaged-directional-contraction},
\begin{equation}
\label{eq:fixed-directional-basic}
        \E\mathcal D_{\xi,\alpha}(x,e)
        \le
        1-\frac{\alpha(1-\theta)}{2}
        \ip{e}{\nabla^2H(x)e},
        \qquad e\in\Sph^{d-1}.
\end{equation}

We first record two weight estimates used below.

\begin{lemma}
\label{lem:compact-weight-increment}
There exist constants \(C,c>0\), independent of
\(\alpha,x,\xi,\kappa,\kappa_0\), such that, for all sufficiently small
\(\alpha\), the following hold.  If \(\norm{x}\le \frac{R_H}{2}\), then
\begin{equation}
\label{eq:core-overshoot-restored}
        \E\bigl[(\mathcal T_\beta(f_{\xi^+}(x))-1)_+\bigr]
        \le
        C e^{-c/\alpha},
\end{equation}
with the left-hand side equal to zero when \(\beta=2\).  If \(\norm{x}\le R_H\),
then, with the convention that \(\kappa=0\) when \(\beta=2\),
\begin{equation}
\label{eq:compact-positive-increment}
        \E\bigl[(V_\alpha(f_{\xi^+}(x))-V_\alpha(x))_+\bigr]
        \le
        C(\kappa+\kappa_0)\alpha V_\alpha(x),
\end{equation}
and consequently
\begin{equation}
\label{eq:compact-growth}
        \E V_\alpha(f_{\xi^+}(x))
        \le
        \bigl(1+C(\kappa+\kappa_0)\alpha\bigr)V_\alpha(x).
\end{equation}
\end{lemma}

\begin{proof}
The case \(\beta=2\) has \(\mathcal T_\beta\equiv1\), so
\eqref{eq:core-overshoot-restored} is immediate.  Suppose throughout the rest
of the proof that \(1\le\beta<2\).  Since \(h\) is
continuous, it is bounded on the compact set \(\{\norm{x}\le R_H\}\).  Also
\(1+\norm{x}^{\beta-1}\) is bounded on this compact set, so Assumption~\ref{ass:N}\assitemref{itm:N5}
implies that, for some \(\lambda_K,M_K>0\),
\begin{equation}
\label{eq:compact-g-exp-moment}
        \sup_{\norm{x}\le R_H}\sup_{\xi\in\Xi}
        \E\exp\{\lambda_K\norm{g(x,\xi^+)}\}\le M_K.
\end{equation}

First assume \(\norm{x}\le \frac{R_H}{2}\).  If \(\mathcal T_\beta(f_{\xi^+}(x))>1\), then
\(\norm{f_{\xi^+}(x)}>\frac{3R_H}{4}\).  Since \(\frac{3R_H}{4}>\frac{R_H}{2}\) and
\(\norm{h(x)}\le C\), this implies, after decreasing the stepsize threshold,
\(\alpha\norm{g(x,\xi^+)}\ge c\).  Moreover,
\[
        \bigl(\norm{f_{\xi^+}(x)}^{2-\beta}-(\tfrac{3R_H}{4})^{2-\beta}\bigr)_+
        \le C\{1+\alpha\norm{g(x,\xi^+)}\},
\]
because \(2-\beta\le1\).  Taking \(\kappa\) small enough that
\(C\kappa\alpha\le \lambda_K/2\), \eqref{eq:compact-g-exp-moment} gives
\[
\begin{aligned}
        \E\bigl[(\mathcal T_\beta(f_{\xi^+}(x))-1)_+\bigr]
        &\le
        C\E\!\bigl[
        e^{C\kappa\alpha\norm{g(x,\xi^+)}}
        \1_{\{\norm{g(x,\xi^+)}\ge c/\alpha\}}
        \bigr]  \\
        &\le C e^{-c/\alpha},
\end{aligned}
\]
which proves \eqref{eq:core-overshoot-restored}.

It remains to prove \eqref{eq:compact-positive-increment}.  Fix
\(\norm{x}\le R_H\) and set \(G=\norm{g(x,\xi^+)}\).  Split according to
\(B_\alpha:=\{G\le \alpha^{-1/2}\}\).  On \(B_\alpha\),
\(
        \norm{f_{\xi^+}(x)-x}\le C\alpha(1+G)
        \le C\alpha^{1/2},
\)
so both \(x\) and \(f_{\xi^+}(x)\) lie in a fixed compact enlargement of
\(\{\norm{y}\le R_H\}\).  On this enlargement, \(\mathcal T_\beta\) has Lipschitz
constant at most \(C\kappa\), for \(0<\kappa\le1\).  Hence
\[
\begin{aligned}
        \E\bigl[(\mathcal T_\beta(f_{\xi^+}(x))-\mathcal T_\beta(x))_+\1_{B_\alpha}\bigr]
        &\le
        C\kappa\alpha\E(1+G)
        \le C\kappa\alpha.
\end{aligned}
\]
On \(B_\alpha^c\), use
\(e^a-e^b\le (a-b)_+e^a\) for \(a\ge b\), the bound
\(\bigl(\norm{f_{\xi^+}(x)}^{2-\beta}-(\tfrac{3R_H}{4})^{2-\beta}\bigr)_+\le C\{1+(\alpha G)^{2-\beta}\}\), and
\eqref{eq:compact-g-exp-moment}.  Decreasing \(\alpha\) if necessary,
\[
\begin{aligned}
&\E\bigl[(\mathcal T_\beta(f_{\xi^+}(x))-\mathcal T_\beta(x))_+\1_{B_\alpha^c}\bigr]  \\
&\qquad\le
C\kappa\E\!\bigl[(1+(\alpha G)^{2-\beta})
        e^{C\kappa(1+(\alpha G)^{2-\beta})}
        \1_{\{G>\alpha^{-1/2}\}}\bigr]
\le C\kappa e^{-c\alpha^{-1/2}}
\le C\kappa\alpha.
\end{aligned}
\]
Therefore
\begin{equation}
\label{eq:compact-U-positive-increment-expanded}
        \E\bigl[(\mathcal T_\beta(f_{\xi^+}(x))-\mathcal T_\beta(x))_+\bigr]
        \le C\kappa\alpha.
\end{equation}
The \(\omega\)-part satisfies \(0\le\omega\le1\), and hence
\[
        \delta_\alpha(\omega(f_{\xi^+}(x))-\omega(x))_+
        \le \delta_\alpha\le C\kappa_0\alpha.
\]
Since \(V_\alpha(x)\ge1\), \eqref{eq:compact-positive-increment} follows from
\eqref{eq:compact-U-positive-increment-expanded}.  Finally,
\eqref{eq:compact-growth} follows from
\(V_\alpha(f)\le V_\alpha(x)+(V_\alpha(f)-V_\alpha(x))_+\).
\end{proof}

\begin{lemma}
\label{lem:subquadratic-tail-decrease}
Assume \(1\le\beta<2\).  There exist \(c>0\) and
\(\alpha_{\rm tail}\in(0,1]\) such that, for all
\(0<\alpha\le\alpha_{\rm tail}\), all \(\xi\in\Xi\), and all
\(\norm{x}\ge R_H\),
\begin{equation}
\label{eq:far-U-beta-decrease-new}
        \E \mathcal T_\beta(f_{\xi^+}(x))
        \le
        (1-c\alpha)\mathcal T_\beta(x).
\end{equation}
\end{lemma}

\begin{proof}
Let \(r=\norm{x}\), \(G=g(x,\xi^+)\), and
\(\zeta=G-\bar g(x,\xi)\).  Then
\(\E[\zeta\mid\xi]=0\) and
\(
        f_{\xi^+}(x)=x-\alpha\{h(x)+\bar g(x,\xi)+\zeta\}.
\)
By Assumption~\ref{ass:N}\assitemref{itm:N5},
\(\norm{\bar g(x,\xi)}\le C(1+\norm{x}^{\beta-1})\),
uniformly in \((x,\xi)\).  Fix \(\eta_0\in(0,1/2)\), put
\[
        N_x:=\frac{\norm{g(x,\xi^+)}}{1+\norm{x}^{\beta-1}},
        \qquad E_x:=\{N_x\le\alpha^{-\eta_0}\}.
\]
On \(E_x\), for sufficiently small
\(\alpha\),
\[
        \alpha\norm{h(x)+G}
        \le C\alpha^{1-\eta_0}r^{\beta-1}
        \le \varepsilon_* r,
        \qquad r\ge R_H,
\]
with \(\varepsilon_*>0\) chosen so that the segment from \(x\) to
\(f_{\xi^+}(x)\) remains in the tail region where the cutoff in \(\mathcal T_\beta\) is
inactive.  Lemma~\ref{lem:fixed-elementary-estimates}(iii), applied to
\(z=\alpha\{h(x)+G\}\), gives on \(E_x\)
\[
\begin{aligned}
        \norm{f_{\xi^+}(x)}^{2-\beta}
        &\le
        r^{2-\beta}
        -(2-\beta)\alpha r^{-\beta}\ip{x}{h(x)+\bar g(x,\xi)}
        -(2-\beta)\alpha r^{-\beta}\ip{x}{\zeta}
        +C\alpha^{2-2\eta_0}.
\end{aligned}
\]
By the tail dissipativity condition \assitemref{itm:N4},
\(r^{-\beta}\ip{x}{h(x)+\bar g(x,\xi)}\ge c_{\rm diss}\).  Therefore,
conditionally on \(\xi\),
\[
        \mathcal T_\beta(f_{\xi^+}(x))\1_{E_x}
        \le
        \mathcal T_\beta(x)
        \exp\{-\kappa(2-\beta)c_{\rm diss}\alpha
        -\kappa(2-\beta)\alpha r^{-\beta}\ip{x}{\zeta}
        +C\alpha^{2-2\eta_0}\}.
\]
Writing \(e=x/r\), the centered term is
\[
        r^{-\beta}\ip{x}{\zeta}
        =r^{1-\beta}(1+\norm{x}^{\beta-1})
        \ip{e}{
        \frac{g(x,\xi^+)}{1+\norm{x}^{\beta-1}}
        -\E\!\left[\frac{g(x,\xi^+)}{1+\norm{x}^{\beta-1}}\,\middle|\,\xi\right]}.
\]
Since \(r^{1-\beta}(1+\norm{x}^{\beta-1})\) is uniformly bounded for \(r\ge R_H\), the
coefficient of the centered normalized noise is \(O(\alpha)\).  The conditional
exponential moment in \assitemref{itm:N5} also gives a uniform exponential moment for
this centered normalized noise.  Lemma~\ref{lem:fixed-elementary-estimates}(i)
therefore implies
\[
        \E\left[
        \exp\{-\kappa(2-\beta)\alpha r^{-\beta}\ip{x}{\zeta}\}\mid \xi
        \right]
        \le e^{C\alpha^2}.
\]
Thus
\begin{equation}
\label{eq:tail-good-event-bound}
        \E[\mathcal T_\beta(f_{\xi^+}(x))\1_{E_x}]
        \le
        e^{-c\alpha+C\alpha^{2-2\eta_0}}\mathcal T_\beta(x).
\end{equation}

It remains to control \(E_x^c\).  Because \(2-\beta\in(0,1]\) and
\(s\mapsto s^{2-\beta}\) is concave,
\[
\begin{aligned}
        \norm{x-\alpha(h(x)+G)}^{2-\beta}-r^{2-\beta}
        &\le
        (r+\alpha\norm{h(x)+G})^{2-\beta}-r^{2-\beta}                         \\
        &\le
        (2-\beta)r^{1-\beta}\alpha\norm{h(x)+G}.
\end{aligned}
\]
For \(r\ge R_H\), Assumption~\ref{ass:H}\assitemref{itm:H3b} and the definition of
\(N_x\) give
\(
        \norm{h(x)+G}\le C r^{\beta-1}(1+N_x).
\)
The powers of \(r\) therefore cancel, and we obtain the deterministic bound
\(
        \norm{f_{\xi^+}(x)}^{2-\beta}-r^{2-\beta}
        \le C\alpha(1+N_x).
\)
Consequently,
\[
        \mathcal T_\beta(f_{\xi^+}(x))
        \le
        \mathcal T_\beta(x)\exp\{C\kappa\alpha(1+N_x)\}.
\]
Choosing \(\alpha\) so that \(C\kappa\alpha\le\lambda_0/2\), Assumption
\assitemref{itm:N5} yields
\begin{equation}
\label{eq:tail-bad-event-bound}
        \E[\mathcal T_\beta(f_{\xi^+}(x))\1_{E_x^c}]
        \le
        Ce^{-c\alpha^{-\eta_0}}\mathcal T_\beta(x).
\end{equation}
Combining \eqref{eq:tail-good-event-bound} and
\eqref{eq:tail-bad-event-bound}, using \(2-2\eta_0>1\) and the fact that
\(e^{-c\alpha^{-\eta_0}}=o(\alpha)\), gives
\eqref{eq:far-U-beta-decrease-new} after reducing the stepsize threshold.
\end{proof}

\begin{lemma}
\label{lem:global-weight-positive-increment}
There exist \(C>0\) and \(\alpha_0\in(0,1]\) such that, for every
\(0<\alpha\le\alpha_0\), \(x\in\R^d\), and \(\xi\in\Xi\),
\begin{equation}
\label{eq:global-weight-positive-increment}
        \E\left[
        \bigl(V_\alpha(f_{\Phi(\xi,U_1)}(x))-V_\alpha(x)\bigr)_+
        \right]
        \le C\alpha V_\alpha(x).
\end{equation}
Consequently,
\begin{equation}
\label{eq:mean-contractive-weight-gate}
        \E\left[
        L_\Phi(U_1)V_\alpha(f_{\Phi(\xi,U_1)}(x))
        \right]
        \le(\rho_\Xi+C\alpha)V_\alpha(x).
\end{equation}
\end{lemma}

\begin{proof}
For \(\norm{x}\le R_H\), \eqref{eq:global-weight-positive-increment}
is \eqref{eq:compact-positive-increment}.  Suppose \(\norm{x}\ge R_H\).
If \(\beta=2\), then \(\mathcal T_\beta\equiv1\) and \(\omega(x)=0\), so
\[
        \bigl(V_\alpha(f_{\Phi(\xi,U_1)}(x))-V_\alpha(x)\bigr)_+
        \le\delta_\alpha\le C\alpha V_\alpha(x).
\]
If \(1\le\beta<2\), use the normalized noise
\(
        N_x:=\frac{\norm{g(x,\Phi(\xi,U_1))}}
        {1+\norm{x}^{\beta-1}}.
\)
The deterministic tail comparison used in the proof of
Lemma~\ref{lem:subquadratic-tail-decrease} gives
\[
        \mathcal T_\beta(f_{\Phi(\xi,U_1)}(x))
        \le\mathcal T_\beta(x)e^{C\kappa\alpha(1+N_x)}.
\]
Since \(e^t-1\le te^t\) for \(t\ge0\), Assumption~\ref{ass:N}
\assitemref{itm:N5} yields, uniformly in \(x,\xi\),
\[
\begin{aligned}
        \E\bigl[
        (\mathcal T_\beta(f_{\Phi(\xi,U_1)}(x))
        -\mathcal T_\beta(x))_+\bigr]
        &\le
        C\alpha\mathcal T_\beta(x)
        \E[(1+N_x)e^{C\kappa\alpha(1+N_x)}]\\
        &\le C\alpha V_\alpha(x).
\end{aligned}
\]
The positive increment of \(\delta_\alpha\omega\) is at most
\(\delta_\alpha\le C\alpha V_\alpha(x)\), which proves
\eqref{eq:global-weight-positive-increment}.  Finally, \(0\le L_\Phi\le1\)
and \(\E L_\Phi(U_1)=\rho_\Xi\) give
\[
\begin{aligned}
        \E[L_\Phi(U_1)V_\alpha(f_{\Phi(\xi,U_1)}(x))]
        &\le \rho_\Xi V_\alpha(x)
        +\E[(V_\alpha(f_{\Phi(\xi,U_1)}(x))-V_\alpha(x))_+],
\end{aligned}
\]
and \eqref{eq:mean-contractive-weight-gate} follows.
\end{proof}

\begin{lemma}
\label{lem:fixed-core}
There exist \(c_{\rm core}>0\) and \(\alpha_{\rm core}\in(0,1]\) such that,
for every \(0<\alpha\le\alpha_{\rm core}\), every \(\xi\in\Xi\), and every
\(\norm{x}\le \frac{R_H}{2}\),
\begin{equation*}
        \sup_{e\in\Sph^{d-1}}
        (\mathcal K_{\xi,\alpha}V_\alpha)(x;e)
        \le
        (1-c_{\rm core}\alpha^{m-1})V_\alpha(x).
\end{equation*}
\end{lemma}

\begin{proof}
Fix \(x\) with \(r=\norm{x}\le \frac{R_H}{2}\), and fix \(e\in\Sph^{d-1}\).  Since
\(\frac{3R_H}{4}>\frac{R_H}{2}\), \(\mathcal T_\beta(x)=1\).  Put
\(\mathcal R(u)=\mathcal T_\beta(u)-1\) when \(\beta<2\), and
\(\mathcal R\equiv0\) when \(\beta=2\).  Then
\[
        V_\alpha(u)=1+\delta_\alpha\omega(u)+\mathcal R(u),
        \qquad
        V_\alpha(x)=1+\delta_\alpha\omega(x).
\]
Using \eqref{eq:fixed-directional-nonexpansive},
\begin{equation}
\label{eq:fixed-core-pre}
\begin{aligned}
        (\mathcal K_{\xi,\alpha}V_\alpha)(x;e)-V_\alpha(x)
        &\le
        \E[\mathcal D_{\xi,\alpha}(x,e)]-1
        +\delta_\alpha\{\E\omega(f_{\xi^+}(x))-\omega(x)\}  \\
        &\qquad
        +\E\mathcal R(f_{\xi^+}(x)).
\end{aligned}
\end{equation}

Put \(a_{\rm in}:=(1-\theta)c_{\rm in}/2\).
If \(m=2\), Assumption~\ref{ass:H}\assitemref{itm:H2} gives
\(\ip{e}{\nabla^2H(x)e}\ge c_{\rm in}\) for \(\norm{x}\le \frac{R_H}{2}\).  Hence
\[
        \E\mathcal D_{\xi,\alpha}(x,e)
        \le 1-a_{\rm in}\alpha.
\]
Since \(\omega\equiv0\) when \(m=2\), \eqref{eq:fixed-core-pre} and
\eqref{eq:core-overshoot-restored} give the claim after absorbing the
exponentially small term into \(\alpha\).

Assume now \(m>2\), and put \(s_m:=(m-2)\wedge1\).  Then
\[
        \omega(y)
        =
        \left(1-\left(\frac{2\norm{y}}{R_H}\right)^{s_m}\right)_+,
        \qquad y\in\R^d.
\]
By Assumption~\ref{ass:H}\assitemref{itm:H2} and
\eqref{eq:fixed-directional-basic},
\(
        \E\mathcal D_{\xi,\alpha}(x,e)
        \le 1-a_{\rm in}\alpha r^{m-2}.
\)
We claim that
\begin{equation}
\label{eq:fixed-core-omega-drop-new}
        \E\omega(f_{\xi^+}(x))-\omega(x)
        \le
        -c_\omega\E\min\{(\alpha\norm{g(0,\xi^+)})^{s_m},1\}
        +C_\omega r^{s_m}.
\end{equation}
Indeed, write
\[
        f_{\xi^+}(x)=-\alpha g(0,\xi^+)+v,
        \qquad
        v=x-\alpha h(x)-\alpha\{g(x,\xi^+)-g(0,\xi^+)\}.
\]
Lemma~\ref{lem:fixed-deterministic-bounds} gives \(\norm h(x)\le Cr^{m-1}\)
for \(\norm{x}\le \frac{R_H}{2}\), while \eqref{eq:global-jacobian-bounds} gives
\(\norm{g(x,\xi^+)-g(0,\xi^+)}
\le (2-\theta)L_{\mathsf G}r/(1-\theta)\).  Hence, after
reducing the stepsize threshold, \(\norm v\le Cr\).  The function
\(y\mapsto\omega(y)\) is \(s_m\)-H\"older, because
\((1-s^{s_m})_+\) is \(s_m\)-H\"older on \([0,\infty)\).  Therefore
\[
        \omega(f_{\xi^+}(x))
        \le
        \omega(-\alpha g(0,\xi^+))+C\norm v^{s_m}
        \le
        \left(1-\left(\frac{2\alpha\norm{g(0,\xi^+)}}{R_H}\right)^{s_m}\right)_+
        +Cr^{s_m}.
\]
Since
\[
        \left(1-\left(\frac{2z}{R_H}\right)^{s_m}\right)_+
        \le 1-c\min\{z^{s_m},1\},
        \qquad z\ge0,
\]
and \(\omega(x)=1-(\frac{2r}{R_H})^{s_m}\) for \(r\le \frac{R_H}{2}\),
\eqref{eq:fixed-core-omega-drop-new} follows.

Combining \eqref{eq:fixed-core-pre}, \eqref{eq:fixed-core-omega-drop-new},
\eqref{eq:core-overshoot-restored}, and Lemma~\ref{lem:fixed-truncated-bounds},
we get
\begin{equation}
\label{eq:core-combined-before-absorption}
        (\mathcal K_{\xi,\alpha}V_\alpha)(x;e)-V_\alpha(x)
        \le
        -a_{\rm in}\alpha r^{m-2}
        -c\delta_\alpha\bar a_{\alpha,s_m}
        +C\delta_\alpha r^{s_m}
        +Ce^{-c/\alpha}.
\end{equation}
If \(2<m\le3\), then \(s_m=m-2\) and \(\delta_\alpha=\kappa_0\alpha\).
Using \(\bar a_{\alpha,s_m}\ge c\bar a_\alpha^{m-2}\),
\eqref{eq:core-combined-before-absorption} becomes
\[
        (\mathcal K_{\xi,\alpha}V_\alpha)(x;e)-V_\alpha(x)
        \le
        -(a_{\rm in}-C\kappa_0)\alpha r^{m-2}
        -c\kappa_0\alpha\bar a_\alpha^{m-2}
        +Ce^{-c/\alpha}.
\]
Choose \(\kappa_0\) so small that \(a_{\rm in}-C\kappa_0\ge a_{\rm in}/2\).
Then \(\bar a_\alpha\ge c_1\alpha\), and the exponentially small term is
negligible compared with \(\alpha^{m-1}\).  Hence
\begin{equation}
\label{eq:core-additive-drop-flat-case}
        (\mathcal K_{\xi,\alpha}V_\alpha)(x;e)-V_\alpha(x)
        \le -c\alpha^{m-1}.
\end{equation}

If \(m>3\), then \(s_m=1\),
\(\delta_\alpha=\kappa_0\alpha^{m-2}\), and
\(\bar a_{\alpha,1}=\bar a_\alpha\).  Factoring out \(\alpha\) in
\eqref{eq:core-combined-before-absorption} gives
\[
\begin{aligned}
        (\mathcal K_{\xi,\alpha}V_\alpha)(x;e)-V_\alpha(x)
        \le
        \alpha\bigl[
        -a_{\rm in}r^{m-2}
        -c\kappa_0\alpha^{m-3}\bar a_\alpha
        +C\kappa_0\alpha^{m-3}r
        \bigr]
        +Ce^{-c/\alpha}.
\end{aligned}
\]
Put \(q:=(m-2)/(m-3)>1\).  Young's inequality yields
\[
        C\kappa_0\alpha^{m-3}r
        \le
        \frac{a_{\rm in}}{2}r^{m-2}
        +C_*\kappa_0^q\alpha^{m-2}.
\]
Since \(\bar a_\alpha\ge c_1\alpha\), the negative noise term inside the
brackets is at most \(-cc_1\kappa_0\alpha^{m-2}\).  Choose \(\kappa_0\)
so small that \(C_*\kappa_0^{q-1}\le cc_1/2\).  After restoring the outer
factor \(\alpha\), the Young remainder is absorbed by the negative noise
term.  Finally, \(e^{-c/\alpha}=o(\alpha^{m-1})\), so
\eqref{eq:core-additive-drop-flat-case} also holds for \(m>3\).

On the near-minimizer region, \(1\le V_\alpha(x)\le C\) uniformly in \(\alpha\).
Since \eqref{eq:core-additive-drop-flat-case} gives an additive decrease of at
least \(c\alpha^{m-1}\), it implies
\[
        (\mathcal K_{\xi,\alpha}V_\alpha)(x;e)
        \le
        (1-c_{\rm core}\alpha^{m-1})V_\alpha(x),
\]
after decreasing \(c_{\rm core}\).  Taking the supremum over \(e\) completes the
proof.
\end{proof}

\begin{lemma}
\label{lem:fixed-bridge}
There exist \(c_{\rm br}>0\) and \(\alpha_{\rm br}\in(0,1]\) such that, for
every \(0<\alpha\le\alpha_{\rm br}\), every \(\xi\in\Xi\), and every
\(\frac{R_H}{2}\le\norm{x}\le R_H\),
\[
        \sup_{e\in\Sph^{d-1}}
        (\mathcal K_{\xi,\alpha}V_\alpha)(x;e)
        \le
        (1-c_{\rm br}\alpha^{m-1})V_\alpha(x).
\]
\end{lemma}

\begin{proof}
Fix \(e\in\Sph^{d-1}\).  By Assumption~\ref{ass:H}\assitemref{itm:H2}, for
\(\frac{R_H}{2}\le\norm{x}\le R_H\),
\(
        \ip{e}{\nabla^2H(x)e}
        \ge
        c_{\rm in}\left(\frac{R_H}{2}\right)^{m-2}.
\)
Together with \eqref{eq:fixed-directional-basic}, this gives
\(
        \E\mathcal D_{\xi,\alpha}(x,e)
        \le
        1-c\alpha
\)
for some \(c>0\).  Since \(\mathcal D_{\xi,\alpha}(x,e)\le1\),
\[
\begin{aligned}
        (\mathcal K_{\xi,\alpha}V_\alpha)(x;e)
        &\le
        V_\alpha(x)\E\mathcal D_{\xi,\alpha}(x,e)
        +\E[(V_\alpha(f_{\xi^+}(x))-V_\alpha(x))_+]  \\
        &\le
        (1-c\alpha+C(\kappa+\kappa_0)\alpha)V_\alpha(x),
\end{aligned}
\]
where the last line uses \eqref{eq:compact-positive-increment}.  Choose
\(\kappa\) and \(\kappa_0\) sufficiently small to obtain a factor
\(1-c\alpha\).  Since \(0<\alpha\le1\) and \(m\ge2\),
\(\alpha\ge\alpha^{m-1}\), which gives the stated estimate after decreasing
the constant.
\end{proof}

\begin{lemma}
\label{lem:fixed-far}
There exist \(c_{\rm far}>0\) and \(\alpha_{\rm far}\in(0,1]\) such that, for
every \(0<\alpha\le\alpha_{\rm far}\), every \(\xi\in\Xi\), and every
\(\norm{x}\ge R_H\),
\[
        \sup_{e\in\Sph^{d-1}}
        (\mathcal K_{\xi,\alpha}V_\alpha)(x;e)
        \le
        (1-c_{\rm far}\alpha^{m-1})V_\alpha(x).
\]
\end{lemma}

\begin{proof}
First suppose \(\beta=2\).  Then \(\mathcal T_\beta\equiv1\), and \(\omega(x)=0\) for
\(\norm{x}\ge R_H\), so \(V_\alpha(x)=1\).  By
Assumption~\ref{ass:H}\assitemref{itm:H3a} and \eqref{eq:fixed-directional-basic},
\(
        \E\mathcal D_{\xi,\alpha}(x,e)
        \le
        1-c\alpha
\)
for some \(c>0\).
Since \(0\le\omega\le1\),
\(
        (\mathcal K_{\xi,\alpha}V_\alpha)(x;e)
        \le
        1-c\alpha+\delta_\alpha.
\)
Since \(\delta_\alpha\le\kappa_0\alpha\), choosing \(\kappa_0\) small gives
a factor \(1-c\alpha\).

Now suppose \(1\le\beta<2\).  On \(\norm{x}\ge R_H\), \(\omega(x)=0\), so
\(V_\alpha(x)=\mathcal T_\beta(x)\).  By \eqref{eq:fixed-directional-nonexpansive} and
Lemma~\ref{lem:subquadratic-tail-decrease},
\[
        (\mathcal K_{\xi,\alpha}V_\alpha)(x;e)
        \le
        \E \mathcal T_\beta(f_{\xi^+}(x))+\delta_\alpha
        \le
        (1-c\alpha)\mathcal T_\beta(x)+\delta_\alpha.
\]
Since \(\mathcal T_\beta(x)\ge \exp\{\kappa(R_H^{2-\beta}-(\tfrac{3R_H}{4})^{2-\beta})\}>1\) on the far
region and \(\delta_\alpha\le C\kappa_0\alpha\), choosing \(\kappa_0\) small
absorbs the last term and again gives a factor \(1-c\alpha\).  In both tail
cases, \(\alpha\ge\alpha^{m-1}\); taking the supremum over \(e\) and
decreasing the constant proves the stated estimate.
\end{proof}

\begin{proof}[Proof of Lemma~\ref{lem:fixed-conditional-drift}]
The estimates in Lemmas~\ref{lem:fixed-core}, \ref{lem:fixed-bridge},
and~\ref{lem:fixed-far} give \eqref{eq:conditional-directional-drift-main}
after taking
\(
        c_0:=\min\{c_{\rm core},c_{\rm br},c_{\rm far}\}
\)
and reducing the common stepsize threshold.

It remains to prove the growth estimate \eqref{eq:conditional-growth-main}.  On
\(\norm{x}\le R_H\), it is exactly \eqref{eq:compact-growth}.  If
\(\norm{x}\ge R_H\) and \(\beta=2\), then \(V_\alpha(x)=1\) and
\(V_\alpha(f_{\xi^+}(x))\le1+\delta_\alpha\le(1+C\alpha)V_\alpha(x)\).  If
\(\norm{x}\ge R_H\) and \(1\le\beta<2\), then
Lemma~\ref{lem:subquadratic-tail-decrease} and \(0\le\omega\le1\) give
\[
        \E V_\alpha(f_{\xi^+}(x))
        \le
        (1-c\alpha)\mathcal T_\beta(x)+\delta_\alpha
        \le
        (1+C\alpha)V_\alpha(x),
\]
again after reducing the threshold and increasing \(C\).  This proves
Lemma~\ref{lem:fixed-conditional-drift}.
\end{proof}

\subsection{Proof of Lemma~\ref{lem:aug-lipschitz}}
\label{subsec:augmented-induced-metric-proof}

Set
\(
        \rho_\Xi:=\E L_\Phi(U_1)<1,
\)
where the inequality follows from Assumption~\ref{ass:N}\assitemref{itm:N1}.

\begin{proof}[Proof of Lemma~\ref{lem:aug-lipschitz}]
Let \(C_V\) be as in Lemma~\ref{lem:fixed-conditional-drift}.
Let \(C_+\) be as in Lemma~\ref{lem:global-weight-positive-increment}.
Choose \(\alpha_0\in(0,1]\) small enough that the conclusions of
Lemmas~\ref{lem:fixed-conditional-drift}
and~\ref{lem:global-weight-positive-increment} hold for
\(0<\alpha\le\alpha_0\) and
\begin{equation}
\label{eq:augmented-xi-dominance-new}
        \rho_\Xi+C_+\alpha
        +\alpha^2L_{g,\Phi}(1+C_V\alpha)
        \le
        1-c_0\alpha^{m-1},
        \qquad 0<\alpha\le\alpha_0.
\end{equation}
This is possible because \(\rho_\Xi<1\) and
\(\alpha^{m-1}\to0\).

Fix \(z=(x,\xi)\), \(z'=(y,\eta)\), and an absolutely continuous path
\(\gamma(t)=(x(t),\xi(t))\) from \(z\) to \(z'\).  For fixed \(u\), write
\[
        \xi^+(t):=\Phi(\xi(t),u),
        \qquad
        X^+(t):=x(t)-\alpha\{h(x(t))+g(x(t),\xi^+(t))\}.
\]
The curve \(F_u\circ\gamma=(X^+(\cdot),\xi^+(\cdot))\) is absolutely
continuous.  Indeed, \(\Phi(\cdot,u)\) is Lipschitz by
Assumption~\ref{ass:N}\assitemref{itm:N1}, the composed dependence
\(g(x,\Phi(\xi,u))\) is Lipschitz in \(\xi\) by
Assumption~\ref{ass:N}\assitemref{itm:N2}, and the map \(x\mapsto h(x)+g(x,\zeta)\)
has derivative \(A(x,\zeta)\), which is uniformly bounded by
\eqref{eq:global-jacobian-bounds}.

Let \(\operatorname{md}_\alpha\) denote the metric derivative with respect to
the augmented base norm \(\vert\cdot\vert_\alpha\).  For a.e. \(t\),
\begin{equation}
\label{eq:augmented-metric-derivative-bound}
\begin{aligned}
        \operatorname{md}_\alpha(F_u\circ\gamma)(t)
        &\le
        \norm{(I_d-\alpha A(x(t),\xi^+(t)))\dot x(t)}  \\
        &\qquad
        +(L_\Phi(u)+\alpha^2L_{g,\Phi})\alpha^{-1}\norm{\dot\xi(t)}.
\end{aligned}
\end{equation}
To prove \eqref{eq:augmented-metric-derivative-bound}, compare
\(F_u(\gamma(s))\) and \(F_u(\gamma(t))\).  In the \(x\)-coordinate, first
freeze the noise argument at \(\xi^+(t)\); division by \(|s-t|\) and passage to
the a.e.\ differentiability point gives the derivative
\((I_d-\alpha A(x(t),\xi^+(t)))\dot x(t)\).  The remaining change in the
\(x\)-coordinate is bounded by
\(\alpha L_{g,\Phi}\norm{\xi(s)-\xi(t)}\).  In the noise coordinate,
Assumption~\ref{ass:N}\assitemref{itm:N1} gives
\(\norm{\Phi(\xi(s),u)-\Phi(\xi(t),u)}
\le L_\Phi(u)\norm{\xi(s)-\xi(t)}\).  Dividing by \(|s-t|\) and using the
\(\alpha^{-1}\)-weight in the augmented norm gives
\eqref{eq:augmented-metric-derivative-bound}.

If \(\dot x(t)\ne0\), set \(e(t)=\dot x(t)/\norm{\dot x(t)}\); otherwise choose
any measurable unit vector \(e(t)\).  By the definition of induced length,
\eqref{eq:augmented-metric-derivative-bound},
Lemmas~\ref{lem:fixed-conditional-drift}
and~\ref{lem:global-weight-positive-increment}, and
\eqref{eq:augmented-xi-dominance-new},
\[
\begin{aligned}
&\E\, d_{V_\alpha,\alpha}\bigl(F_{U_1}(z),F_{U_1}(z')\bigr)  \\
&\quad\le
\int_0^1
\E\!\left[
V_\alpha(X^+(t))
\norm{(I_d-\alpha A(x(t),\xi^+(t)))\dot x(t)}
\right]dt                                      \\
&\qquad
+
\int_0^1
\E\!\left[
(L_\Phi(U_1)+\alpha^2L_{g,\Phi})V_\alpha(X^+(t))
\right]
\alpha^{-1}\norm{\dot\xi(t)}\,dt               \\
&\quad\le
(1-c_0\alpha^{m-1})
\int_0^1
V_\alpha(x(t))
\left(\norm{\dot x(t)}+\alpha^{-1}\norm{\dot\xi(t)}\right)dt.
\end{aligned}
\]
Taking the infimum over all admissible paths \(\gamma\) gives
\eqref{eq:main-contraction-aug}.
\end{proof}

\subsection{Proof of Lemma~\ref{lem:fixed-reference-estimates}}
\label{subsec:fixed-reference-estimates-proof}

\begin{proof}[Proof of Lemma~\ref{lem:fixed-reference-estimates}]
Let \(z_\star=(0,\xi_\star)\), where \(\xi_\star\) is the point from Assumption~\ref{ass:N}\assitemref{itm:N3}.  Write
\[
        \xi_\star^+:=\Phi(\xi_\star,U_1),
        \qquad
        x_{\rm ref}^+:=-\alpha g(0,\xi_\star^+).
\]
Move from \(z_\star=(0,\xi_\star)\) to \((0,\xi_\star^+)\), and then from \((0,\xi_\star^+)\) to \((x_{\rm ref}^+,\xi_\star^+)\).  Since \(V_\alpha\) depends only on the \(x\)-coordinate,
\[
        d_{V_\alpha,\alpha}\bigl(F_{U_1}(z_\star),z_\star\bigr)
        \le
        V_\alpha(0)\alpha^{-1}\norm{\xi_\star^+-\xi_\star}
        +
        d_{V_\alpha}^{X}(x_{\rm ref}^+,0).
\]
The first term is integrable by Assumption~\ref{ass:N}\assitemref{itm:N3}.

If \(\beta=2\), then \(\mathcal T_\beta\equiv1\) and \(V_\alpha\le1+\delta_\alpha\).  Hence
\[
        d_{V_\alpha}^{X}(x_{\rm ref}^+,0)
        \le
        (1+\delta_\alpha)\norm{x_{\rm ref}^+}
        =
        (1+\delta_\alpha)\alpha\norm{g(0,\xi_\star^+)},
\]
whose expectation is finite by Assumption~\ref{ass:N}\assitemref{itm:N3}.

If \(\beta<2\), Assumption~\ref{ass:N}\assitemref{itm:N5} at \(x=0\) and \(\xi=\xi_\star\) gives an exponential moment of \(g(0,\xi_\star^+)\).  Along the straight line from \(0\) to \(x_{\rm ref}^+\),
\[
        d_{V_\alpha}^{X}(x_{\rm ref}^+,0)
        \le
        \alpha\norm{g(0,\xi_\star^+)}
        \left[
        1+\delta_\alpha+
        \exp\!\left\{\kappa\bigl(\alpha^{2-\beta}\norm{g(0,\xi_\star^+)}^{2-\beta}-(\tfrac{3R_H}{4})^{2-\beta}\bigr)_+\right\}
        \right].
\]
Since \(2-\beta\in(0,1]\), \(r^{2-\beta}\le1+r\) for \(r\ge0\), we may choose \(\alpha_0\in(0,1]\) such that, for \(0<\alpha\le\alpha_0\), the exponential factor above is dominated by
\(
        C e^{\lambda\norm{g(0,\xi_\star^+)}}
\)
for some \(\lambda>0\) below the exponential-moment threshold supplied by \assitemref{itm:N5}.  Thus
\(
        \E d_{V_\alpha}^{X}(x_{\rm ref}^+,0)<\infty.
\)
This proves \eqref{eq:fixed-reference-integrability}.
\end{proof}

\subsection{Proof of Corollary~\ref{cor:ordinary-W1-markov}}
\label{sec:proof-ordinary-W1}
\begin{proof}[Proof of Corollary~\ref{cor:ordinary-W1-markov}]
Since \(V_\alpha\ge1\), the augmented induced metric dominates the
Euclidean distance in the $X$-coordinate:
\(
        d_{V_\alpha,\alpha}\bigl((x,\xi),(y,\eta)\bigr)
        \ge \norm{x-y}.
\)
Hence, for probability laws \(\mu,\nu\) on \(\mathsf Z\),
\[
        W_1(\mu_X,\nu_X)\le W_{V_\alpha,\alpha}(\mu,\nu).
\]
Applying this with \(\nu=\pi_\alpha\), and then using
Theorem~\ref{thm:main-markov}, gives
\[
        W_1\bigl((\mu P_\alpha^n)_X,(\pi_\alpha)_X\bigr)
        \le
        W_{V_\alpha,\alpha}(\mu P_\alpha^n,\pi_\alpha)
        \le
        (1-c\alpha^{m-1})^n
        W_{V_\alpha,\alpha}(\mu,\pi_\alpha).
\]
Since \(\pi_\alpha\in\mathcal P_1(\mathsf Z,d_{V_\alpha,\alpha})\), the final
quantity is finite if and only if
\(\mu\in\mathcal P_1(\mathsf Z,d_{V_\alpha,\alpha})\).

It remains to characterize \(\mathcal P_1(\mathsf Z,d_{V_\alpha,\alpha})\).  Since
\(
        \vert(x,\xi)\vert_\alpha
        =
        \norm{x}+\alpha^{-1}\norm{\xi},
\)
and since \(V_\alpha\ge1\), projection of paths gives the lower bounds
\[
        d_{V_\alpha,\alpha}\bigl((x,\xi),(0,\xi_\star)\bigr)
        \ge
        d_{V_\alpha}^{X}(x,0),
        \qquad
        d_{V_\alpha,\alpha}\bigl((x,\xi),(0,\xi_\star)\bigr)
        \ge
        \alpha^{-1}\norm{\xi-\xi_\star}.
\]
Conversely, moving first in the \(x\)-coordinate and then in the
\(\xi\)-coordinate gives
\[
        d_{V_\alpha,\alpha}\bigl((x,\xi),(0,\xi_\star)\bigr)
        \le
        d_{V_\alpha}^{X}(x,0)
        +
        V_\alpha(0)\alpha^{-1}\norm{\xi-\xi_\star}.
\]
Thus \(d_{V_\alpha,\alpha}((x,\xi),(0,\xi_\star))\)-integrability is
equivalent to integrability of
\(
        d_{V_\alpha}^{X}(x,0)
        +
        \alpha^{-1}\norm{\xi-\xi_\star}.
\)

If \(\beta=2\), then \(\mathcal T_\beta\equiv1\) and
\(1\le V_\alpha\le1+\delta_\alpha\).  Hence
\(
        \norm{x-y}
        \le
        d_{V_\alpha}^{X}(x,y)
        \le
        (1+\delta_\alpha)\norm{x-y},
\)
so \(d_{V_\alpha}^{X}(x,0)\)-integrability is equivalent to
\(\norm{x}\)-integrability.

If \(1\le\beta<2\), let \(R_H\) be as in
Assumption~\ref{ass:H}\assitemref{itm:H2} and define
\[
        \Psi_\kappa(r)
        :=
        \int_0^r
        \exp\{\kappa\bigl(s^{2-\beta}-(\tfrac{3R_H}{4})^{2-\beta}\bigr)_+\}\,ds.
\]
Because the correction \(\delta_\alpha\omega\) is compactly supported and
\(0\le\omega\le1\),
\[
        \Psi_\kappa(\norm{x})
        \le
        d_{V_\alpha}^{X}(x,0)
        \le
        \Psi_\kappa(\norm{x})+\frac{\delta_\alpha R_H}{2}.
\]
Moreover,
\[
        \Psi_\kappa(r)
        \sim
        \frac{e^{-\kappa(\frac{3R_H}{4})^{2-\beta}}}{\kappa(2-\beta)}
        r^{\beta-1}
        e^{\kappa r^{2-\beta}},
        \qquad r\to\infty.
\]
Therefore \(d_{V_\alpha}^{X}(x,0)\)-integrability is equivalent to
\[
        \int
        (1+\norm{x})^{\beta-1}
        \exp\{\kappa\norm{x}^{2-\beta}\}\,\mu(dx,d\xi)
        <\infty.
\]
Since \(2-\beta\in(0,1]\), for \(r\ge0\),
\(
        r^{2-\beta}\le(1+r)^{2-\beta}\le1+r^{2-\beta},
\)
and hence
\[
        e^{-\kappa}e^{\kappa r^{2-\beta}}
        \le
        \exp\!\left\{\kappa\bigl((1+r)^{2-\beta}-1\bigr)\right\}
        \le
        e^{\kappa r^{2-\beta}}.
\]
Thus the preceding integrability condition is equivalent to integrability of
\(\Gamma_\beta(\norm{x})\).
Combining this with the \(\xi\)-coordinate term gives the displayed description of
\(\mathcal P_1(\mathsf Z,d_{V_\alpha,\alpha})\).
\end{proof}

\section{Scaling-limit estimates}
\label{sec:scaling-technical-estimates}

This appendix proves the auxiliary results used for the scaling limit of
invariant laws.  Constants denoted by \(C,c\) may change from line to line but
are independent of \(\alpha\), unless stated otherwise.

\subsection{Deterministic scaling estimates}

\begin{lemma}
\label{lem:scaling-deterministic-estimates}
Under Assumption~\ref{ass:H}, the following bounds hold.
\begin{enumerate}[label=\textup{(\roman*)}, leftmargin=2.6em]
\item If \(\norm{x}\le \frac{R_H}{2}\), then
\(
        \ip{x}{h(x)}
        \ge
        \frac{c_{\rm in}}{m-1}\norm{x}^m.
\)

\item There exists \(C_{\rm tail}>0\) such that
\[
        \norm{x}^\beta
        \le
        C_{\rm tail}\bigl(1+\ip{x}{h(x)}\bigr),
        \qquad 1\le\beta<2,
\]
and
\[
        \norm{x}^2
        \le
        C_{\rm tail}\bigl(1+\ip{x}{h(x)}\bigr),
        \qquad \beta=2.
\]

\item With \(r_0:=\min\{1,\frac{R_H}{2}\}\), there exists \(c_*>0\) such that
\[
        \ip{x}{h(x)}\ge c_*,
        \qquad \norm{x}\ge r_0.
\]

\item There exists \(C_a>0\) such that, for all \(x\in\R^d\),
\(
        (1+\norm{x})(1+\norm{x}^{\beta-1})
        \le
        C_a\bigl(1+\ip{x}{h(x)}\bigr).
\)
\end{enumerate}
\end{lemma}

\begin{proof}
For the lower bound near the minimizer, use
\(
        h(x)=\int_0^1\nabla^2H(tx)x\,dt.
\)
Assumption~\ref{ass:H}\assitemref{itm:H2} gives
\[
        \ip{x}{h(x)}
        =
        \int_0^1\ip{x}{\nabla^2H(tx)x}\,dt
        \ge
        c_{\rm in}\norm{x}^m\int_0^1t^{m-2}\,dt
        =
        \frac{c_{\rm in}}{m-1}\norm{x}^m.
\]

If \(1\le\beta<2\), the tail condition \(\ip{x}{h(x)}\ge c_{\rm out}\norm{x}^\beta\) on \(\{\norm{x}\ge R_H\}\), together with boundedness on \(\{\norm{x}<R_H\}\), gives
\(
        \norm{x}^{\beta}
        \le C\bigl(1+\ip{x}{h(x)}\bigr).
\)
If \(\beta=2\), then for \(\norm{x}\ge2R_H\), convexity and minimality at zero imply
\(
        \ip{x}{h(x/2)}\ge0.
\)
Using Assumption~\ref{ass:H}\assitemref{itm:H3a} on the segment \(\{tx:1/2\le t\le1\}\),
\[
\begin{aligned}
        \ip{x}{h(x)}
        &\ge
        \ip{x}{h(x)-h(x/2)}                                      \\
        &=
        \int_{1/2}^1\ip{x}{\nabla^2H(tx)x}\,dt
        \ge
        \frac{c_{\rm out}}2\norm{x}^2.
\end{aligned}
\]
The region \(\{\norm{x}<2R_H\}\) is absorbed into the additive constant.

Finally, write \(x=ru\), where \(r=\norm{x}\) and \(u\in\Sph^{d-1}\).
Convexity of \(H\) and minimality at zero imply that
\(s\mapsto\ip{u}{h(su)}\) is nondecreasing on \([0,\infty)\).  For
\(r\ge r_0\),
\[
        \ip{x}{h(x)}
        =
        r\ip{u}{h(ru)}
        \ge
        r_0\ip{u}{h(r_0u)}
        =
        \ip{r_0u}{h(r_0u)}
        \ge
        \frac{c_{\rm in}}{m-1}r_0^m.
\]
This proves item \textup{(iii)}.  For item \textup{(iv)}, if \(\beta=2\), then
\(
        (1+\norm{x})(1+\norm{x}^{\beta-1})
        \le C(1+\norm{x}^2),
\)
and the quadratic-tail estimate in item \textup{(ii)} gives the result.  If \(1\le\beta<2\), then
\(
        (1+\norm{x})(1+\norm{x}^{\beta-1})
        \le C(1+\norm{x}^\beta),
\)
and the subquadratic-tail estimate in item \textup{(ii)} gives the result.  This proves the lemma.
\end{proof}

\subsection{Driving-chain ergodicity and Poisson equation}

\begin{lemma}
\label{lem:xi-invariant-law}
Assume Assumption~\ref{ass:N}\assitemref{itm:N1} and the reference-point condition \eqref{eq:markov-reference-state}.  Let \(Q\) be the transition kernel of
\(
        \xi_{n+1}=\Phi(\xi_n,U_{n+1}).
\)
Then \(Q\) admits an invariant law \(\pi_\Xi\in\mathcal P_1(\Xi)\), unique among all Borel invariant probability laws on \(\Xi\).  Moreover,
\(
        \E\norm{\xi_n^\xi-\xi_n^\eta}
        \le\rho_\Xi^n\norm{\xi-\eta},
\)
for synchronously coupled chains started from \(\xi\) and \(\eta\), and
\[
        W_1(\mu Q,\nu Q)\le \rho_\Xi W_1(\mu,\nu),
        \qquad \mu,\nu\in\mathcal P_1(\Xi),
\]
and if \(\pi_\alpha\) is the invariant law of the augmented chain, then its \(\Xi\)-marginal is \(\pi_\Xi\).
\end{lemma}

\begin{proof}
Let \(\xi_\star\) be as in Assumption~\ref{ass:N}\assitemref{itm:N3}.  If \(\eta\sim\mu\in\mathcal P_1(\Xi)\), then Assumption~\ref{ass:N}\assitemref{itm:N1} gives
\[
\begin{aligned}
        \E\norm{\Phi(\eta,U_1)-\xi_\star}
        &\le
        \E\norm{\Phi(\eta,U_1)-\Phi(\xi_\star,U_1)}
        +
        \E\norm{\Phi(\xi_\star,U_1)-\xi_\star}     \\
        &\le
        \rho_\Xi\E\norm{\eta-\xi_\star}
        +
        \E\norm{\Phi(\xi_\star,U_1)-\xi_\star}.
\end{aligned}
\]
Thus \(Q\) maps \(\mathcal P_1(\Xi)\) into itself.

For any coupling \((\eta,\widetilde\eta)\) of \((\mu,\nu)\), drive both chains by the same innovation \(U_1\).  Then
\[
        \E\norm{\Phi(\eta,U_1)-\Phi(\widetilde\eta,U_1)}
        \le
        \rho_\Xi\E\norm{\eta-\widetilde\eta}.
\]
Iteration with independent innovations gives the stated synchronous
\(n\)-step estimate.
Taking the infimum over couplings yields
\(
        W_1(\mu Q,\nu Q)\le\rho_\Xi W_1(\mu,\nu).
\)
Since \(\Xi\) is closed in a finite-dimensional Euclidean space, \((\mathcal P_1(\Xi),W_1)\) is complete.  Banach's fixed-point theorem gives a unique invariant law \(\pi_\Xi\in\mathcal P_1(\Xi)\).

To prove uniqueness among all Borel invariant laws, let \(\rho\) be such a law
on \(\Xi\).  For each fixed \(\xi\), the contraction estimate gives
\(\delta_\xi Q^n\to\pi_\Xi\) in \(W_1\).  Hence, for every bounded Lipschitz
\(\varphi\),
\(
        Q^n\varphi(\xi)\to\pi_\Xi(\varphi).
\)
By invariance and dominated convergence,
\(
        \rho(\varphi)
        =
        \rho(Q^n\varphi)
        \to
        \pi_\Xi(\varphi).
\)
Bounded Lipschitz functions determine Borel probability laws on the Polish space \(\Xi\), so \(\rho=\pi_\Xi\).

Finally, if \((X,\xi)\sim\pi_\alpha\), invariance of the augmented chain
implies that \(\Phi(\xi,U_1)\) has the same law as \(\xi\).  Thus the
\(\Xi\)-marginal of \(\pi_\alpha\) is invariant for \(Q\) and equals
\(\pi_\Xi\).
\end{proof}

\begin{lemma}
\label{lem:poisson-noise}
Assume the hypotheses of Lemma~\ref{lem:xi-invariant-law}.  Let \(f:\Xi\to\R\) be bounded, Lipschitz, and centered under \(\pi_\Xi\).  Define
\(
        u(\xi):=\sum_{k=0}^{\infty}Q^kf(\xi).
\)
Then the series converges absolutely for every \(\xi\), \(u\) is Lipschitz with
\(
        \Lip(u)\le\frac{\Lip(f)}{1-\rho_\Xi},
\)
\(u\in L^2(\pi_\Xi)\), and
\(
        u-Qu=f.
\)
\end{lemma}

\begin{proof}
Let \(\xi_k^\xi\) be the driving chain started from \(\xi\).  Let \(\eta_0\sim\pi_\Xi\) and drive the stationary chain \((\eta_k)\) by the same innovations.  Since \(f\) is centered under \(\pi_\Xi\),
\(
        Q^kf(\xi)
        =
        \E[f(\xi_k^\xi)-f(\eta_k)].
\)
Therefore
\(
        \abs{Q^kf(\xi)}
        \le
        \min\left\{
        2\norm{f}_\infty,\,
        \Lip(f)\rho_\Xi^k\bigl(\norm{\xi}+\E_{\pi_\Xi}\norm{\eta}\bigr)
        \right\}.
\)
This bound proves absolute convergence.  Similarly, synchronous coupling of two chains started from \(\xi\) and \(\eta\) gives
\(
        \abs{Q^kf(\xi)-Q^kf(\eta)}
        \le
        \Lip (f)\rho_\Xi^k\norm{\xi-\eta},
\)
and summing over \(k\) proves the Lipschitz bound for \(u\).

The first bound also implies logarithmic growth:
\(
        \abs{u(\xi)}
        \le
        C_f\bigl(1+\log(1+\norm{\xi})\bigr).
\)
Since \(\pi_\Xi\in\mathcal P_1(\Xi)\) and \(\log^2(1+r)\le C(1+r)\), we get \(u\in L^2(\pi_\Xi)\).  Finally, absolute convergence allows termwise application of \(Q\), giving
\[
        Qu=\sum_{k=1}^{\infty}Q^kf,
        \qquad
        u-Qu=f.
\]
\end{proof}

\subsection{Proof of Lemma~\ref{lem:poisson-gk-construction}}
\label{sec:proof-poisson-gk-construction}
\begin{proof}[Proof of Lemma~\ref{lem:poisson-gk-construction}]
Let \(\xi_k^\xi\) and \(\xi_k^\eta\) be two copies of the driving chain,
started from \(\xi\) and \(\eta\) and coupled through the same innovations.
By Assumption~\ref{ass:N}\assitemref{itm:N1},
\(
        \E\norm{\xi_k^\xi-\xi_k^\eta}
        \le \rho_\Xi^k\norm{\xi-\eta}.
\)
For \(k\ge1\), Assumption~\ref{ass:N}\assitemref{itm:N2} gives
\[
        \norm{Q^kg_x(\xi)-Q^kg_x(\eta)}
        \le
        L_{g,\Phi}\rho_\Xi^{k-1}\norm{\xi-\eta}.
\]
Let \(\eta\sim\pi_\Xi\), independent of the driving innovations.  Since \(\pi_\Xi g_x=0\) by \eqref{eq:stationary-centering-gkx},
\[
        Q^kg_x(\xi)=\E[g_x(\xi_k^\xi)-g_x(\xi_k^\eta)],
        \qquad k\ge1.
\]
Consequently
\[
        \norm{Q^kg_x(\xi)}
        \le
        L_{g,\Phi}\rho_\Xi^{k-1}
        \bigl(\norm{\xi-\xi_\star}
        +\E_{\pi_\Xi}\norm{\eta-\xi_\star}\bigr),
\]
and the series defining \(\chi_x\) converges absolutely for every \(\xi\).
Applying \(Q\) termwise gives \(\chi_x-Q\chi_x=Qg_x\).  The same estimate
gives
\begin{equation}
\label{eq:poisson-solution-growth}
        \norm{\chi_x(\xi)}
        \le C_\chi R_\chi(\xi),
        \qquad
        R_\chi(\xi):=1+\norm{\xi-\xi_\star},
\end{equation}
with \(R_\chi\in L^2(\pi_\Xi)\) by \eqref{eq:g0-moment-gk}.  Centering of \(\chi_x\) follows from invariance of \(\pi_\Xi\):
\(
        \int_\Xi \chi_x(\xi)\,\pi_\Xi(d\xi)=0,
\)
because \(\int Q^kg_x\,d\pi_\Xi=\int g_x\,d\pi_\Xi=0\) for every \(k\ge1\).

It remains to control the dependence on \(x\).  The uniform derivative bound
\eqref{eq:global-jacobian-bounds} shows that, for some \(L_x>0\),
\[
        \norm{g_x(\zeta)-g_y(\zeta)}
        \le L_x\norm{x-y},
        \qquad x,y\in\R^d,\ \zeta\in\Xi.
\]
Fix \(p\in(0,1)\), and put \(f_{x,y}:=g_x-g_y\).  By
\eqref{eq:stationary-centering-gkx}, \(\pi_\Xi f_{x,y}=0\).  On the one hand,
\[
        \norm{Q^kf_{x,y}(\xi)}
        \le L_x\norm{x-y},
        \qquad k\ge1.
\]
On the other hand, comparison with the stationary chain used above and
Assumption~\ref{ass:N}\assitemref{itm:N2}, applied separately to \(g_x\)
and \(g_y\), give
\[
\begin{aligned}
        \norm{Q^kf_{x,y}(\xi)}
        &\le
        2L_{g,\Phi}\rho_\Xi^{k-1}
        \bigl(\norm{\xi-\xi_\star}
        +\E_{\pi_\Xi}\norm{\eta-\xi_\star}\bigr)\\
        &\le C\rho_\Xi^{k-1}R_\chi(\xi).
\end{aligned}
\]
Since \(\min\{a,b\}\le a^pb^{1-p}\) for \(a,b\ge0\), summing these two
bounds over \(k\ge1\) yields the H\"older estimate
\begin{equation}
\label{eq:poisson-solution-x-holder}
        \norm{\chi_x(\xi)-\chi_y(\xi)}
        \le C_p\norm{x-y}^pR_\chi(\xi)^{1-p},
        \qquad x,y\in\R^d,\ \xi\in\Xi.
\end{equation}
The martingale-difference property \eqref{eq:Dx-martingale-difference} is
immediate from \eqref{eq:chi-poisson}.

Finally, let \(f(\xi)=g(0,\xi)\).  The geometric estimate used to prove
\eqref{eq:poisson-solution-growth}, with \(x=0\), gives
\[
        \norm{Q^kf(\xi)}
        \le C\rho_\Xi^{k-1}(1+\norm{\xi-\xi_\star}),
        \qquad k\ge1.
\]
Together with \eqref{eq:g0-moment-gk}, this implies absolute summability of
the lag covariances.  Write the stationary chain as
\((\xi_n)_{n\in\mathbb Z}\), set \(f_n=f(\xi_n)\) and
\(\chi_n=\chi_0(\xi_n)\), and define
\(
        D_n:=f_{n+1}+\chi_{n+1}-\chi_n.
\)
Then \((D_n)\) is a martingale-difference sequence and
\(
        f_{n+1}=D_n+\chi_n-\chi_{n+1}.
\)
Therefore, for \(N\ge1\),
\(
        \sum_{n=1}^N f_n
        =
        \sum_{n=0}^{N-1}D_n+\chi_0-\chi_N.
\)
Since \(\chi_0\in L^2(\pi_\Xi)\), the telescoping term is negligible after division by \(N\) in the covariance of the partial sums.  Hence
\[
        \lim_{N\to\infty}
        \frac1N\Var\!\left(\sum_{n=1}^N f_n\right)
        =
        \E[D_0D_0^\top]
        =
        \Sigma.
\]
On the other hand, stationarity and absolute summability of the lag covariances give
\[
        \lim_{N\to\infty}
        \frac1N\Var\!\left(\sum_{n=1}^N f_n\right)
        =
        \Gamma_0+\sum_{k\ge1}(\Gamma_k+\Gamma_k^\top).
\]
This proves \eqref{eq:GK-lag-series} and completes the proof of the lemma.
\end{proof}

\subsection{Proof of Lemma~\ref{lem:scaling-tightness-package}}
\label{sec:proof-scaling-tightness-package}

\begin{lemma}
\label{lem:scaling-noise-second-markov}
Assume Assumptions~\ref{ass:Hscale} and~\ref{ass:Nscale}.  Then there exist constants \(A,B>0\), independent of \(\alpha\), such that
\(
        \E\norm{g(X_\alpha,\xi_\alpha^+)}^2
        \le
        A+B\,\E\ip{X_\alpha}{h(X_\alpha)}.
\)
\end{lemma}

\begin{proof}
By Lemma~\ref{lem:xi-invariant-law}, \(\xi_\alpha^+\sim\pi_\Xi\).

If \(\beta=2\), \eqref{eq:global-jacobian-bounds} gives a uniform
\(x\)-Lipschitz constant \((2-\theta)L_{\mathsf G}/(1-\theta)\) for \(g\).  Hence
\[
        \norm{g(x,\zeta)}^2
        \le
        2\norm{g(0,\zeta)}^2
        +2\left(\frac{(2-\theta)L_{\mathsf G}}{1-\theta}\right)^2\norm{x}^2.
\]
The stationary second-moment condition \eqref{eq:g0-moment-gk} gives
\(
        \E\norm{g(0,\xi_\alpha^+)}^2<\infty,
\)
and Lemma~\ref{lem:scaling-deterministic-estimates} gives
\(
        \norm{x}^2\le C(1+\ip{x}{h(x)}).
\)
Taking expectations proves the claim in the quadratic-tail case.

If \(1\le\beta<2\), Assumption~\ref{ass:N}\assitemref{itm:N5} gives a uniform second moment for
\(
        \frac{g(x,\Phi(\xi,U_1))}{1+\norm{x}^{\beta-1}}.
\)
Therefore, conditionally on \((X_\alpha,\xi_\alpha)\),
\[
        \E\bigl[\norm{g(X_\alpha,\xi_\alpha^+)}^2\mid X_\alpha,\xi_\alpha\bigr]
        \le
        C(1+\norm{X_\alpha}^{\beta-1})^2
        \le
        C(1+\norm{X_\alpha}^{\beta}).
\]
Using again Lemma~\ref{lem:scaling-deterministic-estimates},
\(
        \norm{x}^{\beta}\le C(1+\ip{x}{h(x)}),
\)
and taking expectations proves the claim.
\end{proof}

\begin{lemma}
\label{lem:scaling-stationary-L2-markov}
Assume Assumptions~\ref{ass:Hscale} and~\ref{ass:Nscale}.  For all sufficiently small \(\alpha\),
\[
        \E\norm{X_\alpha}^2<\infty,
        \qquad
        \E\norm{h(X_\alpha)+g(X_\alpha,\xi_\alpha^+)}^2<\infty.
\]
Moreover,
\begin{equation}
\label{eq:stationary-cross-poisson-bound}
        \E\ip{X_\alpha}{g(X_\alpha,\xi_\alpha^+)}
        \ge
        -\theta\E\ip{X_\alpha}{h(X_\alpha)}
        -C\alpha\left(1+\E\ip{X_\alpha}{h(X_\alpha)}\right).
\end{equation}
\end{lemma}

\begin{proof}
We first justify the second moments.  If \(1\le\beta<2\), the induced
first-moment bound in Theorem~\ref{thm:main-markov} implies that \(X_\alpha\)
has finite moments of every polynomial order.  Lemma~\ref{lem:scaling-noise-second-markov}
and the global Lipschitz bound on \(h\) from Lemma~\ref{lem:N7-consequences}
then give
\(
        \E\norm{h(X_\alpha)+g(X_\alpha,\xi_\alpha^+)}^2<\infty.
\)

For \(\beta=2\), Lemma~\ref{lem:scaling-deterministic-estimates} gives the
global coercivity bound
\begin{equation}
\label{eq:h-quadratic-coercive-for-L2}
        \ip{x}{h(x)}\ge c\norm{x}^2-C,
        \qquad x\in\R^d.
\end{equation}
Indeed, outside a sufficiently large ball this follows from the quadratic-tail lower Hessian bound, while the remaining bounded region is absorbed into the additive constant.

Start the augmented chain from \(X_0=0\) and \(\xi_0\sim\pi_\Xi\).  Then \(\xi_n\sim\pi_\Xi\) for all \(n\).  Define the martingale-corrected quadratic Lyapunov function
\(
        J_\alpha^0(x,\xi)
        :=\norm{x}^2-2\alpha\ip{x}{\chi_0(\xi)}.
\)
For any law whose \(\Xi\)-marginal is \(\pi_\Xi\), \eqref{eq:poisson-solution-growth}, Young's inequality, and \(R_\chi\in L^2(\pi_\Xi)\) give
\begin{equation}
\label{eq:Lalpha-comparable-L2}
        \frac12\E\norm{X}^2-C\alpha^2
        \le
        \E J_\alpha^0(X,\xi)
        \le
        \frac32\E\norm{X}^2+C\alpha^2.
\end{equation}
For finite \(n\), the second moments below are finite by induction from \(X_0=0\), the finite variance of \(g(0,\xi)\) under \(\pi_\Xi\), and the global Lipschitz bounds in Lemma~\ref{lem:N7-consequences}.  Write
\[
        \mathsf G_n:=\mathsf G(X_n,\xi_{n+1}),
        \qquad
        X_{n+1}=X_n-\alpha\mathsf G_n.
\]
The mean-perturbation condition \eqref{eq:mean-perturbation-control}, with
\(y=0\), implies
\[
\begin{aligned}
        &\E\!\left[
        \ip{X_n}{\mathsf G(X_n,\xi_{n+1})-\mathsf G(0,\xi_{n+1})}
        \,\middle|\,X_n,\xi_n\right]\\
        &\qquad=
        \ip{X_n}{
        h(X_n)+\bar g(X_n,\xi_n)-\bar g(0,\xi_n)}
        \ge
        (1-\theta)\ip{X_n}{h(X_n)}.
\end{aligned}
\]
The Poisson identity at the minimizer gives
\(
        \E\ip{X_n}{g(0,\xi_{n+1})}
        =
        \E\ip{X_n}{\chi_0(\xi_n)-\chi_0(\xi_{n+1})}.
\)
Expanding \(J_\alpha^0\), using
\(g(X_n,\xi_{n+1})=g(0,\xi_{n+1})+
\mathsf G(X_n,\xi_{n+1})-\mathsf G(0,\xi_{n+1})-h(X_n)\), and applying the preceding cancellation yield
\[
\begin{aligned}
        &\E[J_\alpha^0(X_{n+1},\xi_{n+1})-
        J_\alpha^0(X_n,\xi_n)]\\
        &\quad=-2\alpha\E\ip{X_n}{
        \mathsf G(X_n,\xi_{n+1})-\mathsf G(0,\xi_{n+1})} \\
        &\quad+\alpha^2\E\norm{\mathsf G_n}^2
        +2\alpha^2\E\ip{\mathsf G_n}{\chi_0(\xi_{n+1})}.
\end{aligned}
\]
The global Lipschitz bounds in Lemma~\ref{lem:N7-consequences}, together with
\(\E_{\pi_\Xi}\norm{g(0,\xi)}^2<\infty\), imply
\(
        \E\norm{\mathsf G_n}^2\le C\bigl(1+\E\norm{X_n}^2\bigr).
\)
Moreover, Cauchy--Schwarz and \(\chi_0\in L^2(\pi_\Xi)\) give
\[
        \abs{\E\ip{\mathsf G_n}{\chi_0(\xi_{n+1})}}
        \le
        \bigl(\E\norm{\mathsf G_n}^2\bigr)^{1/2}
        \bigl(\E_{\pi_\Xi}\norm{\chi_0}^2\bigr)^{1/2}
        \le C\bigl(1+\E\norm{X_n}^2\bigr).
\]
Combining these bounds with \eqref{eq:h-quadratic-coercive-for-L2} yields
\[
        \E J_\alpha^0(X_{n+1},\xi_{n+1})
        \le
        \E J_\alpha^0(X_n,\xi_n)-c\alpha\E\norm{X_n}^2+C\alpha
        +C\alpha^2\bigl(1+\E\norm{X_n}^2\bigr).
\]
For sufficiently small \(\alpha\), \eqref{eq:Lalpha-comparable-L2} absorbs the \(C\alpha^2\E\norm{X_n}^2\) term and gives
\begin{equation}
\label{eq:Lalpha-geometric-bound}
        \E J_\alpha^0(X_{n+1},\xi_{n+1})
        \le
        (1-c\alpha)\E J_\alpha^0(X_n,\xi_n)+C\alpha.
\end{equation}
Iterating \eqref{eq:Lalpha-geometric-bound} and using \(X_0=0\) gives
\(\sup_n\E\norm{X_n}^2<\infty\).  By Theorem~\ref{thm:main-markov}, the
finite-time laws converge weakly to \(\pi_\alpha\).  Lower semicontinuity of
\(x\mapsto\norm{x}^2\) then gives
\(\E_{\pi_\alpha}\norm{X_\alpha}^2<\infty\).  Lemma~\ref{lem:scaling-noise-second-markov}
and the global Lipschitz bound in Lemma~\ref{lem:N7-consequences} give the
second moment of \(h(X_\alpha)+g(X_\alpha,\xi_\alpha^+)\).

We now prove \eqref{eq:stationary-cross-poisson-bound}.
Lemma~\ref{lem:scaling-noise-second-markov} and
\eqref{eq:population-cocoercivity} imply
\[
        \E\norm{h(X_\alpha)+g(X_\alpha,\xi_\alpha^+)}^2
        \le C\left(1+\E\ip{X_\alpha}{h(X_\alpha)}\right).
\]
By stationarity and the Poisson identity at the minimizer,
\[
\begin{aligned}
        \E\ip{X_\alpha}{g(0,\xi_\alpha^+)}
        &=
        \E\ip{X_\alpha}{\chi_0(\xi_\alpha)-\chi_0(\xi_\alpha^+)}\\
        &=
        \E\ip{X_\alpha^+-X_\alpha}{\chi_0(\xi_\alpha^+)}\\
        &=
        -\alpha\E\ip{
        h(X_\alpha)+g(X_\alpha,\xi_\alpha^+)}
        {\chi_0(\xi_\alpha^+)}.
\end{aligned}
\]
Thus Cauchy--Schwarz gives
\[
        \E\ip{X_\alpha}{g(0,\xi_\alpha^+)}
        \ge
        -C\alpha
        \left(\E\norm{h(X_\alpha)+g(X_\alpha,\xi_\alpha^+)}^2\right)^{1/2}
        \ge
        -C\alpha\left(1+\E\ip{X_\alpha}{h(X_\alpha)}\right).
\]
Finally, the mean-perturbation condition with \(y=0\) gives
\[
\begin{aligned}
        &\E\ip{X_\alpha}{
        h(X_\alpha)+g(X_\alpha,\xi_\alpha^+)-g(0,\xi_\alpha^+)}\\
        &\qquad\ge(1-\theta)\E\ip{X_\alpha}{h(X_\alpha)}.
\end{aligned}
\]
Combining the last two displays and subtracting
\(\E\ip{X_\alpha}{h(X_\alpha)}\) proves
\eqref{eq:stationary-cross-poisson-bound}.
This proves the lemma.
\end{proof}

\begin{proof}[Proof of Lemma~\ref{lem:scaling-tightness-package}]
By Lemma~\ref{lem:scaling-stationary-L2-markov}, the stationary square identity is justified.  Since \(X_\alpha^+\stackrel d=X_\alpha\),
\[
        0
        =
        -2\alpha
        \E\ip{X_\alpha}{h(X_\alpha)+g(X_\alpha,\xi_\alpha^+)}
        +
        \alpha^2
        \E\norm{h(X_\alpha)+g(X_\alpha,\xi_\alpha^+)}^2.
\]
Using \eqref{eq:stationary-cross-poisson-bound},
Lemma~\ref{lem:scaling-noise-second-markov}, and
\eqref{eq:population-cocoercivity}, we obtain
\[
        (1-\theta)\E\ip{X_\alpha}{h(X_\alpha)}
        \le
        C\alpha
        +C\alpha\E\ip{X_\alpha}{h(X_\alpha)}.
\]
For small \(\alpha\), the last term is absorbed into the left-hand side, giving
\(
        \E\ip{X_\alpha}{h(X_\alpha)}\le C\alpha.
\)
This proves \eqref{eq:scaling-xh-main}, and Lemma~\ref{lem:scaling-noise-second-markov} then gives \eqref{eq:scaling-gradient-second-main}.

The local moment and outside-probability estimates now follow from the same drift lower bounds.  Let \(r_0:=\min\{1,\frac{R_H}{2}\}\).  Lemma~\ref{lem:scaling-deterministic-estimates} gives
\[
        \norm{x}^m\le C\ip{x}{h(x)},
        \qquad \norm{x}\le r_0,
        \qquad
        \ip{x}{h(x)}\ge c_*,
        \qquad \norm{x}\ge r_0.
\]
Therefore
\[
\begin{aligned}
        \E[\norm{X_\alpha}^m\1_{\{\norm{X_\alpha}\le1\}}]
        &\le
        C\E\ip{X_\alpha}{h(X_\alpha)}
        +\PP(\norm{X_\alpha}\ge r_0)                                      \\
        &\le C\alpha,
\end{aligned}
\]
and
\[
        \PP(\norm{X_\alpha}\ge1)
        \le \PP(\norm{X_\alpha}\ge r_0)
        \le c_*^{-1}\E\ip{X_\alpha}{h(X_\alpha)}
        \le C\alpha.
\]
This proves all estimates in Lemma~\ref{lem:scaling-tightness-package}. 
\end{proof}

\begin{lemma}
\label{lem:scaled-increment-lindeberg}
Under the assumptions of Lemma~\ref{lem:scaling-tightness-package}, for every \(\eta>0\),
\begin{equation}
\label{eq:scaling-gradient-lindeberg-main}
        \E\left[
        \frac{\norm{Y_\alpha^+-Y_\alpha}^2}{\alpha^{2-2/m}}
        \1_{\{\norm{Y_\alpha^+-Y_\alpha}>\eta\}}
        \right]\longrightarrow0.
\end{equation}
Moreover, for every fixed \(R>0\), the family
\[
        \left\{
        \norm{g(\alpha^{1/m}Y_\alpha,\xi_\alpha^+)}^2
        \1_{\{\norm{Y_\alpha}\le R\}}:0<\alpha\le\alpha_0
        \right\}
\]
is uniformly integrable.
\end{lemma}

\begin{proof}
We first prove the local uniform integrability statement.  On the event \(\{\norm{Y_\alpha}\le R\}\), the point \(x=\alpha^{1/m}Y_\alpha\) remains in a fixed compact set.  If \(\beta=2\), the uniform \(x\)-Lipschitz bound \eqref{eq:global-jacobian-bounds} gives
\(
        \norm{g(x,\xi_\alpha^+)}^2
        \le C_R\{1+\norm{g(0,\xi_\alpha^+)}^2\},
\)
and \(\xi_\alpha^+\sim\pi_\Xi\), so \eqref{eq:g0-moment-gk} gives uniform integrability.  If \(1\le\beta<2\), Assumption~\ref{ass:N}\assitemref{itm:N5} gives a uniform exponential moment for the ratio
\(
        \frac{\norm{g(x,\xi_\alpha^+)}}{1+\norm{x}^{\beta-1}},
\)
while \(1+\norm{x}^{\beta-1}\) is bounded on the same compact set.  This again gives uniform integrability.

For \eqref{eq:scaling-gradient-lindeberg-main}, write
\[
        Y_\alpha^+-Y_\alpha
        =-\alpha^{1-1/m}
        \{h(X_\alpha)+g(X_\alpha,\xi_\alpha^+)\}.
\]
It is enough to prove
\begin{equation}
\label{eq:scaled-Falpha-lindeberg-proof}
        \E\left[
        \norm{\mathsf G_\alpha}^2
        \1_{\{\alpha^{1-1/m}\norm{\mathsf G_\alpha}>\eta\}}
        \right]\to0,
        \qquad
        \mathsf G_\alpha:=\mathsf G(X_\alpha,\xi_\alpha^+).
\end{equation}

Consider first \(\beta=2\).  The global Lipschitz bounds in
Lemma~\ref{lem:N7-consequences} imply
\(
        \norm{\mathsf G_\alpha}
        \le C\norm{X_\alpha}+\norm{g(0,\xi_\alpha^+)}.
\)
Using \((a+b)^2\1_{\{a+b>t\}}\le4a^2\1_{\{a>t/2\}}+4b^2\1_{\{b>t/2\}}\),
the contribution of \(g(0,\xi_\alpha^+)\) to
\eqref{eq:scaled-Falpha-lindeberg-proof} tends to zero by dominated
convergence under \(\pi_\Xi\).  For the \(X_\alpha\)-term, the threshold
\(c\eta\alpha^{1/m-1}\) tends to infinity.  For all sufficiently small
\(\alpha\), it lies in the quadratic-tail region, where
\(\norm{x}^2\le C\ip{x}{h(x)}\).  Therefore
\[
        \E\!\bigl[\norm{X_\alpha}^2
        \1_{\{\alpha^{1-1/m}C\norm{X_\alpha}>\eta\}}\bigr]
        \le C\E\ip{X_\alpha}{h(X_\alpha)}\to0
\]
by \eqref{eq:scaling-xh-main}.  This verifies the Lindeberg condition in the quadratic-tail case.

Now assume \(1\le\beta<2\), and set
\(
        w_\alpha:=1+\norm{X_\alpha}^{\beta-1}.
\)
Since \(\norm{h(x)}\le C(1+\norm{x}^{\beta-1})\) globally, conditionally on \((X_\alpha,\xi_\alpha)\) we have
\[
        \norm{\mathsf G_\alpha}
        \le C w_\alpha\left\{1+
        \frac{\norm{g(X_\alpha,\xi_\alpha^+)}}{w_\alpha}\right\}.
\]
The exponential moment in Assumption~\ref{ass:N}\assitemref{itm:N5} gives a uniform conditional second moment and a uniform conditional exponential tail for the normalized factor.  Split according to
\(w_\alpha\le \alpha^{-(m-1)/(2m)}\).  On this event, the normalized factor must exceed a multiple of \(\alpha^{-(m-1)/(2m)}\), so the contribution is bounded by
\(
        C\E w_\alpha^2 e^{-c\alpha^{-(m-1)/(2m)}}\to0.
\)
On the complementary event, the conditional second-moment bound gives a contribution at most
\(
        C\E\bigl[w_\alpha^2
        \1_{\{w_\alpha>\alpha^{-(m-1)/(2m)}\}}\bigr].
\)
To show that this expectation vanishes, fix a sufficiently large \(R_0\).
The event \(\{w_\alpha>\alpha^{-(m-1)/(2m)}\}\) is eventually contained in
\(\{\norm{X_\alpha}\ge R_0\}\).  On this region,
Lemma~\ref{lem:scaling-deterministic-estimates} gives
\((1+\norm{x}^{\beta-1})^2\le C(1+\norm{x}^\beta)\le C\ip{x}{h(x)}\),
after increasing \(R_0\) if necessary.  Hence
\[
        \E\bigl[w_\alpha^2
        \1_{\{w_\alpha>\alpha^{-(m-1)/(2m)}\}}\bigr]
        \le C\E\ip{X_\alpha}{h(X_\alpha)}\to0.
\]
This proves \eqref{eq:scaled-Falpha-lindeberg-proof}, and hence \eqref{eq:scaling-gradient-lindeberg-main}.
\end{proof}

\subsection{Proof of Lemma~\ref{lem:noisestate-decorrelation}}
\label{sec:proof-noisestate-decorrelation}

\begin{proof}[Proof of Lemma~\ref{lem:noisestate-decorrelation}]
Let \(u\) be the Poisson solution from Lemma~\ref{lem:poisson-noise}; thus
\[
        u-Qu=f,
        \qquad
        u\in L^2(\pi_\Xi).
\]
Since \(\psi(Y_\alpha)\) is measurable with respect to \((X_\alpha,\xi_\alpha)\),
\[
\begin{aligned}
        \E[\psi(Y_\alpha)f(\xi_\alpha)]
        &=
        \E[\psi(Y_\alpha)u(\xi_\alpha)]
        -
        \E[\psi(Y_\alpha)Qu(\xi_\alpha)]                                      \\
        &=
        \E[\psi(Y_\alpha)u(\xi_\alpha)]
        -
        \E[\psi(Y_\alpha)u(\xi_\alpha^+)].
\end{aligned}
\]
Add and subtract \(\psi(Y_\alpha^+)u(\xi_\alpha^+)\).  Since
\(
        (Y_\alpha^+,\xi_\alpha^+)\stackrel d=(Y_\alpha,\xi_\alpha),
\)
the first and third terms cancel, giving
\[
        \E[\psi(Y_\alpha)f(\xi_\alpha)]
        =
        \E[(\psi(Y_\alpha^+)-\psi(Y_\alpha))u(\xi_\alpha^+)].
\]
Therefore, by Cauchy--Schwarz,
\[
        \abs{\E[\psi(Y_\alpha)f(\xi_\alpha)]}
        \le
        \Lip(\psi)
        \bigl(\E\norm{Y_\alpha^+-Y_\alpha}^2\bigr)^{1/2}
        \bigl(\E\abs{u(\xi_\alpha^+)}^2\bigr)^{1/2}.
\]
The second factor is finite and independent of \(\alpha\), because \(\xi_\alpha^+\sim\pi_\Xi\).  For the first factor,
\[
        Y_\alpha^+-Y_\alpha
        =
        -\alpha^{1-1/m}
        \{h(X_\alpha)+g(X_\alpha,\xi_\alpha^+)\}.
\]
Lemma~\ref{lem:scaling-tightness-package} gives
\(
        \E\norm{Y_\alpha^+-Y_\alpha}^2
        \le
        C\alpha^{2-2/m}.
\)
This proves the desired bound.
 
\end{proof}

\subsection{Proof of Lemma~\ref{lem:markov-generator-identification}}
\label{sec:proof-markov-generator-identification}

By \eqref{eq:population-cocoercivity}, \eqref{eq:scaling-xh-main}, and
Lemma~\ref{lem:scaling-noise-second-markov},
\begin{equation}
\label{eq:scaling-generator-moment-bounds}
        \E\norm{h(X_\alpha)}^2\le C\alpha,
        \qquad
        \sup_{0<\alpha\le\alpha_0}
        \E\norm{g(X_\alpha,\xi_\alpha^+)}^2<\infty,
        \qquad
        \E\bigl[\norm{h(X_\alpha)}
        \norm{g(X_\alpha,\xi_\alpha^+)}\bigr]\longrightarrow0.
\end{equation}
The last conclusion follows from the first two by Cauchy--Schwarz.

For the covariance calculation below, define, for \(x\in\R^d\),
\begin{equation*}
        M_x(\zeta)
        :=
        g_x(\zeta)g_x(\zeta)^\top
        +g_x(\zeta)\chi_x(\zeta)^\top
        +\chi_x(\zeta)g_x(\zeta)^\top,
        \qquad \zeta\in\Xi.
\end{equation*}
The next two lemmas justify replacing \(M_{\alpha^{1/m}Y_\alpha}\) by
\(M_0\) on compact sets and handle integrable functions of the driving chain.

\begin{lemma}
\label{lem:covariance-replacement-compact}
Under Assumptions~\ref{ass:Hscale} and~\ref{ass:Nscale}, for every compact set \(K\subset\R^d\),
one has \(QM_0\in L^1(\pi_\Xi)\) and
\begin{equation}
\label{eq:Q-Mx-M0-L1-compact}
        \E\left[
        \norm{QM_{\alpha^{1/m}Y_\alpha}(\xi_\alpha)-QM_0(\xi_\alpha)}
        \1_{\{Y_\alpha\in K\}}
        \right]
        \longrightarrow0.
\end{equation}
\end{lemma}

\begin{proof}
The uniform \(x\)-Lipschitz bound \eqref{eq:global-jacobian-bounds} gives
\(
        \norm{g_x(\zeta)-g_0(\zeta)}\le C\norm{x}.
\)
Fix \(p\in(0,1)\).  The estimates for the solution of the Poisson equation in
\eqref{eq:poisson-solution-growth} and
\eqref{eq:poisson-solution-x-holder} give
\[
        \norm{\chi_x(\zeta)}+\norm{\chi_0(\zeta)}\le C R_\chi(\zeta),
        \qquad
        \norm{\chi_x(\zeta)-\chi_0(\zeta)}
        \le C_p\norm{x}^pR_\chi(\zeta)^{1-p}.
\]
Expanding \(M_x-M_0\), for \(\norm{x}\le1\) we obtain
\(
        \norm{M_x(\zeta)-M_0(\zeta)}
        \le C_p(\norm{x}+\norm{x}^p)F(\zeta),
\)
where
\[
        F(\zeta):=
        1+\norm{g_0(\zeta)}+R_\chi(\zeta)
        +\norm{g_0(\zeta)}R_\chi(\zeta).
\]
The stationary second moments in \eqref{eq:g0-moment-gk} and
Cauchy--Schwarz imply \(F\in L^1(\pi_\Xi)\).  In particular,
\(M_0\in L^1(\pi_\Xi)\), and invariance gives
\(QM_0\in L^1(\pi_\Xi)\).  If \(K\) is compact, then, for all sufficiently
small \(\alpha\), \(\alpha^{1/m}\sup_{y\in K}\norm y\le1\).  On
\(\{Y_\alpha\in K\}\), the preceding pointwise bound, invariance of
\(\pi_\Xi\), and \(\xi_\alpha\sim\pi_\Xi\) give
\[
\begin{aligned}
        &\E\left[
        \norm{QM_{\alpha^{1/m}Y_\alpha}(\xi_\alpha)-QM_0(\xi_\alpha)}
        \1_{\{Y_\alpha\in K\}}\right]\\
        &\qquad\le
        C_{p,K}(\alpha^{1/m}+\alpha^{p/m})\E QF(\xi_\alpha)
        =
        C_{p,K}(\alpha^{1/m}+\alpha^{p/m})\int_\Xi F\,d\pi_\Xi
        \longrightarrow0.
\end{aligned}
\]
This proves \eqref{eq:Q-Mx-M0-L1-compact}.
\end{proof}

\begin{lemma}
\label{lem:unbounded-decorrelation}
Assume Assumptions~\ref{ass:Hscale} and~\ref{ass:Nscale}.  Let
\(\alpha_k\downarrow0\) and suppose \(Y_{\alpha_k}\Rightarrow\nu\).  If
\(\psi:\R^d\to\R\) is bounded and globally Lipschitz, and if
\(\Psi\in L^1(\pi_\Xi)\), then
\begin{equation*}
        \E[\psi(Y_{\alpha_k})\Psi(\xi_{\alpha_k})]
        \longrightarrow
        \left(\int_{\R^d}\psi(y)\,\nu(dy)\right)
        \left(\int_\Xi\Psi(\xi)\,\pi_\Xi(d\xi)\right).
\end{equation*}
\end{lemma}

\begin{proof}
The constant part of \(\Psi\) is handled by weak convergence of
\(Y_{\alpha_k}\), because the \(\Xi\)-marginal of \(\pi_{\alpha_k}\) is
\(\pi_\Xi\) by Lemma~\ref{lem:xi-invariant-law}.  It remains to
consider centered \(\Psi\), i.e. \(\pi_\Xi(\Psi)=0\).  Bounded Lipschitz
functions are dense in \(L^1(\pi_\Xi)\) on the closed Euclidean set \(\Xi\).
Thus, for every \(\varepsilon>0\), choose a bounded Lipschitz \(f\) such that
\(\norm{\Psi-f}_{L^1(\pi_\Xi)}<\varepsilon\).  Replacing \(f\) by
\(f-\pi_\Xi(f)\) increases this error by at most another \(\varepsilon\), so
we may assume \(\pi_\Xi(f)=0\).  Since \(\xi_{\alpha_k}\sim\pi_\Xi\),
\[
        \abs{\E[\psi(Y_{\alpha_k})(\Psi-f)(\xi_{\alpha_k})]}
        \le 2\norm{\psi}_\infty\varepsilon.
\]
Lemma~\ref{lem:noisestate-decorrelation} gives
\(\E[\psi(Y_{\alpha_k})f(\xi_{\alpha_k})]\to0\).  Letting
\(\varepsilon\downarrow0\) proves the centered case and hence the lemma.
 
\end{proof}

\begin{lemma}
\label{lem:perturbed-generator-prelimit}
Under Assumptions~\ref{ass:Hscale} and~\ref{ass:Nscale}, for every
\(\varphi\in C_c^3(\R^d)\),
\begin{equation}
\label{eq:prelimit-generator-identity}
        0=
        -\E\ip{h_\alpha(Y_\alpha)}{\nabla\varphi(Y_\alpha)}
        +\frac12\E\tr\!\left(
        QM_{\alpha^{1/m}Y_\alpha}(\xi_\alpha)\nabla^2\varphi(Y_\alpha)
        \right)+o(1),
\end{equation}
where the \(o(1)\) is deterministic and tends to zero as \(\alpha\downarrow0\).
\end{lemma}

\begin{proof}
Write
\[
        G_\alpha:=g_{X_\alpha}(\xi_\alpha^+),
        \qquad
        \mathsf G_\alpha:=\mathsf G(X_\alpha,\xi_\alpha^+).
\]
Then
\[
        Y_\alpha^+-Y_\alpha
        =-\alpha^{2-2/m}h_\alpha(Y_\alpha)-\alpha^{1-1/m}G_\alpha
        =-\alpha^{1-1/m}\mathsf G_\alpha.
\]
Let \(K_0:=\supp(\nabla\varphi)\), and let
\(
        K_1:=\{y\in\R^d:\operatorname{dist}(y,K_0)\le1\}.
\)
If \(K_0=\varnothing\), the conclusion is immediate, so assume otherwise.
Define
\[
        \widetilde\chi_{\alpha,y}(\xi)
        :=
        \begin{cases}
        \chi_{\alpha^{1/m}y}(\xi),&y\in K_1,\\
        0,&y\notin K_1.
        \end{cases}
\]
Set
\(
        A_\alpha(y,\xi)
        :=\ip{\nabla\varphi(y)}{\widetilde\chi_{\alpha,y}(\xi)}.
\)
This is well defined because \(\nabla\varphi=0\) outside \(K_0\), and
\(\abs{A_\alpha(y,\xi)}\le C\norm{\nabla\varphi}_\infty R_\chi(\xi)\).
Introduce the localized perturbed test function
\begin{equation}
\label{eq:perturbed-test-function-GK}
        \varphi_\alpha(y,\xi)
        :=
        \varphi(y)-\alpha^{1-1/m}A_\alpha(y,\xi).
\end{equation}
The integrability required below follows from
\eqref{eq:scaling-gradient-second-main}, \(R_\chi\in L^2(\pi_\Xi)\), and
\eqref{eq:poisson-solution-growth}.  Stationarity gives
\begin{equation}
\label{eq:perturbed-stationary-identity}
        \E[\varphi_\alpha(Y_\alpha^+,\xi_\alpha^+)-\varphi_\alpha(Y_\alpha,\xi_\alpha)]=0.
\end{equation}

First consider the uncorrected part.  Taylor's formula and \(Y_\alpha^+-Y_\alpha=-\alpha^{2-2/m}h_\alpha(Y_\alpha)-\alpha^{1-1/m}G_\alpha\) give
\begin{equation}
\label{eq:uncorrected-taylor-generator}
\begin{aligned}
        \frac{\varphi(Y_\alpha^+)-\varphi(Y_\alpha)}{\alpha^{2-2/m}}
        &=-\ip{h_\alpha(Y_\alpha)}{\nabla\varphi(Y_\alpha)}
        -\alpha^{1/m-1}\ip{\nabla\varphi(Y_\alpha)}{G_\alpha} \\
        &\quad+
        \frac12\tr\!\left(G_\alpha G_\alpha^\top\nabla^2\varphi(Y_\alpha)\right)
        +r_\alpha^{(1)},
\end{aligned}
\end{equation}
with \(\E|r_\alpha^{(1)}|\to0\).  The terms in the quadratic expansion containing \(h_\alpha\) are negligible because
\[
        \alpha^{2-2/m}\norm{h_\alpha(Y_\alpha)}^2=\norm{h(X_\alpha)}^2,
        \qquad
        \alpha^{1-1/m}\norm{h_\alpha(Y_\alpha)}\norm{G_\alpha}
        =\norm{h(X_\alpha)}\norm{G_\alpha},
\]
and \eqref{eq:scaling-generator-moment-bounds} makes both expectations tend to
zero.  The third-order Taylor remainder is controlled by the uniform continuity of \(\nabla^2\varphi\): for every \(\eta>0\), after division by \(\alpha^{2-2/m}\) its absolute value is bounded by
\[
        \eta\frac{\norm{Y_\alpha^+-Y_\alpha}^2}{\alpha^{2-2/m}}
        +C_\eta
        \frac{\norm{Y_\alpha^+-Y_\alpha}^2}{\alpha^{2-2/m}}
        \1_{\{\norm{Y_\alpha^+-Y_\alpha}>\eta\}}.
\]
The first term is bounded in expectation uniformly in \(\alpha\) by
\eqref{eq:scaling-gradient-second-main}; the second tends to zero by
\eqref{eq:scaling-gradient-lindeberg-main}.  Sending \(\eta\downarrow0\) proves
\(\E|r_\alpha^{(1)}|\to0\).

Now expand the term involving the solution of the Poisson equation in
\eqref{eq:perturbed-test-function-GK}.  In
\[
        -\alpha^{1/m-1}\E\left[
        A_\alpha(Y_\alpha^+,\xi_\alpha^+)
        -A_\alpha(Y_\alpha,\xi_\alpha)
        \right],
\]
add and subtract \(A_\alpha(Y_\alpha,\xi_\alpha^+)\).  Since
\(A_\alpha(Y_\alpha,\cdot)\) can be nonzero only when
\(Y_\alpha\in K_0\), the first resulting increment is
\[
\begin{aligned}
        -\alpha^{1/m-1}
        \E\left[A_\alpha(Y_\alpha,\xi_\alpha^+)
        -A_\alpha(Y_\alpha,\xi_\alpha)\right]
        &=\alpha^{1/m-1}
        \E\ip{\nabla\varphi(Y_\alpha)}{Qg_{X_\alpha}(\xi_\alpha)}.
\end{aligned}
\]
Here the solution of the Poisson equation is evaluated at
\(X_\alpha=\alpha^{1/m}Y_\alpha\) with \(Y_\alpha\in K_0\).  The last display follows from
\(\chi_x-Q\chi_x=Qg_x\) and cancels the conditional mean of the singular
term in \eqref{eq:uncorrected-taylor-generator}, since
\(\E[G_\alpha\mid Y_\alpha,\xi_\alpha]=Qg_{X_\alpha}(\xi_\alpha)\).

It remains to control
\[
        -\alpha^{1/m-1}\E\left[
        A_\alpha(Y_\alpha^+,\xi_\alpha^+)
        -A_\alpha(Y_\alpha,\xi_\alpha^+)\right].
\]
Let \(E_\alpha:=\{\norm{Y_\alpha^+-Y_\alpha}\le1/2\}\).  By
\eqref{eq:scaling-gradient-lindeberg-main},
\[
        \PP(E_\alpha^c)
        \le4\E\left[\norm{Y_\alpha^+-Y_\alpha}^2\1_{E_\alpha^c}\right]
        =o(\alpha^{2-2/m}).
\]
The bound on \(A_\alpha\), Cauchy--Schwarz, and
\(\xi_\alpha^+\sim\pi_\Xi\) therefore give
\begin{equation}
\label{eq:localized-poisson-large-jump}
\begin{aligned}
        &\alpha^{1/m-1}\E\left[
        \left|A_\alpha(Y_\alpha^+,\xi_\alpha^+)
        -A_\alpha(Y_\alpha,\xi_\alpha^+)\right|
        \1_{E_\alpha^c}\right]\\
        &\qquad\le C\alpha^{1/m-1}
        \bigl(\E_{\pi_\Xi}R_\chi^2\bigr)^{1/2}
        \PP(E_\alpha^c)^{1/2}=o(1).
\end{aligned}
\end{equation}

On \(E_\alpha\), if either term involving \(A_\alpha\) is nonzero, then
one of \(Y_\alpha,Y_\alpha^+\) belongs to \(K_0\), and both belong to
\(K_1\).  Hence
\begin{align*}
        &A_\alpha(Y_\alpha^+,\xi_\alpha^+)
        -A_\alpha(Y_\alpha,\xi_\alpha^+)\\
        &\quad=
        \ip{\nabla\varphi(Y_\alpha^+)-\nabla\varphi(Y_\alpha)}
        {\widetilde\chi_{\alpha,Y_\alpha}(\xi_\alpha^+)}
        +\ip{\nabla\varphi(Y_\alpha^+)}
        {\widetilde\chi_{\alpha,Y_\alpha^+}(\xi_\alpha^+)
        -\widetilde\chi_{\alpha,Y_\alpha}(\xi_\alpha^+)}.
\end{align*}

For the first term on \(E_\alpha\), Taylor's formula gives
\begin{equation*}
\begin{aligned}
        &-\alpha^{1/m-1}
        \E\ip{\nabla\varphi(Y_\alpha^+)-\nabla\varphi(Y_\alpha)}
        {\widetilde\chi_{\alpha,Y_\alpha}(\xi_\alpha^+)}\1_{E_\alpha} \\
        &\qquad=
        \E\tr\!\left(G_\alpha
        \widetilde\chi_{\alpha,Y_\alpha}(\xi_\alpha^+)^\top
        \nabla^2\varphi(Y_\alpha)\right)+o(1).
\end{aligned}
\end{equation*}
The contribution of \(E_\alpha^c\) to the displayed leading term tends to
zero.  Indeed,
\[
        \E[\norm{G_\alpha}^2\1_{E_\alpha^c}]
        \le2\E[\norm{\mathsf G_\alpha}^2\1_{E_\alpha^c}]
        +2\E\norm{h(X_\alpha)}^2\longrightarrow0
\]
by \eqref{eq:scaling-gradient-lindeberg-main} and
\eqref{eq:scaling-generator-moment-bounds}; Cauchy--Schwarz with
\(R_\chi\in L^2(\pi_\Xi)\) then applies.  Also,
\(\nabla^2\varphi(Y_\alpha)\ne0\) implies \(Y_\alpha\in K_0\), so the
solution appearing in the leading term is well defined.

The contribution of \(-\alpha^{2-2/m}h_\alpha\) is
\(O(\alpha^{1-1/m})\): the factor
\(\nabla^2\varphi(Y_\alpha)\) restricts this term to a fixed compact set,
where \(h_\alpha\) is locally uniformly bounded by the expansion in
Assumption~\ref{ass:Hscale}\assitemref{itm:H4}, while
\(\E R_\chi(\xi_\alpha^+)<\infty\).

To control the Taylor remainder, let \(\omega_\varphi\) be a bounded modulus
of continuity of \(\nabla^2\varphi\).  Since
\(Y_\alpha^+-Y_\alpha=-\alpha^{1-1/m}\mathsf G_\alpha\), the absolute
expectation of the remainder after division by \(\alpha^{1-1/m}\) is at
most
\[
        C\E\!\left[
        \omega_\varphi(\alpha^{1-1/m}\norm{\mathsf G_\alpha})
        \norm{\mathsf G_\alpha}R_\chi(\xi_\alpha^+)\right].
\]
For every \(\eta>0\), the contribution of
\(\{\alpha^{1-1/m}\norm{\mathsf G_\alpha}\le\eta\}\) is bounded by
\[
        C\omega_\varphi(\eta)
        \bigl(\E\norm{\mathsf G_\alpha}^2\bigr)^{1/2}
        \bigl(\E_{\pi_\Xi}R_\chi^2\bigr)^{1/2}
        \le C\omega_\varphi(\eta).
\]
On the complementary event, Cauchy--Schwarz bounds the contribution by
\[
        C\left(
        \E\!\left[\norm{\mathsf G_\alpha}^2
        \1_{\{\alpha^{1-1/m}\norm{\mathsf G_\alpha}>\eta\}}\right]
        \right)^{1/2}
        \bigl(\E_{\pi_\Xi}R_\chi^2\bigr)^{1/2},
\]
which tends to zero by \eqref{eq:scaling-gradient-lindeberg-main}.  Sending
\(\eta\downarrow0\) proves the \(o(1)\) remainder in the gradient increment.

For the second term on \(E_\alpha\), which changes the solution of the Poisson equation from
\(\alpha^{1/m}Y_\alpha\) to \(\alpha^{1/m}Y_\alpha^+\), fix
the exponent \(p\in(1-1/m,1)\) in
\eqref{eq:poisson-solution-x-holder}.  That estimate gives
\[
\begin{aligned}
        &\alpha^{1/m-1}
        \E\left|\ip{\nabla\varphi(Y_\alpha^+)}
        {\widetilde\chi_{\alpha,Y_\alpha^+}(\xi_\alpha^+)
        -\widetilde\chi_{\alpha,Y_\alpha}(\xi_\alpha^+)}\right|
        \1_{E_\alpha} \\
        &\qquad\le
        C_p\alpha^{p-(1-1/m)}
        \E[\norm{\mathsf G_\alpha}^p
        R_\chi(\xi_\alpha^+)^{1-p}]\\
        &\qquad\le
        C_p\alpha^{p-(1-1/m)}
        =o(1).
\end{aligned}
\]
Here the second inequality follows from weighted AM--GM,
\(a^pb^{1-p}\le pa+(1-p)b\), the uniform second moment of
\(\mathsf G_\alpha\), and \(R_\chi\in L^2(\pi_\Xi)\).

Combining the preceding identities with
\eqref{eq:perturbed-stationary-identity} gives
\[
\begin{aligned}
        0&=-\E\ip{h_\alpha(Y_\alpha)}{\nabla\varphi(Y_\alpha)} \\
        &\quad+\frac12\E\tr\!\left(
        \Bigl[G_\alpha G_\alpha^\top
        +G_\alpha\widetilde\chi_{\alpha,Y_\alpha}(\xi_\alpha^+)^\top
        +\widetilde\chi_{\alpha,Y_\alpha}(\xi_\alpha^+)G_\alpha^\top\Bigr]
        \nabla^2\varphi(Y_\alpha)
        \right)+o(1).
\end{aligned}
\]
After multiplication by \(\nabla^2\varphi(Y_\alpha)\), conditioning on
\((Y_\alpha,\xi_\alpha)\) identifies the conditional mean with
\(QM_{\alpha^{1/m}Y_\alpha}(\xi_\alpha)\nabla^2\varphi(Y_\alpha)\).
This proves \eqref{eq:prelimit-generator-identity}.
\end{proof}

\begin{proof}[Proof of Lemma~\ref{lem:markov-generator-identification}]
Let \(\varphi\in C_c^3(\R^d)\) and let
\(K:=\supp(\nabla^2\varphi)\).  Lemma~\ref{lem:perturbed-generator-prelimit}
gives \eqref{eq:prelimit-generator-identity}.  Along the subsequence
\(\alpha_k\downarrow0\), local uniform convergence of \(h_\alpha\) and weak
convergence give
\begin{equation}
\label{eq:drift-limit-markov-new}
        \E\ip{h_{\alpha_k}(Y_{\alpha_k})}{\nabla\varphi(Y_{\alpha_k})}
        \longrightarrow
        \int_{\R^d}\ip{h_0(y)}{\nabla\varphi(y)}\,\nu(dy).
\end{equation}
Here \(\nabla\varphi\) is compactly supported and \(h_\alpha\to h_0\)
locally uniformly by Assumption~\ref{ass:Hscale}\assitemref{itm:H4}.

For the covariance term, \eqref{eq:Q-Mx-M0-L1-compact} replaces
\(QM_{\alpha_k^{1/m}Y_{\alpha_k}}(\xi_{\alpha_k})\) by
\(QM_0(\xi_{\alpha_k})\) on \(K\).  Each entry of \(QM_0\) belongs to
\(L^1(\pi_\Xi)\) by Lemma~\ref{lem:covariance-replacement-compact}.
Applying Lemma~\ref{lem:unbounded-decorrelation} entrywise, with the bounded
Lipschitz entries of \(\nabla^2\varphi\), yields
\begin{equation}
\label{eq:QM0-decorrelation-limit}
        \E\tr\!\left(QM_0(\xi_{\alpha_k})\nabla^2\varphi(Y_{\alpha_k})\right)
        \longrightarrow
        \int_{\R^d}\tr\!\left(\bar M\nabla^2\varphi(y)\right)\nu(dy),
\end{equation}
where
\(
        \bar M:=\int_\Xi M_0(\xi)\,\pi_\Xi(d\xi).
\)
It remains only to identify \(\bar M\) with \(\Sigma\).  Let \((\xi_0,\xi_1)\) be two successive values of the stationary driving chain and write \(f=g_0\), \(\chi=\chi_0\).  Since \(\chi-Q\chi=Qf\),
\(
        \E[f(\xi_1)+\chi(\xi_1)\mid\xi_0]=Qf(\xi_0)+Q\chi(\xi_0)=\chi(\xi_0).
\)
With \(D_0=f(\xi_1)+\chi(\xi_1)-\chi(\xi_0)\), expansion and the preceding conditional identity give
\[
\begin{aligned}
        \Sigma
        &=\E[D_0D_0^\top]  \\
        &=\E\bigl[f(\xi_1)f(\xi_1)^\top
        +f(\xi_1)\chi(\xi_1)^\top
        +\chi(\xi_1)f(\xi_1)^\top\bigr] \\
        &=\int_\Xi M_0(\xi)\,\pi_\Xi(d\xi)=\bar M.
\end{aligned}
\]
Combining \eqref{eq:prelimit-generator-identity}, \eqref{eq:drift-limit-markov-new}, and \eqref{eq:QM0-decorrelation-limit} proves \eqref{eq:stationary-generator-limit-main}. 
\end{proof}

\subsection{Proof of Lemma~\ref{lem:limit-diffusion-identification}}
\label{sec:proof-limit-diffusion-identification}

\begin{lemma}
\label{lem:tangent-coercive}
Under Assumption~\ref{ass:Hscale} (specifically, \assitemref{itm:H2} and \assitemref{itm:H4}),
\[
        \ip{y}{h_0(y)}
        \ge
        \frac{c_{\rm in}}{m-1}\norm{y}^m,
        \qquad y\in\R^d.
\]
\end{lemma}

\begin{proof}
The claim is trivial at \(y=0\).  For \(y\ne0\), set \(x=\alpha^{1/m}y\).  For all sufficiently small \(\alpha\), \(\norm{x}\le \frac{R_H}{2}\).  The identity
\(h(x)=\int_0^1\nabla^2H(tx)x\,dt\) and
Assumption~\ref{ass:H}\assitemref{itm:H2} give
\(
        \ip{x}{h(x)}
        \ge
        \frac{c_{\rm in}}{m-1}\norm{x}^m.
\)
Equivalently,
\[
        \ip{y}{\alpha^{-(m-1)/m}h(\alpha^{1/m}y)}
        =
        \alpha^{-1}\ip{x}{h(x)}
        \ge
        \frac{c_{\rm in}}{m-1}\norm{y}^m.
\]
Letting \(\alpha\downarrow0\) and using the expansion in Assumption~\ref{ass:Hscale}\assitemref{itm:H4} together with the homogeneity of \(h_0\) gives the result.
\end{proof}

\begin{lemma}
\label{lem:tangent-power-monotone}
Under Assumption~\ref{ass:Hscale} (specifically, \assitemref{itm:H2} and \assitemref{itm:H4}), there exists \(c_m>0\) such that, for all \(y,z\in\R^d\),
\(
        \ip{y-z}{h_0(y)-h_0(z)}
        \ge c_m\norm{y-z}^m.
\)
\end{lemma}

\begin{proof}
The claim is trivial when \(y=z\).  Set \(e:=y-z\ne0\).  For sufficiently
small \(\alpha\), the segment between \(\alpha^{1/m}z\) and
\(\alpha^{1/m}y\) is contained in
\(\{\norm{x}\le \frac{R_H}{2}\}\).  It can pass through the origin for at
most one parameter value, which does not affect the integral below.  Using
Assumption~\ref{ass:H}\assitemref{itm:H2} along the segment gives
\[
\begin{aligned}
        &\ip{e}{\alpha^{-(m-1)/m}\{h(\alpha^{1/m}y)-h(\alpha^{1/m}z)\}} \\
        &\qquad =
        \alpha^{-1}\ip{\alpha^{1/m}e}{h(\alpha^{1/m}y)-h(\alpha^{1/m}z)} \\
        &\qquad =
        \alpha^{-1}\int_0^1
        \ip{\alpha^{1/m}e}{\nabla^2H(\alpha^{1/m}(z+te))\alpha^{1/m}e}\,dt \\
        &\qquad \ge
        c_{\rm in}\norm{e}^2
        \int_0^1\norm{z+te}^{m-2}\,dt.
\end{aligned}
\]
We use the elementary segment bound
\begin{equation}
\label{eq:elementary-segment-bound-power}
        \inf_{\norm v=1}\inf_{w\in\R^d}
        \int_0^1\norm{w+tv}^{m-2}\,dt>0.
\end{equation}
To prove \eqref{eq:elementary-segment-bound-power}, the case \(m=2\) is
immediate.  For \(m>2\), decompose \(w=a v+w_\perp\), with
\(w_\perp\perp v\).  Then
\[
        \int_0^1\norm{w+tv}^{m-2}\,dt
        \ge
        \int_0^1 |a+t|^{m-2}\,dt.
\]
The last integral is minimized over \(a\in\R\) when the interval \([a,a+1]\) is centered at the origin, and its minimum is
\(2\int_0^{1/2}t^{m-2}\,dt>0\).  Thus \eqref{eq:elementary-segment-bound-power} holds and gives
\(
        \int_0^1\norm{z+te}^{m-2}\,dt
        \ge c\norm{e}^{m-2}.
\)
Letting \(\alpha\downarrow0\) and using the expansion in Assumption~\ref{ass:Hscale}\assitemref{itm:H4} together with the homogeneity of \(h_0\) gives the result.
\end{proof}

\begin{proof}[Proof of Lemma~\ref{lem:limit-diffusion-identification}]
The drift \(-h_0\) is locally Lipschitz because \(H_0\in C^2\), and it has
polynomial growth by homogeneity.  Thus the SDE has a pathwise unique maximal
strong solution up to its explosion time.  Let
\[
        W(y):=1+\norm{y}^2,
        \qquad
        \mathcal L\varphi(y)
        :=-\ip{h_0(y)}{\nabla\varphi(y)}
        +\frac12\tr(\Sigma\nabla^2\varphi(y)).
\]
By Lemma~\ref{lem:tangent-coercive},
\begin{equation}
\label{eq:limit-diffusion-LW}
        \mathcal LW(y)
        =-2\ip{y}{h_0(y)}+\tr(\Sigma)
        \le C-c\norm{y}^m.
\end{equation}
If \(\tau_R:=\inf\{t:\norm{Y_t}\ge R\}\), It\^o's formula applied to \(W(Y_{t\wedge\tau_R})\) gives
\(
        \E_y W(Y_{t\wedge\tau_R})\le W(y)+Ct.
\)
On \(\{\tau_R\le t\}\), the left-hand side is at least \(1+R^2\).  Hence
\(\PP_y(\tau_R\le t)\le (W(y)+Ct)/(1+R^2)\to0\), which proves
nonexplosion.  Applying It\^o's formula without stopping and using
\eqref{eq:limit-diffusion-LW} gives
\(
        \frac1T\int_0^T \E_y\norm{Y_t}^m\,dt
        \le C+\frac{W(y)}{cT}.
\)
The time-averaged laws \(T^{-1}\int_0^T P_t(y,\cdot)\,dt\) are therefore tight.  Since the nonexplosive locally Lipschitz SDE is Feller, the Krylov--Bogoliubov argument gives at least one invariant probability law.

We next show that every invariant probability law has finite \(m\)-moment.  Let \(\theta_n\in C^2([0,\infty))\) be nondecreasing and concave, with \(0\le\theta_n'\le1\), \(\theta_n'(s)=1\) for \(s\le n\), \(\theta_n'(s)=0\) for \(s\ge2n\), \(\theta_n''\le0\), and \(\theta_n'(s)\uparrow1\) for every fixed \(s\).  Set \(W_n(y):=\theta_n(W(y))\).  Then \(W_n\) is bounded with bounded first and second derivatives.  If \(\mu\) is invariant, then \(\int \mathcal LW_n\,d\mu=0\).  Moreover,
\[
        \mathcal LW_n(y)
        =\theta_n'(W(y))\mathcal LW(y)
        +\frac12\theta_n''(W(y))
        \norm{\Sigma^{1/2}\nabla W(y)}^2
        \le \theta_n'(W(y))(C-c\norm{y}^m),
\]
because \(\theta_n''\le0\) and \(\Sigma\succeq0\).  Hence
\[
        c\int \theta_n'(W(y))\norm{y}^m\,\mu(dy)
        \le C\int \theta_n'(W(y))\,\mu(dy)
        \le C.
\]
Letting \(n\to\infty\) and using monotone convergence gives
\(
        \int \norm{y}^m\,\mu(dy)<\infty.
\)
In particular every invariant law has finite second moment.

Let \(\mu\) and \(\nu\) be invariant laws.  Choose a coupling \((Y_0,\widetilde Y_0)\) with finite second moment and drive both solutions by the same Brownian motion.  Since the diffusion coefficient is constant, \(\Delta_t:=Y_t-\widetilde Y_t\) satisfies, up to localization,
\[
        \frac{d}{dt}\norm{\Delta_t}^2
        =-2\ip{\Delta_t}{h_0(Y_t)-h_0(\widetilde Y_t)}
        \le -2c_m\norm{\Delta_t}^m
\]
by Lemma~\ref{lem:tangent-power-monotone}.  With \(D(t):=\E\norm{\Delta_t}^2\), Fatou's lemma and Jensen's inequality give
\(
        D'(t)\le -2c_m\E\norm{\Delta_t}^m
        \le -2c_mD(t)^{m/2}.
\)
Thus \(D(t)\to0\), exponentially if \(m=2\) and polynomially if \(m>2\).  The marginals remain \(\mu\) and \(\nu\), so \(W_2(\mu,\nu)^2\le D(t)\to0\).  Hence \(\mu=\nu\).  Denote the unique invariant law by \(\nu_\infty\).

Finally, suppose \(\rho\) satisfies the stationary weak equation for every
\(\varphi\in C_c^3(\R^d)\).  It then holds on \(C_c^\infty(\R^d)\).  The
martingale problem for \(\mathcal L\) is well posed because the SDE is
pathwise unique and nonexplosive.  The Echeverria criterion therefore implies
that \(\rho\) is invariant; see \cite[Ch.~8]{EthierKurtz1986}.  By uniqueness,
\(\rho=\nu_\infty\).
\end{proof}

\section{Separable extensions}
\label{sec:separable-proofs-appendix}

This appendix proves the separable results from
Section~\ref{subsec:separable-main}.  We sum the one-dimensional contractions
for the constant-stepsize result and apply the generator argument to the
coordinates that remain nonzero under the common normalization.

\subsection{Proof of Corollary~\ref{cor:sep-markov-fixed}}
\label{sec:proof-sep-markov-fixed}

\begin{proof}[Proof of Corollary~\ref{cor:sep-markov-fixed}]
Write \(d_{i,\alpha}:=d_{V_{\alpha,i},\alpha}^{(i)}\).  For each coordinate
\(i\), the one-dimensional proof gives a constant \(c_i>0\) and, after
taking the minimum of the finitely many stepsize thresholds, the synchronous
one-step estimate
\[
    \E\Bigl[d_{i,\alpha}\bigl(F_{U_1}^{(i)}(x_i,\xi),F_{U_1}^{(i)}(y_i,\eta)\bigr)\Bigr]
    \le (1-c_i\alpha^{m_i-1})\,d_{i,\alpha}\bigl((x_i,\xi),(y_i,\eta)\bigr)
\]
for the coordinate map
\[
    F_u^{(i)}(x_i,\xi)
    :=
    \Bigl(x_i-\alpha\bigl(h_i(x_i)+g_i(x_i,\Phi(\xi,u))\bigr),\Phi(\xi,u)\Bigr).
\]
The \(i\)th coordinate projection of the full separable map is
\(
    (x,\xi)\mapsto F_u^{(i)}(x_i,\xi).
\)
Since the same innovation \(U_1\) drives every coordinate, summing the
coordinatewise estimates gives
\[
    \E\Bigl[d_{{\rm sep},\alpha}\bigl(F_{U_1}(x,\xi),F_{U_1}(y,\eta)\bigr)\Bigr]
    \le (1-c_0\alpha^{m_{\max}-1})\,d_{{\rm sep},\alpha}\bigl((x,\xi),(y,\eta)\bigr),
\]
where \(c_0:=\min_{1\le i\le d}c_i>0\).  Indeed, for
\(0<\alpha\le1\),
\[
        c_i\alpha^{m_i-1}
        \ge c_0\alpha^{m_{\max}-1},
        \qquad 1\le i\le d.
\]
This is the required one-step contraction in the additive separable metric.

It remains to check the reference-point integrability.  The one-dimensional
hypotheses may use different reference points.  Let \(\xi_\star^{(i)}\) be one
for coordinate \(i\), and fix \(\xi^\circ\in\Xi\).  By
Assumption~\ref{ass:N}\assitemref{itm:N1},
\[
    \E\norm{\Phi(\xi^\circ,U_1)-\xi^\circ}
    \le (1+\rho_\Xi)\norm{\xi^\circ-\xi_\star^{(i)}}
    + \E\norm{\Phi(\xi_\star^{(i)},U_1)-\xi_\star^{(i)}}<\infty.
\]
If \(\beta_i=2\), then Assumption~\ref{ass:N}\assitemref{itm:N2} gives
\[
    \E\abs{g_i(0,\Phi(\xi^\circ,U_1))}
    \le
    \E\abs{g_i(0,\Phi(\xi_\star^{(i)},U_1))}
    + L_{g,\Phi}\norm{\xi^\circ-\xi_\star^{(i)}}<\infty.
\]
Thus \(\xi^\circ\) works for every coordinate.  Each \(d_{i,\alpha}\)
dominates the base norm
\(\abs{x_i-y_i}+\alpha^{-1}\norm{\xi-\eta}\) and is locally comparable to it
on bounded sets.  Hence \(d_{{\rm sep},\alpha}\) is complete and separable on
\(\R^d\times\Xi\) and has the usual Borel \(\sigma\)-field.  Summing the
coordinatewise reference-point costs at \((0,\xi^\circ)\) and applying
Proposition~\ref{prop:CD-to-ergodic} gives the invariant law, uniqueness, and
\eqref{eq:sep-markov-contraction}.
\end{proof}

\subsection{Proof of Corollary~\ref{cor:sep-markov-scaling}}
\label{sec:proof-sep-markov-scaling}
\begin{proof}[Proof of Corollary~\ref{cor:sep-markov-scaling}]
Write \(q:=\lvert\mathcal I_*\rvert\), use a subscript \(*\) to denote
restriction to the coordinates in \(\mathcal I_*\), and set
\(
        Y_{\alpha,*}
        :=\alpha^{-1/m_{\max}}
        X_{\infty,*}^{(\alpha),\rm sep}.
\)
The coordinatewise moment bounds from the proof of
Theorem~\ref{thm:clt-iso-markov} imply that
\((Y_{\alpha,*})_\alpha\) is tight in \(\R^q\).  We first identify its
subsequential limits.  By separability, the update of the active block does
not depend on the remaining iterate coordinates.  On the scale
\(\alpha^{1/m_{\max}}\), its stationary one-step increment is
\[
        Y_{\alpha,*}^+-Y_{\alpha,*}
        =-\alpha^{2-2/m_{\max}}h_{\alpha,*}(Y_{\alpha,*})
        -\alpha^{1-1/m_{\max}}
        g_*(\alpha^{1/m_{\max}}Y_{\alpha,*},\xi_\alpha^+),
\]
where
\[
        h_{\alpha,*}(y)
        :=\bigl(
        \alpha^{-(m_{\max}-1)/m_{\max}}
        h_i(\alpha^{1/m_{\max}}y_i)
        \bigr)_{i\in\mathcal I_*},
\]
and \(g_*(x,\xi):=(g_i(x_i,\xi))_{i\in\mathcal I_*}\).  The coordinatewise
tangent assumptions give
\[
        h_{\alpha,*}\longrightarrow h_{0,*},
        \qquad
        h_{0,*}(y):=(h_{i,0}(y_i))_{i\in\mathcal I_*}
\]
locally uniformly.

For \(x\in\R^q\), let
\(
        \chi_{*,x}(\xi)
        :=(\chi_{i,x_i}(\xi))_{i\in\mathcal I_*},
\)
where \(\chi_{i,x_i}\) is the one-dimensional solution of the Poisson
equation for coordinate \(i\).  At the minimizer, define
\[
        D_i(\xi,U_1)
        :=g_i(0,\Phi(\xi,U_1))
        +\chi_{i,0}(\Phi(\xi,U_1))-\chi_{i,0}(\xi),
        \qquad i\in\mathcal I_*.
\]
The vector \(D_*=(D_i)_{i\in\mathcal I_*}\) is conditionally centered and,
by Lemma~\ref{lem:poisson-gk-construction},
\[
        \Sigma_{\rm act}:=\E[D_*D_*^\top]
        =(\Sigma_{\rm sep})_{\mathcal I_*\times\mathcal I_*}.
\]
Thus the off-diagonal asymptotic covariances created by the common driving
chain are retained within the active block.

For \(\varphi\in C_c^3(\R^q)\), use the perturbed test function
\[
        \varphi_\alpha(y,\xi)
        :=\varphi(y)
        -\alpha^{1-1/m_{\max}}
        \ip{\nabla\varphi(y)}
        {\chi_{*,\alpha^{1/m_{\max}}y}(\xi)}.
\]
The calculation in Lemma~\ref{lem:markov-generator-identification} applies
componentwise to this block.  For
\[
        M_{*,x}(\xi)
        :=g_*(x,\xi)g_*(x,\xi)^\top
        +g_*(x,\xi)\chi_{*,x}(\xi)^\top
        +\chi_{*,x}(\xi)g_*(x,\xi)^\top,
\]
each entry is a finite sum of products of one-dimensional terms.  The
coordinatewise Lipschitz estimates for \(g_i\), the H\"older estimates
\eqref{eq:poisson-solution-x-holder} for \(\chi_i\), and the coordinatewise
stationary second moments therefore give, as in
Lemma~\ref{lem:covariance-replacement-compact},
\[
        \E\!\left[
        \norm{
        QM_{*,\alpha^{1/m_{\max}}Y_{\alpha,*}}(\xi_\alpha)
        -QM_{*,0}(\xi_\alpha)}
        \1_{\{Y_{\alpha,*}\in K\}}
        \right]\longrightarrow0
\]
for every compact \(K\subset\R^q\).  The entries of \(QM_{*,0}\) belong to
\(L^1(\pi_\Xi)\).  Applying Lemma~\ref{lem:unbounded-decorrelation} entrywise
to \(\nabla^2\varphi(Y_{\alpha,*})\) gives the required product limits.
Consequently, every subsequential weak
limit \(\nu_*\) of \(Y_{\alpha,*}\) satisfies
\[
        \int_{\R^q}
        \left[
        -\ip{\nabla\varphi(y)}{h_{0,*}(y)}
        +\frac12\tr\bigl(\Sigma_{\rm act}\nabla^2\varphi(y)\bigr)
        \right]\nu_*(dy)=0,
        \qquad \varphi\in C_c^3(\R^q).
\]
Because every coordinate in \(\mathcal I_*\) has exponent \(m_{\max}\),
the active drift satisfies, for constants \(c,c'>0\),
\[
        \ip{y}{h_{0,*}(y)}
        \ge c\sum_{i\in\mathcal I_*}|y_i|^{m_{\max}}
        \ge c'\norm{y}^{m_{\max}},
\]
and the same argument gives
\(
        \ip{y-z}{h_{0,*}(y)-h_{0,*}(z)}
        \ge c'\norm{y-z}^{m_{\max}}.
\)
Lemma~\ref{lem:limit-diffusion-identification} therefore identifies a unique
invariant law for the active block of \eqref{eq:sep-limitSDE-d}.  Hence the
entire family \(Y_{\alpha,*}\) converges to this law.

It remains to consider coordinates outside \(\mathcal I_*\).  If
\(i\notin\mathcal I_*\), then
\[
        \frac{X_{\infty,i}^{(\alpha),\rm sep}}
        {\alpha^{1/m_{\max}}}
        =\alpha^{1/m_i-1/m_{\max}}
        \left(
        \frac{X_{\infty,i}^{(\alpha),\rm sep}}{\alpha^{1/m_i}}
        \right)
        \longrightarrow0
        \qquad\text{in probability},
\]
because the bracketed family is tight by the coordinatewise limit stated
before the corollary.  In \eqref{eq:sep-limitSDE-d}, these inactive
coordinates have no Brownian forcing and solve
\(
        dY_{i,t}=-h_{i,0}(Y_{i,t})\,dt.
\)
The one-dimensional coercivity gives
\(y_i h_{i,0}(y_i)\ge c|y_i|^{m_i}\), so every such deterministic flow
converges to zero and has \(\delta_0\) as its unique invariant law.  The
active block has the unique invariant law identified above.  Therefore
\eqref{eq:sep-limitSDE-d} has a unique invariant law, with inactive
coordinates equal to zero.  Combining the active-block convergence with the
inactive-coordinate convergence proves the asserted convergence to
\(Y_\infty^{\rm sep}\).
\end{proof}

\subsection{Proof of Remark~\ref{rem:sep-scaling-independent}}
\label{sec:pf-rmk-rem:sep-scaling-independent}
\begin{proof}[Proof of Remark~\ref{rem:sep-scaling-independent}]
For each coordinate \(i\), the one-dimensional version of Theorem~\ref{thm:clt-iso-markov} gives
\(
    \frac{X_{\infty,i}^{(\alpha)}}{\alpha^{1/m_i}}\Rightarrow Y_{\infty,i}.
\)
For each fixed \(\alpha\), the stationary law factorizes by
Remark~\ref{rem:sep-independent-coordinates}, so the rescaled coordinates are
independent.  Hence the joint characteristic function factorizes:
\[
    \E\exp\left(i\sum_{i=1}^d t_i\frac{X_{\infty,i}^{(\alpha)}}{\alpha^{1/m_i}}\right)
    = \prod_{i=1}^d \E\exp\left(i t_i\frac{X_{\infty,i}^{(\alpha)}}{\alpha^{1/m_i}}\right)
    \longrightarrow \prod_{i=1}^d \E e^{it_iY_{\infty,i}}.
\]
L\'evy's continuity theorem gives the claimed joint convergence.
\end{proof}

\bibliographystyle{abbrvnat}
\bibliography{references}

\end{document}